\documentclass[twoside,11pt]{article}

\usepackage{amsthm}
\usepackage[preprint]{jmlr2e}

\usepackage[a4paper, total={6in, 8in}]{geometry}
\usepackage[T1]{fontenc}    % use 8-bit T1 fonts
\usepackage{lmodern}
\usepackage{hyperref}       % hyperlinks
\usepackage{url}            % simple URL typesetting
\usepackage{booktabs}       % professional-quality tables
\usepackage{amsfonts}       % blackboard math symbols
\usepackage{nicefrac}       % compact symbols for 1/2, etc.
\usepackage{microtype}      % microtypography
\usepackage{xcolor}         % colors

\usepackage{amsmath}
\usepackage{dsfont}
\usepackage{amssymb}
\usepackage{bbm}
\usepackage{stmaryrd}
\usepackage{physics}
\usepackage{mathtools}

\usepackage{subcaption} % Pour mettre plusieurs figures côte à côte
\usepackage{mathrsfs}
\usepackage[all]{xy}
\usepackage{enumitem}
\usepackage{yhmath}
\usepackage{algorithmic}
\usepackage{comment}
\usepackage{cleveref}       % \cref
\usepackage{diagbox}       % For diagonal separation in a table
\usepackage{makecell}

\usepackage{pict2e}
\linethickness{1mm}
\usepackage{dashrule}  
\usepackage{wrapfig}  
\usepackage{tikz}  
\usetikzlibrary{fit}
\usetikzlibrary{patterns}
\usetikzlibrary{arrows,calc,shapes,decorations.pathmorphing}
\usepackage{tikz-cd}
\usepackage{tikz-3dplot}
\usepackage{tkz-euclide}

\linethickness{0.2mm}

\definecolor{cof}{RGB}{219,144,71}
\definecolor{pur}{RGB}{186,146,162}
\definecolor{greeo}{RGB}{91,173,69}
\definecolor{greet}{RGB}{52,111,72}

\newtheorem{thm}{Theorem}
\newtheorem{prop}[thm]{Proposition}
\theoremstyle{definition}
\newtheorem{df}[thm]{Definition}

\newtheorem{coro}[thm]{Corollary}
\newtheorem{lem}[thm]{Lemma}

\theoremstyle{remark}

\newtheorem{exs}{Examples}

\definecolor{theoremlinkcolor}{RGB}{0,0,200} % Bleu foncé,
\definecolor{citelinkcolor}{RGB}{0,64,0} % Bleu foncé,
\hypersetup{
    colorlinks=true,
    linkcolor=theoremlinkcolor, % Couleur pour tous les liens internes
    citecolor=citelinkcolor, % Couleur pour les citations (optionnel)
    urlcolor=blue % Couleur pour les URLs (optionnel)
}
\newtheoremstyle{condition}{}{}{}{}{\bfseries}{.}{.5em}{#1 \thmnote{#3}}
\theoremstyle{condition}

\makeatletter
\newtheorem*{rep@theorem}{\rep@title}
\newcommand{\newreptheorem}[2]{%
\newenvironment{rep#1}[1]{%
 \def\rep@title{#2 \ref{##1}}%
 \begin{rep@theorem}}%
 {\end{rep@theorem}}}
\makeatother

\newreptheorem{lem}{Lemma}

\newcommand{\RR}{\mathbb{R}}
\newcommand{\NN}{\mathbb{N}}

\newcommand{\cO}{\mathcal{O}}
\newcommand{\cU}{\mathcal{U}}

\newcommand{\cV}{\mathcal{V}}

\newcommand{\cY}{\mathcal{Y}}

\newcommand{\cT}{\mathcal{T}}
\newcommand{\cA}{\mathcal{A}}
\newcommand{\cL}{\mathcal{L}}
\newcommand{\cB}{\mathcal{B}}

\newcommand{\cP}{\mathcal{P}}
\newcommand{\calC}{{\mathcal C}}

\newcommand{\calM}{{\mathcal M}}
\newcommand{\calW}{{\mathcal W}}

\newcommand{\bfe}{{\mathbf e}}

\newcommand{\One}{\vb{1}}
\newcommand{\Par}{\RR^E \times \RR^B}
\newcommand{\Xspace}{\RR^{N_0\times n}}
\newcommand{\Yspace}{\RR^{N_L\times n}}
\newcommand{\KKT}{Karush-Kuhn-Tucker}

\newcommand{\lb}{\llbracket}
\newcommand{\rb}{\rrbracket}
\newcommand{\rk}{\mathrm{rank}}

\newcommand{\DT}{\mathbf{D}}    % macro pour D_\theta, pour l'instant je garde une macro differentiente
\newcommand{\Xtr}{X_{\text{train}}} 
\newcommand{\Xts}{X_{\text{test}}} 
\newcommand{\Ytr}{Y_{\text{train}}} 
\newcommand{\Yts}{Y_{\text{test}}}

\usepackage{transparent}
\newcommand{\semitransp}[2][0.6]{{\transparent{#1}#2}}
\newcommand{\mincolor}[1]{\color{blue} #1 \color{black}}

\newcommand{\NAME}{local dimension}

\DeclareMathOperator{\Span}{Span}
\DeclareMathOperator{\Range}{Range}
\DeclareMathOperator{\vect}{vect}

\DeclareMathOperator{\Inter}{Int}
\DeclareMathOperator{\Closure}{Clos}

\DeclareMathOperator{\argmin}{Argmin}
\DeclareMathOperator{\conv}{conv}
\DeclareMathOperator{\minimize}{minimize}

\newcommand{\RK}[1]{\rk{\left( #1 \right)}}
\newcommand{\CLOSURE}[1]{\Closure{\left( #1 \right)}}

\usepackage{lastpage}
\ShortHeadings{Geometry-induced Regularization}{Bona-Pellissier, Malgouyres and Bachoc}
\firstpageno{1}

\begin{document}

\title{Geometry-induced Regularization in Deep ReLU Neural Networks}

\author{\name Joachim Bona-Pellissier       \email joachim.bona@edu.unige.it \\
       \addr MaLGa Center ; DIBRIS\\
         Università degli Studi di Genova \\
       Genoa, Italy
       \AND
       \name François Malgouyres \email francois.malgouyres@math.univ-toulouse.fr \\
       \addr Institut de Math\'ematiques de Toulouse ; UMR 5219\\
         Universit\'e de Toulouse ; CNRS \\
       UPS, F-31062 Toulouse Cedex 9, France
       \AND
       \name François Bachoc \email francois.bachoc@univ-lille.fr \\
       \addr Laboratoire Paul Painlevé ; UMR 8524\\
         Universit\'e de Lille ; CNRS \\
       F-59000 Lille, France
       }
 
\editor{My editor}

\maketitle
\begin{abstract}
Neural networks with a large number of parameters often do not overfit, owing to implicit regularization that favors \lq good\rq{} networks. Other related and puzzling phenomena include properties of flat minima, saddle-to-saddle dynamics, and neuron alignment.

To investigate these phenomena, we study the local geometry of deep ReLU neural networks. We show that, for a fixed architecture, as the weights vary, the image of a sample $X$ forms a set whose local dimension changes. The parameter space is partitioned into regions where this local dimension remains constant. The local dimension is invariant under the natural symmetries of ReLU networks (i.e., positive rescalings and neuron permutations). 

We establish then that the network’s geometry induces a regularization, with the local dimension serving as a key measure of regularity. Moreover, we relate the local dimension to a new notion of flatness of minima and to saddle-to-saddle dynamics. 
For shallow networks, we also show that the local dimension is connected to the number of linear regions perceived by $X$, offering insight into the effects of regularization. This is further supported by experiments and linked to neuron alignment. 
Our analysis offers, for the first time, a simple and unified geometric explanation that applies to all learning contexts for these phenomena, which are usually studied in isolation.

Finally, we explore the practical computation of the local dimension and present experiments on the MNIST dataset, which highlight geometry-induced regularization in this setting.

\end{abstract}


%%%%%%%%%%%%%%%%%%%%%%%%%%%%%%%%%%%%%%%%%%%%%%%%%%%%%%%%%%%%
%%%%%%%%%%%%%%%%%%%%%%%%%%%%%%%%%%%%%%%%%%%%%%%%%%%%%%%%%%%%

\section{Introduction} \label{section:introduction}

We introduce the context of the present work in Section \ref{local-comp-measures} and provide a first overview of the objects of study in Section \ref{local-dimensions-sec}. Section \ref{main-contributions-sec} outlines the main contributions, while Section \ref{related-works-sec} reviews related work.

\subsection{On the Importance of Local Complexity Measures for Neural Networks}\label{local-comp-measures}

Learning deep neural networks has a huge impact on many practical aspects of our lives. This requires
optimizing a non-convex function, in a large dimensional space. Surprisingly, in many cases, although the number of parameters
defining the neural network exceeds by far the amount of training data, the learned neural network generalizes
and performs well with unseen data \citep{zhang2021understanding}. This is surprising because in this setting the set of global minimizers
is large \citep{cooper2021global,limeasuring} and contains elements that generalize poorly \citep{wu2017towards,neyshabur2017exploring}. In accordance with this empirical observation, the good generalization behavior is not explained by the classical statistical learning theory (e.g., \citealp{AB09-NeuralNetworkLearning,grohs22-mathematicalAspectsDeepLearning}) that considers the worst possible parameters in the parameter set. For instance, the Vapnik-Chervonenkis dimension of feedforward neural networks of depth $L$, with $W$ parameters, with the ReLU activation function is\footnote[1]{The notation $\widetilde{O}(\cdot)$ ignores logarithmic factors.} $\widetilde{O}(L W)$ \citep{bartlett2019nearly,bartlett1998almost,harvey2017nearly,maass1994neural}, leading to an upper bound on the generalization gap of order\textcolor{theoremlinkcolor}{\footnotemark[1]} $\widetilde{O}(\sqrt{\frac{L W}{n}})$, where $n$ is the sample size. This worst-case analysis is not accurate enough to explain the success of deep learning, when $W\gg n$. 

This leads to the conclusion that a global analysis, that applies to all global minima and the worst possible neural network that fits the data, will not permit to explain the success of deep learning. A local analysis is needed.

Despite tremendous research efforts in this direction (see, e.g., \citealp{grohs22-mathematicalAspectsDeepLearning} and references below) a complete explanation for the good generalization behavior in deep learning is still lacking. 
The attempts of explanation suggest that optimization algorithms and notably stochastic algorithms discover \lq good minima\rq. These are minima having special properties that authors would like to model using local complexity measures that are pivotal in the mathematical explanation. Authors aim to establish that stochastic algorithms prioritize outputs (parameterizations at convergence) with low local complexity  and to demonstrate that low local complexity explains the good generalization to unseen data  ~\citep{bartlett2020benign,chaudhari2019entropy,camuto2021fractal,keskar2017on}. This is sometimes also expressed as some form of implicit regularization \citep{imaizumi2022generalization,belkin2021fit,neyshabur2017exploring}, or margin maximization for classification tasks \citep{Lyu2020Gradient,chizat2020implicit}.

In this spirit, many authors contend that the excellent generalization behavior can be attributed to the fulfillment of conditions regarding the flatness of the landscape in the proximity of the algorithm's output \citep{haddouche2025pac,keskar2017on,foret2021sharpnessaware,cha2021swad,hochreiter1997flat}. This is known however not to fully capture the good generalization phenomenon \citep{dinh2017sharp}. Other studies explain the good generalization performances by constraints involving norms of the neural network parameters  \citep{bartlett2020benign,neyshabur2015norm,golowich2018size,bartlett2017spectrally}. Despite being supported by partial arguments, none of the aforementioned local complexity measures fully explain the experimentally observed behaviors.

From a different but related perspective on implicit regularization, the saddle-to-saddle dynamics of optimization trajectories have been studied in \citet{jacot2021saddle, boursier2022gradient, abbe2023sgd, pesme2023saddle}. In addition, neuron alignment has been observed and analyzed in \citet{boursier2024early, boursier2024simplicitybiasoptimizationthreshold}.

The lack of a unifying principle for deep ReLU networks stands in sharp contrast to the case of linear networks, for which implicit regularization is better understood. The consensus is that implicit regularization constrains the rank of the prediction matrix, the matrix obtained when multiplying all the factors of the linear network \citep{arora2019implicit,razin2020implicit,saxe2019mathematical,gidel2019implicit,gissin2019implicit,achour2024loss}.

\subsection{Local Dimensions of the Image and Pre-image Sets}\label{local-dimensions-sec}

Denoting $f_{\theta}(X)$ the prediction made by the neural network of parameter $\theta$, for an input sample $X = (x^{(i)})_{i \in \lb 1,n \rb} \in \RR^{N_0 \times n}$, where $x^{(i)}$ is the $i$-th column of $X$ and the $i$-th input of the sample, 
this article investigates local geometrical complexity measures of deep ReLU neural networks, recently introduced by \cite{grigsby2022functional}. The considered complexity measure relates to the \underline{local} geometry of the {\it image set} as defined by 
\[\{f_\theta(X)~|~ \theta\mbox{ varies}\}
\]
and of the {\it pre-image set} 
\[\{\theta'~|~ f_{\theta'}(X) = f_{\theta}(X)\}.
\] 
More precisely, when the differential $\DT f_\theta(X)$ of $\theta \longmapsto f_\theta(X)$ is appropriately defined, the concept of complexity, called\footnote{It is called the {\it batch functional dimension} by \citet{grigsby2022functional}.} {\em \NAME{}}, is the rank of the aforementioned differential, denoted 
\[\RK{\DT f_\theta(X)}.
\]
 It is locally, in the vicinity of $f_\theta(X)$, the dimension of the image set and locally, in the vicinity of $\theta$, the co-dimension of the pre-image set, see Corollary \ref{coro-optim}.
Notice that, before \citeauthor{grigsby2022functional}, the \NAME~already appeared in an identifiability condition introduced by \citet{bona2022local}.
 
The analysis using the local dimension has the potential to explain implicit regularization. To explain this point, the simplest way is to look at the counterpart of $\{f_\theta(X)~|~ \theta\mbox{ varies}\}$ for a well-understood problem: $\ell^1$ regularization.
 
 \paragraph{Analogy with $\ell^1$ regularization}
 Given $A\in\RR^{n\times p}$ and $y\in\RR^n$, we write the  $\ell^1$ regularization in the form
\begin{equation}\label{Pb_l1}
        \left\{\begin{array}{l}
		 \argmin_{x} ~\|Ax- y\|^2 \\
		 \|x\|_1\leq \tau,
		 \end{array}\right.
\end{equation}
for a fixed parameter $\tau >0$.

As is well known, the analogue of $\{f_\theta(X)~|~ \theta\mbox{ varies}\}$ for this problem is then the polytope 
\[\{Ax~|~ \|x\|_1\leq \tau\} = \tau \conv(A_1,-A_1, \cdots, A_p, - A_p), 
\]
where $A_i$ denotes the $i$-th column of $A$ and $\conv$ denotes the convex hull (see Figure \ref{fig1}). This polytope is made up of facets of different dimensions. They are organized hierarchically, with smaller-dimensional facets on the boundary of larger-dimensional facets, and so on. As can be seen in Figure \ref{fig1}, the shape of the polytope will influence the trajectory of the iterates of an optimization algorithm. They will move from facet to facet until reaching a smaller-dimensional facet, and a sparse solution $x^*$.  
 
%\begin{wrapfigure}{r}{0.6\textwidth}
\begin{figure}[ht]
  	\begin{center}
    	\begin{tikzpicture}[scale=0.6,decoration=snake]  
        \node[] at (-2,2) {$y$}; \node[] at (-1.5,2) {$\times$};
%        \draw (-5,0)  circle (1);
	    \draw (-1.5,2)  circle (1.5);
        \node[] at (-1.2,1) {$Ax^*$}; \node[] at (-0.7,0.7) {$\boldsymbol{\times}$};
%	\node[] at (2,3.5) {};
	
 	\coordinate (A1) at (-2,0);\node[] at (-2.4,0) {$A_1$};
	\coordinate (A2) at (3,1);\node[] at (3.5,1.2) {$A_2$};
	\coordinate (A3) at (7,0);\node[] at (7.5,0.2) {$-A_1$};
	\coordinate (A4) at (2,-1);
	\coordinate (B1) at (2.5,2.5);\node[] at (2.5,2.8) {$A_3$};
	\coordinate (B2) at (2.5,-2.5);\node[] at (2.5,-2.8) {$-A_3$};

	\begin{scope}[thick,dashed,,opacity=0.6]
	\draw (A1) -- (A2) -- (A3);
	\draw (B1) -- (A2) -- (B2);
	\end{scope}
	\draw[fill=cof,opacity=0.6] (A1) -- (A4) -- (B1);
	\draw[fill=pur,opacity=0.6] (A1) -- (A4) -- (B2);
	\draw[fill=greeo,opacity=0.6] (A3) -- (A4) -- (B1);
	\draw[fill=greet,opacity=0.6] (A3) -- (A4) -- (B2);
	\draw (B1) -- (A1) -- (B2) -- (A3) --cycle;
	\node[] at (1.5,-1.2) {$-A_2$};
 	\end{tikzpicture}
 	\caption{\label{fig1} For $\ell^1$ regularization, the analogue of $\{f_\theta(X)~|~ \theta \mbox{ varies}\}$ is the polytope $\{Ax~|~ \|x\|_1 \leq \tau\} = \tau \conv(A_1, -A_1, \cdots, A_p, -A_p)$. The sparse vector $x^*$ is the solution of \eqref{Pb_l1}, and its image $Ax^*$ lies on a low-dimensional facet of the polytope.}
	\end{center}
\end{figure}
%     \end{wrapfigure}

In the above analogy, the sparsity for $\ell^1$ regularization is the analogue of the local dimension $\RK{\DT f_\theta(X)}$ for deep learning with ReLU networks. The sparsity is key to explaining the performance of methods like the LASSO in the case $p > n$, \cite{meinshausen2006high,yuan2007non}. We will see in this article that the local dimension is a regularity criterion induced by the geometry for deep ReLU networks.

\paragraph{Remark:} Studying $\{f_\theta(X) \mid \theta \text{ varies}\}$ and its local dimension removes the burden of dealing with a specific learning objective or algorithm. This point is crucial, since the advantages of neural networks have been widely demonstrated across diverse applications, objectives, data types, and optimization strategies—suggesting that the performance of deep learning is inherent to the properties of the networks themselves.
This also ensures that the analysis remains applicable to any learning setting.
 
\subsection{Main Contributions and Organization of the Paper}\label{main-contributions-sec}
\begin{itemize}
    \item In Theorem \ref{theorem:constant:rank} (Section \ref{section:rank:properties}), up to a negligible set, we decompose  the parameter space as a finite union of open sets. On each set, the {\it \NAME}~
    $$\RK{\DT f_\theta(X)}$$ 
    is well defined and constant. The construction of the sets shows that almost everywhere, the activation pattern (defined in Section \ref{ReLU networks-sec-main}) determines the \NAME. We also establish in Proposition \ref{prop:rank:invariant:rescaling} (Section \ref{section:rank:properties}) that the \NAME~is invariant under the symmetries of a ReLU neural network's parameterization, positive rescaling and neuron permutation, as defined in Section \ref{ReLU networks-sec-main}. We also provide examples showing that the local dimension actually varies in Sections~\ref{section:rank:properties} and \ref{section:geometric:interpretation}.
    
    \item In Section \ref{section:geometric:interpretation}, we illustrate the consequences of the statements from Section \ref{section:rank:properties} in the context of learning a deep ReLU network. In particular, \Cref{coro-optim} states that $\RK{\DT f_\theta(X)}$ corresponds to the \NAME{} of the image set and the co-dimension of the pre-image set. This is illustrated by an example in \Cref{example:image:preimage:dimdeux}. The example is low-dimensional, so the image set can be explicitly computed and visualized in Figure \ref{dessin-exemple}. We then present the geometry-induced regularization statements in \Cref{regul-implicit-sec-1} and \Cref{regul-implicit-sec-2}, where the \NAME{} emerges as the regularity criterion. Next, we relate the regularity of the network, measured by local dimension, to a new notion of flatness of minima in \Cref{subsection:minima:flatness}. Finally, in \Cref{implicit_regul-sec}, we illustrate both the geometry-induced regularization and flat minima results, and in \Cref{saddle-to-saddle-sec}, we show how local geometric changes in neural networks can lead to saddle-to-saddle dynamics.

    \item In \Cref{shallow-sec}, we examine how geometry-induced regularization affects the learned network in the shallow setting (i.e., a one-hidden-layer ReLU network). In \Cref{shallow-case-thm}, we establish that the local dimension is closely related to the number of linear pieces \lq perceived\rq{} by the sample $X$. This suggests that, in the shallow case, geometry-induced regularization favors large linear regions containing many examples. Finally, in \Cref{recovery-CPL-sec}, we demonstrate through experiments that this phenomenon indeed occurs in practice.

    \item In Section \ref{sec:numerical:computations} we provide the details on the practical computation of $\RK{\DT f_\theta(X)}$, for given $X$ and $\theta$.

    \item Finally, in Section \ref{sec:experiments}, we present experiments demonstrating that geometry-induced regularization arises when a deep ReLU network learns the MNIST dataset. Specifically, in \Cref{expe-w_varies}, we analyze the behavior of the \NAME{} as the network width increases, and in Section \ref{expe-rank_diminu}, we describe its behavior during the learning phase. 
    The results also show that the regularity observed on the training sample is \lq transferred\rq{} to the regularity computed on a large test sample.    
\end{itemize}

All the proofs are in the appendices, and the codes are available at \citep{code_calcul_rang}.

%\paragraph{Why are the complexity measures relevant?}
%Notice first that, as is well known, measuring the complexity of a family of functions parameterized by  $W$ parameters is not relevant. We can indeed artificially augment the number of parameters without modifying the family. For instance, we construct a new family of parameterized functions by artificially adding useless parameters $\theta''$ and change $\theta\in\RR^W$ for $\theta'=(\theta,\theta'')=\in\RR^{W'}$, with $W'\gg W$, but only use $\theta$ in order to compute $f_{\theta'}(X) = f_\theta(X)$.
\subsection{Related Works}\label{related-works-sec} To the best of our knowledge, the local dimension of deep ReLU neural networks has only been explicitly studied by \citet{grigsby2022functional,grigsby2023hidden}. The article \citet{grigsby2022functional} is very rich and it is difficult to summarize it in a few lines\footnote{A weakness of it is that it considers neural networks whose last layer undergoes a ReLU activation.}. The authors establish sufficient conditions guaranteeing that $\theta \longmapsto f_\theta(X)$ is differentiable. The conditions are comparable to but weaker than the one presented here. The benefit of the difference is that our conditions guarantee the value of the \NAME{}, allowing us to make the connection between the activation patterns and the \NAME{}.
Furthermore, \cite{grigsby2022functional} define and provide examples to illustrate that the \NAME~and a related notion called full functional dimension vary in the parameter space. They also prove that for all narrowing architectures\footnote{Narrowing architectures are such that widths decrease.}, the {\it functional dimension} as defined by $\max_\theta \max_X \RK{\DT f_\theta(X)}$ reaches its upper-bound $W-W'$ where $W'$ is the number of positive rescalings. They finish their article with several examples showing that the global structure of the  {\it pre-image set} $\{\theta'~|~ f_{\theta'}(X) = f_{\theta}(X)\}$ can vary in several regards. \cite{grigsby2023hidden} prove that when the input size lower-bounds the other widths, there exist parameters for which the \NAME~reaches the upper-bound $W-W'$. They also numerically estimate, for several neural network architectures, the size of the sets of parameters that reach this upper bound. 

Geometric properties of the pre-image set of a global minimizer have been studied by \citet{cooper2021global}. Topological properties of a variant of the image set included in function spaces, $\{f_{\theta} ~|~ \theta\mbox{ varies}\}$, have been established by \citet{petersen2021topological}.

There are many articles devoted to the identifiability of neural networks \citep{petzka2020symmetries,carlini2020cryptanalytic,pmlr-v119-rolnick20a,stock2022embedding,bona2022local,bona2023parameter}. For a given $\theta$, they study conditions guaranteeing that the pre-image set\footnote{In these articles $X$ sometimes contains infinitely many examples, in which case we let  $f_\theta(X)$ denote the function $f_\theta$ restricted to $X$.} of $f_\theta(X)$ coincides with the set obtained when considering all the positive rescalings of $\theta$. Of particular interest in our context, \citet{bona2022local} shows that the condition $\RK{\DT f_\theta(X)} = W - W'$ is, up to negligible sets, sufficient to guarantee local identifiability. The same condition also appears in a necessary condition for local identifiability.

Other local complexity measures, not related to the geometry of neural networks, have been considered. There are complexity measures using the number of achievable activation patterns \cite{montufar2014number,raghu2017expressive,hanin2019complexity}. Those based on norms and the flatness are already mentioned in Section \ref{local-comp-measures}.

The objects studied in this article are also related to the properties of the landscape of the empirical risk, which have been investigated in the literature. Studies of these properties for instance permit to guarantee that first-order algorithms find a global minimizer  \citep{soudry2016no,nguyen2017loss,safran2021effects,du2019gradient},
focus on the shape at the bottom of the empirical risk \citep{ghorbani2019investigation,sagun2016eigenvalues,gur2018gradient} and (again) on flatness.

The local properties studied in the present article also have an impact on the iterates trajectory of minimization algorithms and therefore the biases induced by the optimization as studied in \citet{bartlett2020benign,camuto2021fractal,keskar2017on,Lyu2020Gradient}.

Finally, \citet{pmlr-v80-arora18b} and \citet{Suzuki2020Compression} establish generalization bounds of compressed neural networks. This might provide hints for the construction of upper-bounds of the generalization gap based on the local geometric complexity measures considered in this article.

\section{ReLU Networks and Notations}\label{ReLU networks-sec-main}

This section is devoted to introducing the formalism and notations that we use throughout the article. In Section \ref{ReLU-architecture-sec}, we present the graph formalism that we use for neural networks, and we specify the architectures that we study, and in Section \ref{ReLU-network-prediction-sec}, we construct the prediction function implemented by a network, and we define the differential $\DT f_\theta(X)$ that is central in this work. In Section \ref{positive-rescaling-permutation-sec}, we recall the two classical symmetries of ReLU networks, namely positive rescalings and permutations. Finally, we introduce the activation patterns in Section \ref{activation-patterns-sec} and some additional notations in Section \ref{further-notation-sec}.

\subsection{ReLU Network Architecture}\label{ReLU-architecture-sec}

Let us introduce our notations for deep fully-connected ReLU neural networks. In this paper, a network is a graph $(E,V)$ of the following form. 
\begin{itemize}
\item $V$ is a set of neurons, which is divided into $L+1$ layers, with $L \geq 2$: $V = \bigcup_{\ell = 0}^L V_{\ell}$. 
The layer $V_0$ is the input layer, $V_L$ is the output layer and the layers $V_{\ell}$ with $1 \leq \ell \leq L-1$ are the hidden layers. Using the notation $|C|$ for the cardinality of a finite set $C$, we denote, for all\footnote{Throughout the paper, for $a,b \in \NN$, $a \leq b$, $\lb a,b \rb$ is the set of consecutive integers $\{ 
a,a+1, \ldots , b \}$.} $\ell \in \lb 0, L \rb$, $N_{\ell} = |V_{\ell} |$ the size of the layer $V_{\ell}$.
\item $E$ is the set of all oriented edges $v' \rightarrow v$ between neurons in consecutive layers, that is
\[E = \{ v' \rightarrow v ~|~\ v' \in V_{\ell-1}, v \in V_{\ell}, \text{for } \ell \in \lb 1 , L \rb \}.\]
\end{itemize}
A network is parameterized by weights and biases, gathered in its parameterization $\theta$, with
\[\theta = \left( (w_{v' \rightarrow v})_{v' \rightarrow v \in E}, (b_v)_{v \in B}\right) \quad \in \Par ,\]
where $B = \bigcup_{\ell=1}^L V_{\ell}$. We let $W=|E| + |B|$.

The activation function used in the hidden layers, and denoted $\sigma$, is always ReLU: for any $p \in \NN^*$ and any vector $x = (x_1, \dots , x_p)^T \in \mathbb{R}^p$, we set $\sigma(x) = (\max(x_1,0), \dots , \max(x_p,0))^T$. Here and in the sequel, the symbol $\NN^*$ denotes the set of natural numbers without $0$. We allow the use of a specific activation $\sigma_L: \RR^{N_L} \longrightarrow \RR^{N_L}$ for the output layer, which we only require to be analytic. For instance, $\sigma_L$ can be the identity, as is generally the case in regression, or the softmax, as is generally the case in classification. The ReLU neural network architectures considered in this article are fully characterized by a triplet $(E, V , \sigma_L)$.

\subsection{ReLU Network Prediction}\label{ReLU-network-prediction-sec}

For a given $\theta\in\Par$, we define recursively $f_\theta^{\ell} : \RR^{N_0} \longrightarrow \RR^{N_{\ell}}$ , for $\ell \in \lb 0, L \rb$ and $x\in\RR^{N_0}$, by
\begin{equation}\label{deff_l}
\left\{\begin{array}{ll}
(f_\theta^0(x))_v = x_v  &\mbox{for } v\in V_0,\\
 (f_\theta^{\ell}(x))_v = \sigma \left( \sum_{v' \in V_{\ell-1}} w_{v'\rightarrow v} (f_\theta^{\ell-1}(x))_{v'} ~+ b_v \right) & \mbox{for } v\in V_\ell \mbox{, when } \ell \in \lb 1, L-1 \rb, \\
(y_\theta^{L}(x))_v = \sum_{v' \in V_{L-1}} w_{v'\rightarrow v} (f_\theta^{L-1}(x))_{v'} ~+ b_v & \mbox{for } v\in V_L, \\
f_\theta^{L}(x) = \sigma_L(y_\theta^{L}(x)),&
\end{array}
\right.
\end{equation}
where the definition of $f_\theta^{L}(x)$ takes into account that $\sigma^L: \RR^{N_L} \longrightarrow \RR^{N_L}$ may require the whole pre-activation output. This is for instance the case for the softmax activation function.
We define the function $f_\theta : \RR^{N_0} \longrightarrow \RR^{N_L}$ implemented by the network of parameter $\theta$ as $f_\theta = f_\theta^L$. We call it the prediction.

For all $n \in \NN^*$, we concatenate a set of $n$ inputs in a matrix $X = (x^{(i)})_{i \in \lb 1,n \rb} \in \RR^{N_0 \times n}$, where $x^{(i)}$ is the $i$-th column of $X$ and the $i$-th input of the network. 
We also allow to write $f_{\theta}$ as operating on an input set $X$. In this case, we write
$f_\theta: \Xspace\longrightarrow \Yspace$ and we define $f_\theta(X)$ as the matrix whose columns correspond to the outputs $(f_\theta(x^{(i)}))_{i\in\lb 1, n \rb}$.

Among other quantities, we study in this article the set
\[\{f_\theta(X) ~|~ \theta \in \Par\},
\]
for $X\in\Xspace$ fixed, which we call an \emph{image set}. When it is differentiable at $\theta$, we denote by $\DT f_\theta(X)$ the differential, at the point $\theta$, of the mapping
\begin{eqnarray*}
 \RR^E\times \RR^B & \longrightarrow & \Yspace \\
\theta' & \longmapsto & f_{\theta'}(X).
\end{eqnarray*}
We recall that the differential at $\theta$ is the linear map 
\begin{equation}\label{def-DT}
\DT f_\theta(X) :  \RR^E\times \RR^B\longrightarrow \Yspace
\end{equation}
such that, for $\theta'\in \Par$ in a neighborhood of zero,
\begin{equation}\label{diff-def-eq} 
f_{\theta + \theta'}(X) = f_{\theta}(X) + \DT f_\theta(X) (\theta') + o(\|\theta'\|).
\end{equation}

\subsection{Positive Rescaling and Neuron Permutations Symmetries}\label{positive-rescaling-permutation-sec}

Consider two parameters $\theta , \widetilde{\theta} \in \RR^{E \times B}$, with $ \widetilde{\theta} = \left( (\widetilde{w}_{v' \to v})_{v' \to v \in E}, (\widetilde{b}_v)_{v \in B}\right)$.  We say that $\theta$ and $\widetilde{\theta}$ are equivalent modulo positive rescaling, and we write $\theta \sim_s \widetilde{\theta}$, when the following holds. 
There are $(\lambda_v)_{v \in V_0 \cup \cdots \cup V_L } \in (0,\infty)^{N_0 + \cdots + N_L}$ such that $\lambda_v = 1$ for $v \in V_0 \cup V_{L}$ and for $\ell  \in \lb 1,L \rb$, $v' \in V_{\ell-1}$, $v \in V_{\ell}$, 
\begin{align} 
w_{v' \to v}
& =
\frac{\lambda_{v}}{\lambda_{v'}}
\widetilde{w}_{v' \to v}, 
\label{eq:scaling:un}
\\ 
b_{v}
& =
\lambda_{v} \widetilde{b}_{v}.
\label{eq:scaling:deux}
\end{align}
Then it is a well-known property of ReLU networks
\citep{bona2023parameter, 
bona2022local,neyshabur2015path,stock:tel-03208517, stock2022embedding, yi2019positively} that if $\theta \sim_s \widetilde{\theta}$ then $f_{\theta} = f_{\widetilde{\theta}}$, that is, positive rescalings are a symmetry of the network parameterization. 

Another classic symmetry consists in swapping neurons, and their corresponding weights, within each hidden layer. If $\widetilde \theta$ stands for the permuted weights, we denote the corresponding equivalence relation $\widetilde \theta \sim_p \theta$. Again, when $\widetilde \theta \sim_p \theta$, we have $f_{\widetilde \theta} = f_\theta$. 

We say that  $\widetilde \theta \sim \theta$ if there exists $\theta'$ such that $\widetilde \theta \sim_p \theta'$ and $\theta' \sim_s \theta$. Again, if $\widetilde \theta \sim \theta$, then $f_{\theta} = f_{\widetilde{\theta}}$.

\subsection{Activation Patterns}\label{activation-patterns-sec}

For  any $\ell \in \lb 1 , L-1 \rb$, $v \in V_{\ell}$, $\theta \in \Par$ and $x \in \RR^{N_0}$, we define the activation indicator at neuron $v$ by
\[a_v(x,\theta) = \begin{cases} 
1 & \text{if } \sum_{v' \in V_{\ell-1}} w_{v'\rightarrow v} (f_\theta^{\ell-1}(x))_{v'} + b_v  \geq 0 \\ 
0 & \text{otherwise.}\end{cases}
\]
Using \eqref{deff_l}, we have for the ReLU activation function $\sigma$, any $\ell \in \lb 1 , L-1 \rb$ and $v \in V_{\ell}$,
\begin{equation}\label{a=relu}
    (f_\theta^{\ell}(x))_v  = a_v(x,\theta) \Bigl( \sum_{v' \in V_{\ell-1}} w_{v'\rightarrow v} (f_\theta^{\ell-1}(x))_{v'} + b_v \Bigr).
\end{equation}
We then define the {\it activation pattern} as the mapping
\begin{eqnarray*}
    a : \RR^{N_0} \times \left( \RR^E \times \RR^B \right) & \longrightarrow & \{0,1\}^{N_1+\cdots+N_{L-1}}\\
    (x, \theta) & \longmapsto & (a_v(x, \theta))_{v \in V_1 \cup \cdots \cup V_{L-1}}.
\end{eqnarray*}
For $X \in \RR^{N_0 \times n}$ as considered above, we let $a(X , \theta) \in
\{ 0 ,1 \}^{ (N_1+\cdots+N_{L-1}) \times n  }$ be defined by, for $i \in \lb 1,n \rb$ and $v \in V_1 \cup \cdots \cup V_{L-1}$, $a_{v,i}(X,\theta) = a_v(x^{(i)} , \theta)$. By extension, we also call {\it activation patterns} the elements of $\{0,1\}^{N_1+\cdots+N_{L-1}}$ or $\{ 0 ,1 \}^{ (N_1+\cdots+N_{L-1}) \times n}$.

\subsection{Further Notation}\label{further-notation-sec} 
 We use  the notation $\RK{\cdot}$ for the rank of linear maps and matrices. The determinant of a square matrix $M$ is denoted $\det(M)$.
 If the matrix $M\in\RR^{a \times b}$ for $a,b \in \NN^*$, then for $i \in \lb 1, a \rb$, we write $M_{i,:}$ for the row $i$ of $M$.

All considered vector spaces are finite dimensional and they are endowed with the standard Euclidean topology. For a subset $C\subset T$ of a topological space, we denote $\Inter(C)$ the topological interior of $C$, $\partial C$ its boundary and $C^c = T\setminus C$ the complement of $C$ (the ambient topological space $T$ should always be clear from context). For all Euclidean space $V$, all element $x\in V$, and all real number $r\geq 0$, the open Euclidean ball of radius $r$ centered at $x$ is denoted by $B(x,r)$.

%%%%%%%%%%%%%%%%%%%%%%%%%%%%%%%%%%%%%%%%%%%%%%%%%%%%%%%%%%%%
\section{Rank Properties} \label{section:rank:properties}

In this section, we give the key technical theorem, namely Theorem \ref{theorem:constant:rank}, on which the remaining of the article relies. We then illustrate the theorem with examples showing the diversity of situations that might occur. In the theorem, we study $\theta \longmapsto f_{\theta}(X)$ over $\Par$, for $X$ fixed. We must first introduce a few definitions.

For $n \in \NN^*$ and $X \in \RR^{N_0 \times n}$, the function $\theta \longmapsto a(X, \theta)$ takes a finite number of values  $\Delta^X_1, \dots, \Delta^X_{q^X}$, and we define, for $j \in \lb 1,q^X \rb$,
\begin{equation}\label{def:u:j:tilde}
    \widetilde{\cU}^X_j = \Inter \{\theta \in \Par \ | \  a(X, \theta) = \Delta^X_j \}.
\end{equation}
We keep only the nonempty such sets, and if $p_X \in \lb 1 , q^X \rb$ is the number of such sets, we can assume up to a re-ordering that we keep $\widetilde{\cU}^X_1, \dots, \widetilde{\cU}^X_{p_X}$. As we will establish in \Cref{theorem:constant:rank}, last item, for all $j \in \lb 1, p_X \rb$, the function $\theta \longmapsto f_\theta(X)$ is differentiable at $\theta$ when $\theta \in \widetilde{\cU}^X_j$.
We can therefore define, for $n \in \NN^*$, $X \in \RR^{N_0 \times n}$ and $j \in \lb 1, p_X \rb$, 
\begin{equation}
    r^X_j = \max_{\theta \in \widetilde{\cU}^X_j} \RK{\DT f_\theta(X)}.
\end{equation}
We finally define the subset of $\widetilde{\cU}^X_j$ on which the rank is maximal. For $n \in \NN^*$, $X \in \RR^{N_0 \times n}$ and $j \in \lb 1, p_X \rb$, 
\begin{equation}\label{defU}
    \cU^X_j = \{ \theta  \in \widetilde{\cU}^X_j \ | \ \RK{\DT f_\theta(X)} = r^X_j \}.
\end{equation}

In the following theorem, we provide properties of the sets $\cU^X_1 , \ldots , \cU^X_{p_X} $.
Note that Items 1, 2 and 3 hold trivially by definition, while Items 4, 5 and 6 require detailed proofs.

\begin{thm} \label{theorem:constant:rank}
Consider any deep fully-connected ReLU network architecture $(E,V, \sigma_L)$.

 For all $n \in \NN^*$ and all $X \in \Xspace$, by definition,
     \begin{enumerate}
        \item the sets $\cU^X_1 , \ldots , \cU^X_{p_X} $ are non-empty and disjoint;
         \item for all $j \in \lb 1,p_X \rb$, the function $\theta \longmapsto a(X,\theta)$ is constant on each $\widetilde{\cU}^X_j$ and takes $p_X$ distinct values on $\cup_{j=1}^{p_X} \widetilde{\cU}^X_j$;
        \item for all $j \in \lb 1,p_X \rb$, $\theta \longmapsto \RK{ \DT f_{\theta}(X) }$ is constant on $\cU^X_j$ and equal to $r_j^X$.
     \end{enumerate}
     Furthermore,
       \begin{enumerate}
        \setcounter{enumi}{3}
        \item  the sets $\cU^X_1 , \ldots , \cU^X_{p_X} $ are open;
         \item  both $\left( \cup_{j=1}^{p_X} \widetilde{\cU}^X_j \right)^c$ and $\left( \cup_{j=1}^{p_X} \cU^X_j \right)^c$ are closed with Lebesgue measure zero;
         \item for all $j \in \lb 1,p_X \rb$, the map $\theta \longmapsto f_{\theta}(X)$ is polynomial of degree $L$ on $\widetilde{\cU}^X_j$, when $\sigma_L=Id$, and it is analytic otherwise.
     \end{enumerate}
\end{thm}

The proof of the theorem is in Appendix \ref{theorem:constant:rank-proof}.

This theorem formalizes that the sets  $(\cU_j^X)_{j\in\lb1,p_X\rb}$ almost cover the spaces $\Par$. Moreover, on each set  $\widetilde{\cU}_j^X$ the activation pattern is constant, and the function $\theta \longmapsto f_{\theta}(X)$ is polynomial or  analytic. When it is polynomial, we would like to emphasize here that the structure of the polynomial is very particular. For instance, every variable appears with a degree at most one. A more complete description of the polynomial structure is, for instance, given by \citet{bona2022local,stock2022embedding}. 

Looking at the definition of  $\widetilde \cU^X_j$ and $\cU^X_j$, using that $\left( \cup_{j=1}^{p_X} \cU^X_j \right)^c$ is a closed set with Lebesgue measure zero, we find that,
\[\cU^X_j \mbox{ is open and dense in } \widetilde \cU^X_j.\]
As a consequence, $\widetilde \cU^X_j\setminus\cU^X_j $ has Lebesgue measure $0$: the activation pattern almost surely determines $\RK{\DT f_\theta(X)}$.

For $\theta \in \cU_j^X$, the conclusions concerning  $\RK{ \DT f_{\theta}(X) }$
have direct consequences on the local dimensions of the image set $\{f_{\theta'}(X) ~|~  \theta' \in B(\theta , \varepsilon) \}$ and the pre-image set $\{\theta' \in B(\theta , \varepsilon) ~|~ f_{\theta'}(X) = f_{\theta}(X)\}$, where $\varepsilon >0$ is small enough. The consequences and their implications in machine learning applications are described in greater detail in the next sections. 

Finally, the mapping $\theta \mapsto f_\theta(X)$ is smooth at any $\theta \in \widetilde \cU^X_j \setminus \cU_j^X$. However, for such a $\theta$, $\RK{\DT f_{\theta}(X)}$ is strictly smaller than $r_j^X$, i.e. for $\theta' \in \cU^X_j$. This behavior may correspond to a singularity, such as a cusp, in the set $\{f_{\theta'}(X) | \theta' \in B(\theta, \varepsilon)\}$. Such singularities are expected to influence the optimization of a learning objective. In particular, although $\widetilde \cU^X_j \setminus \cU_j^X$ is of measure $0$, its elements may be disproportionately represented among the local and global minimizers of any learning objective.

When compared to existing similar statements \citep{stock2022embedding,grigsby2022functional,bona2022local,grigsby2022transversality}, the particularity of Theorem \ref{theorem:constant:rank} is that the construction of the sets $\cU_j^X$ permits to include, in the third item, a statement on $\RK{ \DT f_{\theta}(X)}$. To the best of our knowledge, this quantity appears for the first time in conditions of local parameter identifiability introduced by \citet{bona2022local}. It appears independently a few months later, as the core quantity of a study dedicated to the geometric analysis of neural networks carried out by \citet{grigsby2022functional}. In the latter article, this quantity is called the \lq{}batch functional dimension\rq{} and we will call it \lq{}\NAME{}\rq{} in this article.

Because the input space of $\DT f_{\theta}(X)$ is always $\RR^{E}\times \RR^B$, the quantity $\RK{ \DT f_{\theta}(X)}$ is upper bounded by the number of parameters $|E| + |B|$. 
Furthermore, as formalized by \citet{grigsby2022functional}, because of the invariance by positive rescaling, see the definition and discussion of the relation $\sim_s$ in Section \ref{ReLU networks-sec-main}, we even have $\RK{ \DT f_{\theta}(X)} \leq |E| + |B| - N_1 - \cdots - N_{L-1} $. %The latter inequality includes a rank deficiency due to the well-known positive rescaling invariance \citep{stock:tel-03208517, stock2022embedding, bona2023parameter, neyshabur2015path, yi2019positively}.
In fact, when $\RK{ \DT f_{\theta}(X)} = |E| + |B| - N_1 - \cdots - N_{L-1} $, under mild conditions on $\theta$, the network function is locally identifiable around $\theta$. That is, $f_{\theta}(X) = f_{\theta'}(X)$ and $\|\theta - \theta' \|$ small enough imply $\theta \sim_s \theta'$  (see \citealp{bona2022local}). 

Beyond the case of maximal rank value, i.e. $\RK{ \DT f_{\theta}(X)} = |E| + |B| - N_1 - \cdots - N_{L-1} $, leading to local identifiability, examples of non-identifiable neural networks and rank deficient parameters are in \citet{grigsby2022functional,bona2023parameter,grigsby2023hidden,sonoda2021ghosts}. Let us emphasize a simple example illustrating that several rank values can be achieved.
%%%%%%%%%%%%%%%%%
% Exemples 
\begin{exs}\label{example_1}
Consider $L\geq 3$, any neuron $v\in V_{\ell}$, for $\ell \in \{2, \ldots, L-1\}$, and $\theta\in\Par$ such that 
\begin{equation}\label{cond_deficit}
b_v <0 \qquad\mbox{and}\qquad w_{v'\rightarrow v} <0 \mbox{, for all } v'\in V_{\ell-1}. 
\end{equation}
Because of the ReLU activation function, for all $x\in\RR^{N_0}$ and all $v'\in V_{\ell-1}$, we have $(f^{\ell-1}_\theta (x))_{v'} \geq 0$, and \eqref{deff_l} and \eqref{cond_deficit} guarantee that $(f^{\ell}_\theta (x))_v = 0$. This holds for all $\theta$ in the open set defined by \eqref{cond_deficit}. In this set, the parameters $(w_{v'\rightarrow v})_{v'\in  V_{\ell-1}}$ and $b_v$ have no impact on $f_\theta(X)$, the corresponding partial derivatives $\frac{\partial f_\theta(X)}{\partial w_{v'\rightarrow v}}$  and  $\frac{\partial f_\theta(X)}{\partial b_{v}}$ are null, which leads to a rank deficiency of $\DT f_\theta(X)$. Going further, consider any $\theta\in\Par$. According to the above remark, to diminish $\RK{\DT f_\theta(X)}$, we can change the weights arriving to a given neuron $v$, and assign them negative values so that \eqref{cond_deficit} holds. We can redo  this operation to many neurons to diminish the rank further. As a conclusion to the example, many values of $\RK{\DT f_\theta(X)}$ are reached at different places in the parameter/input space.
\end{exs}

%%%%%%%%%%%%%%%
Let us conclude the section by showing that the quantity $\RK{ \DT f_{\theta}(X) }$ is invariant with respect to the positive rescaling and/or neuron permutation symmetries defined in Section \ref{ReLU networks-sec-main}. 

\begin{prop} \label{prop:rank:invariant:rescaling}
Consider any deep fully-connected ReLU network architecture $(E,V, \sigma_L)$.
Let $\theta, \widetilde{\theta} \in \Par$ such that $\theta \sim \widetilde{\theta}$. Then, for any $n\in\NN^*$ and $X\in\Xspace$,  $\DT f_{\theta}(X)$ is defined if and only if $\DT f_{\widetilde\theta}(X)$ is defined, and in that case we have
     \[\RK{ \DT f_{\widetilde\theta}(X) } = \RK{ \DT f_{\theta}(X) }.
     \]
\end{prop}

The proof of the proposition is in Appendix \ref{app:proof:rank:invariant:rescaling}. 

The invariance in Proposition~\ref{prop:rank:invariant:rescaling} is a benefit of the complexity measure $\RK{ \DT f_{\theta}(X) }$. The invariance will also hold for the regularity criterion and the notion of flatness introduced in the next section.

On the contrary, it does not hold for the local flatness of the empirical risk function studied by \citet{haddouche2025pac,cha2021swad,foret2021sharpnessaware,hochreiter1997flat,keskar2017on}. This leads to undesired behaviors \citep{dinh2017sharp}. Similarly, complexity measures defined by norms \citep{bartlett2017spectrally,bartlett2020benign,golowich2018size,neyshabur2015norm} are not invariant to positive rescalings\footnote{For both flatness and norms, it is, of course, possible to consider the minimum of the complexity criterion over the equivalence class of a $\theta$ element. However, this is an additional burden that does not correspond to the practice.}.

%%%%%%%%%%%%%%%%%%%%%%%%%%%%%%%%%%%%%%%%%%%%%%%%%%%%%%%%%%%%
\section{Geometry-Induced Regularization and Minima Flatness} \label{section:geometric:interpretation}

 In this section, we describe the consequences of Theorem \ref{theorem:constant:rank}. We formalize its theoretical implications in Sections \ref{theo_cons_sec}, \ref{implicit-regule-theory} and \ref{subsection:minima:flatness}, present a concrete example in \Cref{example:image:preimage:dimdeux}, and demonstrate their impact on optimization trajectories in \Cref{implicit_regul-sec}.

\subsection{Geometrical Interpretation of \texorpdfstring{\Cref{theorem:constant:rank}}{Theorem 1}}\label{theo_cons_sec}

The next corollary is a straightforward consequence of the constant rank theorem and Theorem \ref{theorem:constant:rank} (see \Cref{appendix-constant-rank-thm}). The corollary is illustrated by an example in \Cref{example:image:preimage:dimdeux} and \Cref{dessin-exemple}.

\begin{coro}\label{coro-optim}
    Consider any deep fully-connected ReLU network architecture $(E,V, \sigma_L)$.

    For any $n\in\NN^*$, $X\in\Xspace$, $j\in\lb1,p_X\rb$ and $\theta\in\cU_j^X$, there exists $\varepsilon_{X,\theta} >0$ such that
\begin{itemize}
    \item the {\it local image set} 
    \[\{f_{\theta'}(X)\in\Yspace ~|~ \|\theta'-\theta\|<\varepsilon_{X,\theta}\}\]
        is a smooth manifold of dimension $\RK{ \DT f_{\theta}(X) }$;
    \item the {\it local pre-image set}
    \[\{ \theta'\in\Par~|~f_{\theta'}(X) = f_{\theta}(X) \mbox{ and } \|\theta'-\theta\|<\varepsilon_{X,\theta} \} \] 
    is a smooth manifold of dimension $|E|+|B| - \RK{ \DT f_{\theta}(X) }$.
\end{itemize}
\end{coro}

\subsection{Example} \label{example:image:preimage:dimdeux}
In Figure \ref{dessin-exemple} we show the sets $\widetilde \cU^X_j$ (left) and their images $f_{\widetilde \cU^X_j}(X) = \{f_\theta(X) ~|~ \theta \in \widetilde \cU^X_j\}$ (right), for $j\in \lb1,6\rb$, for a one-hidden-layer ReLU neural network ($L=2$) of widths $N_0=N_1=N_2=1$, with the identity activation function in the last layer. To simplify the notation, we denote the weights and biases $\theta = (w,v,b,c)\in\RR^4$ so that $f_\theta(x) = v \sigma(wx+b)+c$, for all $x\in\RR$. We consider $X=(0, 1, 2)\in\RR^{1\times 3}$ and 
\[f_\theta(X) = \bigl(v \sigma(b) + c ~,~ v \sigma(w+b) + c ~,~v \sigma(2w+b) + c\bigr).
\]
%\[
%f_\theta(X)^T = \left(\begin{array}{c}
%v \sigma(b) + c \\
%v \sigma(w+b) + c \\
%v \sigma(2w+b) + c
%end{array}\right).
%\]

For any $j\in \lb1,6\rb$, the set $\widetilde \cU^X_j$ depends on the activations in the hidden layer. These sets are separated by the hyperplanes $b=0$, $w+b=0$, $2w+b=0$. The conditions only depend on $w$ and $b$. We represent the projection of the sets $\widetilde \cU^X_j$ and the lines $b=0$, $w+b=0$, $2w+b=0$ in the plane $(w,b)$,  on the left of Figure \ref{dessin-exemple}.

Similarly, for any $j\in \lb1,6\rb$, the image set $f_{\widetilde \cU^X_j}(X) \subseteq \RR^3$ is invariant to translations by a vector $(c,c,c)$, for $c\in\RR$. On the right of Figure \ref{dessin-exemple}, we represent for all $j$ the intersection $\cV_j = f_{\widetilde \cU^X_j}(X)\cap \cP$ between the image set $f_{\widetilde \cU^X_j}(X)$ and the linear plane $\cP$ orthogonal to $(1,1,1)$, generated by the vectors $\frac{1}{\sqrt{6}}(1,1,-2)$ and $\frac{1}{\sqrt{2}} (-1,1,0)$. The calculations leading to the construction of the figure are in Appendix \ref{appendix-example}.

\begin{figure}[ht]
  \begin{center}
    \begin{tikzpicture}[scale=0.7,decoration=snake]  
    % Partie paramètre 
    \draw [->] (3.5,0) -- (3.5,7); 
    \draw [->] (0,3.5) -- (7,3.5);
    \node[] at (3.2,7.1) {$b$};
    \node[] at (7.1,3.2) {$w$};
    \draw[red,ultra thick] (1,6) -- (6,1);
    \node[red] at (7.2,1) {$w+b=0$};
    \draw[red,ultra thick] (1.75,7) -- (5.25,0);
    \node[red] at (5.25,-0.2) {$2w+b=0$};
    \draw[red,ultra thick] (0.5,3.5) -- (6.5,3.5);
    \node[red] at (0.5,3.2) {$b=0$};

    \node[] at (3,2) {\textcolor{green}{$\widetilde{\cU}^X_1$}};
    \node[] at (1.8,4.2) {\textcolor{cyan}{$\widetilde{\cU}^X_2$}};
    \node[] at (1.7,6) {\textcolor{black}{$\widetilde{\cU}^X_3$}};
    \node[] at (4.5,5.5) {\textcolor{magenta}{$\widetilde{\cU}^X_4$}};
    \node[] at (5.7,2.8) {\textcolor{red}{$\widetilde{\cU}^X_5$}};
    \node[] at (5.4,0.9) {\textcolor{orange}{$\widetilde{\cU}^X_6$}};
   
      % fin partie paramètre
      %Transition
    \node[] at (3.5,8.5) {$\theta=(w,v,b,c)\in\RR^4$};
    \node[] at (7,9) {$f_\theta(X)$};
    \draw[->] (6,8.5) -- (8.5,8.5) ;
    \node[] at (9.5,8.5) {$\RR^{1\times 3}$};
    \draw[->] (10.5,8.5) -- (13,8.5) ;
    \node[] at (11.8,9) {restrict to $\cP$};
    \node[] at (13.5,8.5) {$\RR^2$};
       
    \draw[dotted ] (9,0)--(9,7);
    % partie image
  
    \draw [->] (10,3.5) -- (17,3.5); 
    \draw [->] (13.5,0) -- (13.5,7);
    \node[] at (13.2,7.1) {$y$};
    \node[] at (17.1,3.2) {$x$};

    \draw[pattern=north east lines, pattern color=black, draw=none]
    (13.5,3.5) -- (11.5,6.964) -- (10.5,5.232) -- cycle ;
    \draw[pattern=north east lines, pattern color=black, draw=none]
    (13.5,3.5) -- (16.5,1.768) -- (15.5,0.036) -- cycle ;
    \draw[pattern=north east lines, pattern color=red, draw=none]
    (13.5,3.5) -- (10.5,5.232) -- (10.5,3.5)  -- cycle ;
    \draw[pattern=north east lines, pattern color=red, draw=none]
    (13.5,3.5) -- (16.5,1.768) -- (16.5,3.5) -- cycle ;
    \draw[magenta,line width=1mm] (10.5,5.232) -- (16.5,1.768);
    \node[magenta] at (17.4,1.3) {$x+\sqrt{3}y=0$};
    \draw[cyan,line width=1mm] (11.5,6.964) -- (15.5,0.036);
    \node[cyan] at (15.8,-0.2) {$\sqrt{3}x+y=0$};
    \draw[orange,line width=1mm] (10.5,3.5) -- (16.5,3.5);
    \node[orange] at (16.8,3.9) {$y=0$};
    \node[] at (13.5,3.5) {\textcolor{green}{{\Large $\bullet$}}}; % Fin partie image
    \end{tikzpicture}
    \caption{\label{dessin-exemple}Representation of the sets $\widetilde \cU^X_j$ in the space $(w,b)$ (left) and restriction to $\cP$ of the corresponding image sets $\{f_{\theta}(X) ~ | ~ \theta \in \widetilde \cU^X_j\}$, $j \in \lb 1,6 \rb$ (right). We have $r^X_1=1$, $r^X_2=2$, $r^X_3=3$, $r^X_4=2$, $r^X_5=3$, $r^X_6=2$. The image of $\widetilde \cU^X_1$ such that $r^X_1=1$ is reduced to $(0,0)$ (right). The images of the sets $\widetilde \cU^X_j$ with $r_j^X = 2$ (i.e. $j=2,4,6$) are represented with thick lines of their respective colors (right).
    The images of $\widetilde \cU^X_3$, with $r_3^X = 3$, and $\widetilde \cU^X_5$, with $r_5^X = 3$, are represented by dashed areas, with the corresponding colors (right). %Remark that we can check that for $j \in \lb 1,6 \rb$, $\widetilde \cU^X_j =  \cU^X_j$, that is the rank of the differential is constant on $\widetilde \cU^X_j$.
    }
  \end{center}
\end{figure}

Notice that, as a consequence of the forthcoming \Cref{shallow-case-thm}, for the architecture $(1,1,1)$, we have ${\widetilde \cU^X_j}={\cU^X_j}$, for all $j\in\lb1,p_X\rb$.

The example described in this section illustrates the configuration of the different sets introduced in the preceding sections. We will return to it in \Cref{implicit_regul-sec} to highlight the connection between the geometrical sets, geometry-induced regularization, and saddle-to-saddle dynamics.

\subsection{Geometry-Induced Regularization Statements}\label{implicit-regule-theory}

Below, we consider $n\in\NN^*$, $X\in\Xspace$ and a smooth learning objective $R:\RR^{N_L\times n} \longrightarrow \RR$. The latter may depend on outputs $Y\in\Yspace$ or other relevant problem-related information. For the sake of simplicity and generality, this dependence is not explicitly indicated in the notation. The learning problem is modeled by
\begin{equation}\label{learning_Pb}
\underset{\theta}{    \minimize}
~
R(f_\theta(X)).\tag{$P$}
\end{equation}

Denoting $\cL(\theta) = R(f_\theta(X))$, if $\theta \mapsto f_\theta(X)$ is differentiable at $\theta$, a consequence of the chain rule is that
\begin{equation}\label{orthogonality-condition-eq}
\nabla \cL(\theta) = 0 \quad \Longleftrightarrow \quad \nabla R(f_\theta(X)) \in \Range(\DT f_\theta(X))^\perp,
\end{equation}
where $\Range(\DT f_\theta(X))^\perp$ denotes the orthogonal complement of the image of the linear map $\DT f_\theta(X) : \RR^{|E|+|B|} \rightarrow \RR^{N_L \times n}$.
In particular, if $\theta \in \cU_j^X$ for some $j \in \lb 1 , p_X \rb$, then by Corollary \ref{coro-optim}, the local image set $\{f_{\theta'}(X)\in\Yspace ~|~ \|\theta'-\theta\|<\varepsilon_{X,\theta}\}$ is a smooth manifold of dimension $\RK{\DT f_\theta(X)}$, and the direction of its tangent plane at $f_\theta(X)$ is given by $ \Range(\DT f_\theta(X))$. In that case, the equivalence \eqref{orthogonality-condition-eq} means that $\theta$ is a critical point of $\cL$ if and only if $\nabla R(f_\theta(X))$ is orthogonal to the local image set. This property lies at the heart of the geometry-induced regularization formalized in the statements below.

To formulate the regularization statements, we consider the upper semi-continuous extension $\dim^+: (\Par)\times \Xspace\longrightarrow \RR$ and the lower semi-continuous extension $\dim^-: (\Par)\times \Xspace\longrightarrow \RR$ of $(\theta,X) \longmapsto \RK{ \DT f_{\theta}(X) }$. More precisely, we define for all $(\theta,X)\in(\Par)\times \Xspace$
\begin{equation}\label{dim_def}
\left\{\begin{array}{l}
\dim^+(\theta,X) = \underset{\epsilon \rightarrow 0}{\lim}
~ ~
\underset{j : B(\theta,\epsilon) \cap \cU_j^X \neq \emptyset}{\max} ~r_j^X, \\
\dim^-(\theta,X) = \underset{\epsilon \rightarrow 0}{\lim}
~ ~
\underset{j : B(\theta,\epsilon) \cap \cU_j^X \neq \emptyset}{\min}~ r_j^X .
\end{array}\right.
\end{equation}
Of course, when there exists $j$ such that $\theta \in\cU_j^X $, since $\cU_j^X$ is open, $\dim^-(\theta,X) =\dim^+(\theta,X) = \RK{ \DT f_{\theta}(X) }$. We remind that, according to Theorem \ref{theorem:constant:rank}, the set $\left( \cup_{j=1}^{p_X} \cU^X_j \right)^c$, where the \NAME{} is extended, is closed with Lebesgue measure $0$.

Corollary \ref{regul-implicit-sec-1} establishes a connection between the critical points of \eqref{learning_Pb} 
and those satisfying the \KKT{} (KKT) conditions of the regularized problems
\begin{equation}\label{learning_pb_regul}
        \underset{\theta : \dim^-(\theta,X) \leq k}{\minimize} 
        ~
    R(f_\theta(X)), \tag{$P_k$}
\end{equation}
for $k\in\NN$.

\begin{coro}\label{regul-implicit-sec-1}
    Consider any deep fully-connected ReLU network architecture $(E,V, \sigma_L)$.

    Consider any $n\in\NN^*$, $X\in\Xspace$, any smooth learning objective $R:\RR^{N_L\times n} \longrightarrow \RR$, and $\theta^*\in \bigcup_{j=1}^{p_X} \cU_j^X$. We denote $k= \RK{ \DT f_{\theta^*}(X) }$. 
    \[\theta^* \mbox{ is a critical point of }(P) \qquad\Longleftrightarrow \qquad  (\theta^*,1) \mbox{ satisfies the KKT conditions of }  (P_k).
    \]
\end{coro}
The proof is straightforward, but we provide the details for completeness in \Cref{proof-sec4-app-1}. When considering the points satisfying the KKT condition, we cannot consider points at which the function defining the constraint is discontinuous. This leads to considering $\theta^*\in \bigcup_{j=1}^{p_X} \cU_j^X$. This problem does not arise when, as in Corollary \ref{regul-implicit-sec-2},  establishing a connection between the local minimizers of \eqref{learning_Pb} 
and the local minimizers of \eqref{learning_pb_regul}.

\begin{coro}\label{regul-implicit-sec-2}
    Consider any deep fully-connected ReLU network architecture $(E,V, \sigma_L)$.

    Consider any $n\in\NN^*$, $X\in\Xspace$, any smooth learning objective $R:\RR^{N_L\times n} \longrightarrow \RR$, and $\theta^*\in \Par$. We denote $k= \dim^+(\theta^*,X)$. We have
    \[\theta^* \mbox{ is a local minimizer of }(P) \qquad\Longleftrightarrow \qquad  \theta^* \mbox{ is a local minimizer of }  (P_k).
    \]
    and 
    \[\theta^* \mbox{ is a saddle point of }(P) \qquad\Longleftrightarrow \qquad  \theta^* \mbox{ is a saddle point of }  (P_k).
    \]
\end{coro}
The proof is straightforward, but we provide the details for completeness in \Cref{proof-sec4-app-2}. The above two corollaries show that the limit points of first-order algorithms all exhibit a different trade-off between the minimization of $R(f_\theta(X))$ and $\dim^-(\theta,X)$. The trade-off depends on the local minimizer, which in turn is determined by the initialization and the optimization algorithm. This stands in sharp contrast to the common practice in inverse problems, where the regularization parameter is typically chosen by the user or tuned according to an ad hoc criterion. We empirically observe the dependence of the regularization parameter on the initialization in \Cref{implicit_regul-sec}. We will also observe in the experiments of  \Cref{expe-rank_diminu} that the \NAME{} $ \RK{ \DT f_{\theta}(X) }$ tends to decrease during training.

To understand the practical effect of the regularization induced by the geometry, we detail in \Cref{shallow-sec} the properties shared by the functions $f_\theta$ when $\theta$ satisfies $ \RK{ \DT f_{\theta}(X) } \leq k$, for a given $k\in\NN$, in the case of shallow networks. The effect of the regularization is empirically put to evidence in \Cref{recovery-CPL-sec}.

\subsection{Minima's Flatness} \label{subsection:minima:flatness}
As in the previous section, we consider $n\in\NN^*$, $X\in\Xspace$ and a smooth learning objective $R:\RR^{N_L\times n} \longrightarrow \RR$.

A direct consequence of \Cref{coro-optim} is that any local minimizer $\theta \in \cup_{j=1}^{p_X} \cU^X_j$ of \eqref{learning_Pb} is dimension $\bigl(|E| + |B| - \RK{ \DT f_{\theta}(X) } \bigr)$ flat, as defined in the following definition.
\begin{df}\label{flat-def}
A local minimizer $\theta$ of \eqref{learning_Pb} is said to be dimension $k$ flat, for $k \in \NN$, if and only if there exist $\varepsilon > 0$ and a smooth manifold $\calM \subset \Par$ of dimension $k$ such that $\theta \in \calM$, and every $\theta' \in \calM \cap B(\theta, \varepsilon)$ is also a local minimizer of \eqref{learning_Pb}.
\end{df}
This property is illustrated in \Cref{flat-minima-fig} using a simple scalar function on $\RR^2$, unrelated to deep learning.
    \begin{figure}[ht]
    \centering
    \includegraphics[width=0.45\linewidth]{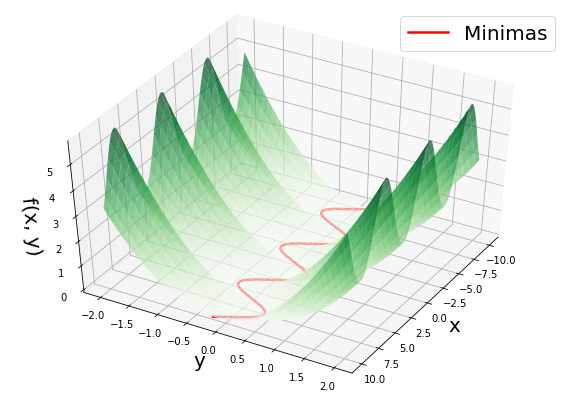}
    \caption{Illustration of the {\em dimension $k$ flat minima} property. The red line represents the smooth manifold of dimension $1$ formed by all local minima.}\label{flat-minima-fig}
    \end{figure}

We consider a minimizer to be flatter when $k$ is larger, corresponding to a smaller value of the regularity criterion $\RK{\DT f_{\theta}(X)}$. With these definitions, flatter minima naturally correspond to more regularized neural networks. This notion of flatness differs from the one based on the Hessian of the objective function, as studied in \citet{haddouche2025pac,keskar2017on, foret2021sharpnessaware, cha2021swad, hochreiter1997flat}. The lack of invariance of Hessian-based flatness with respect to the natural symmetries of neural network parameterizations has led to counterexamples demonstrating that it fails to capture the phenomenon of good generalization \citep{dinh2017sharp}. By contrast, as demonstrated in \Cref{prop:rank:invariant:rescaling}, \Cref{flat-def} benefits from invariance to the natural symmetries of ReLU neural networks.

As for the Hessian-based notion of flatness, for $k$ large, escaping dimension $k$ flat minima is time-consuming for stochastic algorithms. For instance, for the stochastic gradient algorithm, the gradient noise will remain orthogonal to $\calM$, which does not favor the exploration of the flat valley. This should lead to the over-representation of dimension $k$ flat minima, with $k$ large, among the outputs of minimization algorithms.

\subsection{Geometry-Induced Regularization on the Example}\label{implicit_regul-sec}

To illustrate the {\it geometry-induced regularization} of \Cref{implicit-regule-theory}, we compute a series of optimization trajectories on the example of \Cref{example:image:preimage:dimdeux}. The example provides the set of input values $X = (0, 1, 2)$. By selecting a corresponding target output vector \( Y = (y_1, y_2, y_3) \), which can be freely chosen, the network can be optimized by minimizing the MSE between its predictions and the targets, i.e. by minimizing
\begin{equation}\label{pkregihueitn}
R(f_{\theta}(X)) = \frac{1}{3} \Big( (f_\theta(x_1) - y_1)^2 + (f_\theta(x_2) - y_2)^2 + (f_\theta(x_3) - y_3)^2\Big).
\end{equation}

In \Cref{lim-points-sec,stat-example-sec}, we empirically examine where images of the limit points accumulate and interpret these findings in light of the theoretical results of \Cref{implicit-regule-theory}. In \Cref{saddle-to-saddle-sec}, we illustrate how the geometry-induced properties of the landscape give rise to saddle-to-saddle dynamics.

\subsubsection{Limit Point Locations}\label{lim-points-sec}

We make the (arbitrary) choice \( Y = (0, 1, 3) \) as our target output. It is reachable in the sense that there exists  $\theta$ such that $f_\theta(X) = Y$.  To explore the diversity of learning behaviors, we compute the optimization trajectories for $10 ~000$ random initializations. For each trajectory, the parameters $(w, v, b, c)$ are initialized independently using a normal distribution \( \mathcal{N}(0, 1) \).  
The network is then trained via (non-stochastic) gradient descent with a learning rate $\gamma = 0.1$, over $300$ iterations.  

\begin{figure}[ht]
    \centering

    \begin{subfigure}{\linewidth}
        \centering
        \includegraphics[scale=0.5]{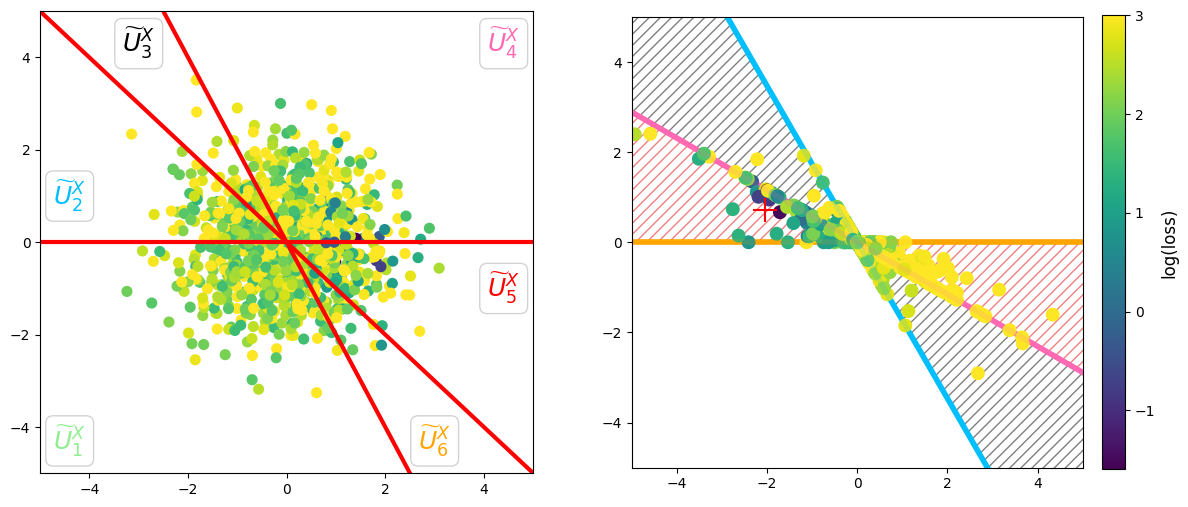}
        \caption{Initialization}
%        \label{fig:top}
    \end{subfigure}

    \vspace{15pt} % optional space between figures

    \begin{subfigure}{\linewidth}
        \centering
        \includegraphics[scale=0.5]{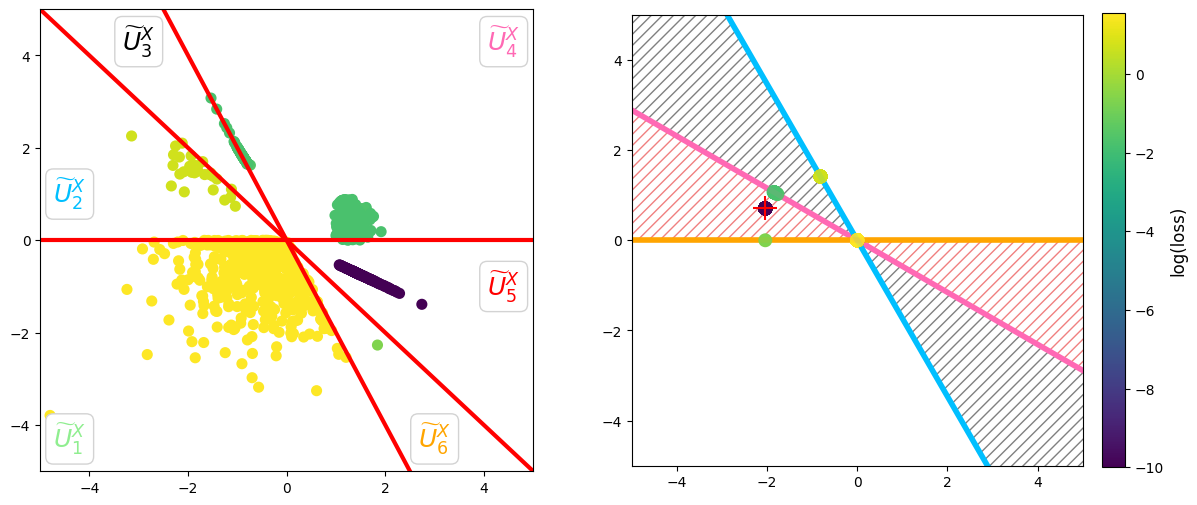}
        \caption{Iteration 300}
%        \label{fig:bottom}
    \end{subfigure}

    \caption{Evolution of the parameters for $1~000$ different initializations, the sets $\widetilde{\cU}^X_j$ and their images. The parameters are represented in the $(w,b)$ space (left), and their corresponding (projected) images are represented in the output set (right), both at initialization (a) and after 300 iterations of gradient descent (b). The color of the points indicates the value of the objective $R(f_\theta(X))$. }
    \label{fig:exp1}
\end{figure}
In \Cref{fig:exp1}, we reproduce \Cref{dessin-exemple}, over which we plot the different parameters of the experiment. In \Cref{fig:exp1} (a), we represent the parameters at initialization; in Figure \ref{fig:exp1} (b),  they are represented after $300$ iterations. As in \Cref{dessin-exemple} and as described in Section~\ref{example:image:preimage:dimdeux}, each parameter vector \( \theta = (w, v, b, c) \in \mathbb{R}^4 \) is represented as follows: on the left, by its projection onto the \( (w, b) \)-plane; on the right, by the projection of \( f_\theta(X) \) onto the plane \( \mathcal{P} \). Still on the right of \Cref{fig:exp1}, the (projected) target $Y$ is represented as the red cross. The color of each point $\theta$ corresponds to the value of the objective $R(f_{\theta}(X))$. For clarity, we only plot $1~ 000$ parameters out of the $10~000$. The others are used for the estimates reported in \Cref{stat-example-sec}.

While at initialization the outputs are scattered (\Cref{fig:exp1} (a), right), after training, they are concentrated in the vicinity of $5$ different limit-points (\Cref{fig:exp1} (b), right). These limit points coincide with the orthogonal projections, denoted $P_jY$, for $j\in\lb1,6\rb$, of $Y$ onto the closure of $f_{\cU^X_j}(X)$ as defined by
\begin{equation}\label{ervhtb}
\underset{Y'\in \CLOSURE{f_{\cU^X_{j}}(X)}}{
\minimize} R(Y'),
\end{equation}
where  $\CLOSURE{.}$ denotes the closure of a set.
Notice that, for the chosen $Y$, $P_3Y = P_4Y$. 

Let us explain this empirical observation in the light of \Cref{implicit-regule-theory}. To do so we study separately $r_j^X \leq 2$ and $r_j^X =3$.

%%%%%%%%%%%
 Recall that, as mentioned at the end of Section \ref{example:image:preimage:dimdeux}, the forthcoming \Cref{shallow-case-thm} establishes that for the architecture of the example we have $\widetilde{\cU}^X_j = \cU^X_j$, for all $j \in \lb 1 , 6 \rb$. Let $j \in \lb 1 , 6 \rb$, and let $\theta \in \cU^X_j$. 

If $r_j^X \leq 2$, the analysis of Section \ref{example:image:preimage:dimdeux} shows that the image $f_{\cU^X_j}(X) = \{ f_\theta(X) \mid \theta \in \cU_j^X \}$ is a linear subspace of $\RR^3$. Thus, (for instance) the orthogonality condition \eqref{orthogonality-condition-eq} implies that $\theta \in \cU^X_j$ is a critical point of $\cL : \theta \mapsto R(f_\theta(X))$ if and only if $\nabla R(f_\theta(X))$ is orthogonal to $f_{\cU^X_j}(X)$. By definition of the MSE and since $f_{\cU^X_j}(X)$ is a vector space, the only point of $f_{\cU^X_j}(X)$ at which the orthogonality is satisfied is the orthogonal projection of $Y$ onto $f_{\cU^X_j}(X)$. This proves that the set of critical points in $\cU_j^X$ is exactly the set of $\theta \in \cU^X_j$ such that $f_\theta(X) = P_jY$. It is then easy to see that each of these critical points $\theta \in \cU_j^X$ is actually a local minimizer of $\cL$, since the orthogonal projection minimizes the distance. Notice that the image $f_\theta(X)$ of a local minimizer $\theta \in \cU_j^X$ is isolated, being equal to $P_jY$. However, multiple $\theta \in \cU_j^X$ can lead to the same value $P_jY$.

If $r_j^X =3$, then $\DT f_\theta(X)$ has full image rank, so the orthogonality condition means that $\nabla R(f_\theta(X)) = 0$. This can only happen if $f_\theta(X) = Y$. Thus, for $\theta \in \cU_j^X$, $\theta$ is a critical point of $\cL$ if and only if it is a global minimizer and $\cL(\theta)=0$. This occurs only for $j=5$. When $j=3$, $\cU_3^X$ does not contain any critical point. This leads to an accumulation of limit points in the vicinity of the boundary between $\cU_3^X$ and $\cU_4^X$ whose images are close to $P_3Y=P_4Y$.

What precedes allows us to characterize all the critical points $\theta$ of \eqref{learning_Pb} when $\theta \in \bigcup_{j=1}^6 \cU^X_j$, which are always local minimizers. Similarly, let $\theta \in \bigcup_{j=1}^6 \cU^X_j$ and consider now problem \eqref{learning_pb_regul}. If $k=1$, then $\theta$ is a minimizer of \eqref{learning_pb_regul} if and only if $f_\theta(X) = P_1Y$. If $k=2$, then $\theta$ is a minimizer of \eqref{learning_pb_regul} if and only if $f_\theta(X) \in \{ P_1Y, P_2 Y, P_4Y, P_6Y\}$. If $k=3$, then $\theta$ is a minimizer of \eqref{learning_pb_regul} if and only if $f_\theta(X) \in \{ P_1Y, P_2 Y, P_4Y, P_6Y, Y\}$. The correspondence between the sets of critical points illustrates the statements of \Cref{implicit-regule-theory}.

\subsubsection{Limit-point Location Depending on the Initialization}\label{stat-example-sec}

\begin{table}[ht]
\centering
     {\renewcommand{\arraystretch}{2.5}
     \begin{tabular}{|c|cccccc|}
    \hline  \textbf{Region} & \makecell{$\boldsymbol{\widetilde{\mathcal{U}}_1^X}$}   & \makecell{$\boldsymbol{\widetilde{\mathcal{U}}_2^X}$} & \makecell{$\boldsymbol{\widetilde{\mathcal{U}}_3^X}$} & \makecell{$\boldsymbol{\widetilde{\mathcal{U}}_4^X}$} & \makecell{\mincolor{$\boldsymbol{\widetilde{\mathcal{U}}_5^X}$}} & \makecell{$\boldsymbol{\widetilde{\mathcal{U}}_6^X}$} \\
    \hline Dimension & 1 & 2 & 3 & 2 & \mincolor{3} & 2 \\
    \hline $P(\theta_0 \in \widetilde{\mathcal{U}}_j^X)$ & 0.33 & 0.12 & 0.05 & 0.32 & \mincolor{0.13} & 0.05 \\
    \hline $P(\theta_{300} \in \widetilde{\mathcal{U}}_j^X)$  & 0.50 & 0.03 & \semitransp{0.00} & 0.29 & \mincolor{0.18} & \semitransp{0.00} \\
    \hline \makecell{$P\left(\theta_{300} \in \widetilde{\mathcal{U}}_j^X \ | \ \theta_0 \in \widetilde{\mathcal{U}}_1^X\right)$\\ \scriptsize 3317 initializations}  & 1.00 & \semitransp{0.00} & \semitransp{0.00} & \semitransp{0.00} & \semitransp{\mincolor{0.00}} & \semitransp{0.00}  \\
    \hline \makecell{$P\left(\theta_{300} \in \widetilde{\mathcal{U}}_j^X \ | \ \theta_0 \in \widetilde{\mathcal{U}}_2^X\right)$\\ \scriptsize 1236 initializations} & 0.35 & 0.20 & \semitransp{0.00} & 0.45 & \semitransp{\mincolor{0.00}} & \semitransp{0.00}  \\
    \hline \makecell{$P\left(\theta_{300} \in \widetilde{\mathcal{U}}_j^X \ | \ \theta_0 \in \widetilde{\mathcal{U}}_3^X\right)$\\ \scriptsize 481 initializations} & 0.11 & 0.01 & \semitransp{0.00} & 0.84 & \mincolor{0.02} & \semitransp{0.00} \\
    \hline \makecell{$P\left(\theta_{300} \in \widetilde{\mathcal{U}}_j^X \ | \ \theta_0 \in \widetilde{\mathcal{U}}_4^X\right)$\\ \scriptsize 3171 initializations} & 0.18 & 0.01 & \semitransp{0.00} & 0.59 & \mincolor{0.22} & \semitransp{0.00} \\
    \hline \makecell{$P\left(\theta_{300} \in \widetilde{\mathcal{U}}_j^X \ | \ \theta_0 \in \widetilde{\mathcal{U}}_5^X\right)$\\ \scriptsize 1291 initializations} & 0.29 & \semitransp{0.00} & \semitransp{0.00} & 0.06 & \mincolor{0.65} & \semitransp{0.00} \\
    \hline \makecell{$P\left(\theta_{300} \in \widetilde{\mathcal{U}}_j^X \ | \ \theta_0 \in \widetilde{\mathcal{U}}_6^X\right)$\\ \scriptsize 504 initializations} & 0.40 & \semitransp{0.00} & \semitransp{0.00} & 0.02 & \mincolor{0.58} & 0.01\\
    \hline
\end{tabular}}
    \caption{Distribution of the parameters in the different regions at initialization and after training, as well as distribution after training conditionally to the initialization region. The computations are based on 10 000 different optimization trajectories started with a random initialization.}
    \label{fig:exp1-distribution}
\end{table}

Based on the $10~ 000$ trajectories, we compute and provide in \Cref{fig:exp1-distribution}, both at initialization and after training, the distribution of the parameters in the different regions. We also compute the distribution of the parameters after training conditionally on the initial region. 

 As a first observation of the table, the probability of being initialized in a region differs from region to region, due to diverse sizes. Note that by symmetry around zero the regions go two by two: $\widetilde \cU_1^X$ and $\widetilde \cU_4^X$ have the same shape (and thus approximately equal initialization probabilities in the table), and similarly for the pairs $(\widetilde \cU_2^X,\widetilde \cU_5^X)$ and $(\widetilde \cU_3^X,\widetilde \cU_6^X$).

In \Cref{fig:exp1-distribution}, the blue column corresponds to the region containing the global minimizers, $\widetilde \cU_5^X$. The table illustrates that the region of initialization has a strong impact on the final parameter. Indeed, we see that all the points starting inside $\widetilde \cU_1^X$ remain in $\widetilde \cU_1^X$. This is because, once $\theta$ is in $\widetilde \cU_1^X$, only the partial derivative with regard to $c$ is non-zero and only $c$ is optimized. This does not permit getting out of $\widetilde \cU_1^X$. On the contrary, none of the trajectories finishes its course in $\widetilde \cU_3^X$ (which does not contain any critical point). Most trajectories starting in $\widetilde \cU_3^X$ converge to a limit-point in $\widetilde \cU_4^X$, but some of them manage to reach $\widetilde \cU_5^X$. None of the trajectories starting in $\widetilde \cU_2^X$ manages to reach a global minimizer. On the contrary, starting from $\widetilde \cU_5^X$ or $\widetilde \cU_6^X$ leads to a probability of converging to a global minimizer greater than $0.5$. The region $\widetilde \cU_4^X$ is an intermediary case where the chance of converging to a global minimizer is non-negligible,
but below $0.5$,
being equal to $0.22$. Surprisingly, many trajectories starting inside $\widetilde \cU_5^X$ finish their course in $\widetilde \cU_1^X$. The only two regions that have more points after the training than before are the region containing the global minimizer, $\widetilde \cU_5^X$, as well as the region of lowest dimension, $\widetilde \cU_1^X$, from which it is impossible to escape.

\subsubsection{Dimension \texorpdfstring{$k$}{k} Flat Minima}

Regarding the pre-image, on \Cref{fig:exp1} (b), left, we remark that for each $j\neq 6$, there are many limit-points $\theta^*$ in ${\widetilde \cU^X_j}$. Since they have the same color, their images on the right of \Cref{fig:exp1} (b) are essentially the same. For $j=5$, we have $r_j^X=3$ and the limit-points differ by a positive-rescaling. This is coherent with the theoretical results in \cite{bona2022local}. Notice that this also holds for the limit-points $\theta^*$ on the boundary between ${\widetilde \cU^X_3}$ and ${\widetilde \cU^X_4}$. These points may correspond to trajectories whose iterates primarily lie in ${\widetilde \cU^X_3}$ but ultimately converge to ${\widetilde \cU^X_4}$. For $j\in\{2, 4\}$, for which $r_j^X=2$, we see groups of limit-points. For $j=6$, the basin of attraction of the local minimizer of \eqref{pkregihueitn} is small and only one of the displayed experiments converges in ${\widetilde \cU^X_{6}}(X)$. For $j=1$, for which $r_j^X=1$, only $c$ is optimized and the projected limit-points coincide with those in \Cref{fig:exp1} (a), left. For $j \in \{1,2,4,6\}$, the sets of limit points in the parameter space are consistent with the fact that they are projections onto the $(a,c)$-plane of points lying on manifolds in $\RR^4$. This description illustrates the statement in \Cref{subsection:minima:flatness}.

\subsubsection{Saddle-to-saddle dynamics from a geometric perspective}\label{saddle-to-saddle-sec}

As illustrated in \Cref{fig:saddle-to-saddle}, which shows a trajectory, its image and the corresponding objective throughout the optimization process, the geometry of the neural network can provide insights into the \textit{saddle-to-saddle} behavior of the loss during training. For this experiment, we consider the same setting as before except that we take $Y=(1,0,5)$. In \Cref{fig:saddle-to-saddle}, top, we observe that the parameters are initialized in $\widetilde{\cU}^X_4$ (at the gray square), and go successively to $\widetilde{\cU}^X_5$ and to $\widetilde{\cU}^X_6$. The trajectory on the top figure allows to understand the bottom figure: after a first decrease in the loss, we observe a plateau. The latter corresponds to the approach of the set of critical points $\theta$ such that $f_\theta(X)= P_4 Y$. When the parameter trajectory reaches $\widetilde{\mathcal{U}}^X_5$, its image can evolve within a higher-dimensional set, leading to a second drop in the objective function. Then, when the parameters move from $\widetilde{\cU}^X_5$ to $\widetilde{\cU}^X_6$ and evolve inside $\widetilde{\cU}^X_6$, we observe another plateau of the objective. 

In this experiment, we illustrate that the transitions between regions can unlock new degrees of freedom, leading to sudden decreases of the objective. This saddle-to-saddle behavior has been observed and analyzed, for example, in \citet{jacot2021saddle, boursier2022gradient, abbe2023sgd, pesme2023saddle}. 

\begin{figure}[ht]
\centering
     \begin{subfigure}{\linewidth}
        \centering
        \includegraphics[scale=0.45]{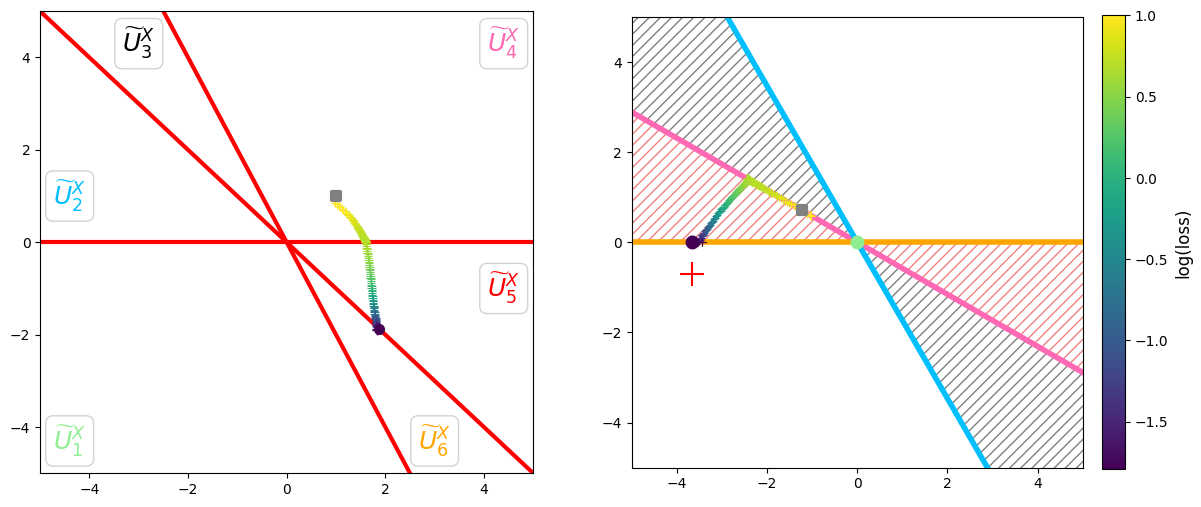}
%        \label{fig:top}
    \end{subfigure}

    \vspace{15pt} % optional space between figures

 \begin{subfigure}{\linewidth}
        \centering
        \includegraphics[scale=0.45]{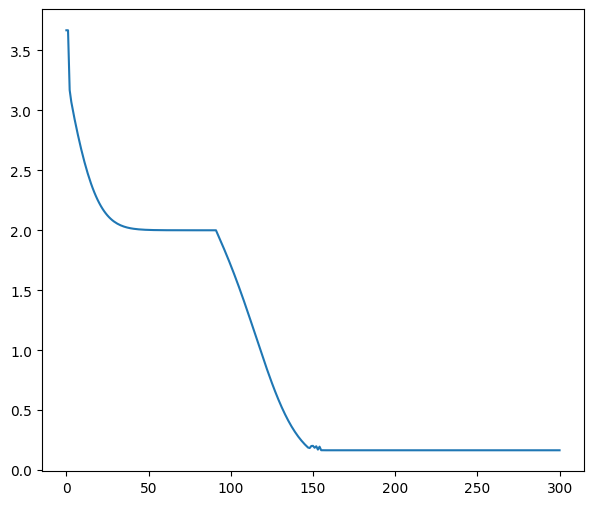}
%        \label{fig:top}
    \end{subfigure}
    
    \caption{Illustration of the saddle-to-saddle phenomenon: Example of a trajectory of the parameters in the $(w,b)$ space (top left), the corresponding projected outputs (top right), and the evolution of the objective (bottom).  }
    \label{fig:saddle-to-saddle}
\end{figure}

\section{Effect of the Regularization in the Shallow Case}\label{shallow-sec}

\subsection{Theoretical Analysis}
In this section, we consider a shallow network of widths $(1, N_1, 1)$, with $\sigma_2=Id$, and provide a simple formula for $ \RK{ \DT f_{\theta}(X) }$ that we interpret. In particular, we denote, for all $X$ and $\theta$,
\begin{equation}\label{def_A}
    \cA(X,\theta) = \{\delta \in\{0,1 \}^{N_1} \mid \mbox{there exists }i\in\lb1,n\rb \mbox{, such that } a(x^{(i)}, \theta) = \delta \}.
\end{equation}
The set $\cA(X,\theta)$ encompasses all activation patterns \lq{}perceived\rq{} by $X$. In the next theorem, we show that $|\cA(X,\theta)|$ is connected to $\RK{ \DT f_{\theta}(X) }$, thereby illustrating the practical implications of the geometry-induced regularization discussed in Section~\ref{section:geometric:interpretation}.

The order of the examples has no influence on $ \RK{ \DT f_{\theta}(X) }$. To simplify notations, we assume, without loss of generality, that the examples of $X=(x^{(1)},x^{(2)}, \ldots,  x^{(n)})\in\RR^{1\times n}$ are distinct and ordered: 
\begin{equation}\label{x_ordered_main_part}
x^{(1)} < x^{(2)} < \cdots <  x^{(n)}.
\end{equation}

We denote for $i\in\lb 1, n\rb$,
\begin{equation}\label{def_bfe_0}
	\bfe_{i} =  \sigma\big( X -x^{(i)} \One\big)\in\RR^{1\times n},
	\quad \text{and} \quad
	\bfe_{n+i} = \sigma\big( x^{(i)}\One - X\big)\in\RR^{1\times n},
\end{equation}
where all the components of $\One\in\RR^{1 \times n}$ equal $1$. We have, for all $i \in\lb 1, n\rb$,
\begin{equation}\label{def_bfe}
\left\{\begin{array}{l}
    \bfe_{i} = (0,\dots, \underset{\underset{i}{\uparrow}}{0},  x^{(i+1)} - x^{(i)}, \dots, x^{(n)}-x^{(i)}), \\  
	\bfe_{n+i} = (x^{(i)}-x^{(1)}, \dots, x^{(i)}-x^{(i-1)},  \underset{\underset{i}{\uparrow}}{0}, \dots, 0).
    \end{array}\right.
\end{equation}
We also set $\bfe_0 = \bfe_{2n}$. Notice that, by definition, $\bfe_{n} =\bfe_{n+1} = 0$. 

We also define, for all $i\in\lb1,n\rb$,
\begin{equation}\label{defDesOnes}
\One_{i} = (0,\dots,0, \underset{\underset{i}{\uparrow}}{1},  \dots, 1)\in\RR^{1\times n} \quad \mbox{and}\quad
	\One_{n+i} = (1, \dots,1, \underset{\underset{i}{\uparrow}}{0}, \dots, 0)\in\RR^{1\times n}.
\end{equation}
Before stating the following theorem, we remind that the activation patterns $a(X,\theta)$ are defined in \Cref{activation-patterns-sec}.

\begin{thm}\label{shallow-case-thm}
    Consider any deep fully-connected ReLU network architecture $(E,V, Id)$, with $L=2$ and $N_0=N_2=1$. Consider $n\in\NN^*$, and a sample $X=(x^{(1)},x^{(2)}, \ldots,  x^{(n)})\in\RR^{1\times n}$ satisfying \eqref{x_ordered_main_part}. 

    For any $j\in\lb1,p_X\rb$, there exists $\alpha \in \{1,\dots,2n\}^{N_1}$ such that for all $\theta \in \widetilde\cU_j^X$ and all $k\in \lb1,N_1\rb$, $a(X,\theta)_{k,:}= \One_{\alpha_k} $, and 
    \begin{equation}\label{formule_rank_shallow}
    \RK{ \DT f_{\theta}(X) } = \RK{\One, \bfe_{\alpha_1 -1},\bfe_{\alpha_1 }, \ldots ,\bfe_{\alpha_{N_1} -1},\bfe_{\alpha_{N_1} }  }.
    \end{equation}
     As a consequence, $\widetilde\cU_j^X = \cU_j^X$ and for all $\theta \in \cU_j^X$
    \begin{equation}\label{upper_bound_rank_shallow}
    \frac{1}{2} |\cA(X,\theta)| \leq 
    \RK{ \DT f_{\theta}(X) } \leq 2 |\cA(X,\theta)|.
    \end{equation}
\end{thm}
The proof of the theorem is in \Cref{shallow-case-app}. \Cref{shallow-case-app} also provides a detailed characterization of the geometry of the image set $\{f_\theta(X)~|~ \theta \text{ varies} \}$ for the architecture $(N_0, N_1, N_2) = (1, N_1, 1)$, along with Theorem~\ref{shallow-case-thm-accurate}, which offers more precise—albeit less interpretable—bounds.

The quantity $\RK{\One, \bfe_{\alpha_1 -1},\bfe_{\alpha_1 }, \ldots ,\bfe_{\alpha_{N_1} -1},\bfe_{\alpha_{N_1} }  }$ counts the effective patterns $\alpha_k$. Typically, if two neurons of the hidden-layer are activated by the same set of examples, then according to \eqref{formule_rank_shallow}, the \NAME{} is the same as if the two neurons are collapsed. This implies that there are groups of neurons of the hidden-layer which are activated by the same set of examples. The geometry-induced regularization described in Corollary \ref{regul-implicit-sec-1} and Corollary \ref{regul-implicit-sec-2} favors the \lq alignment\rq{} of the neurons, such as put to evidence in \cite{boursier2024early,boursier2024simplicitybiasoptimizationthreshold}.

Also, because of $|\cA(X,\theta)| $ in \eqref{upper_bound_rank_shallow}, the \NAME{} diminishes when $\theta$ varies in such a way that more activation patterns $a(x^{(i)},\theta)$ are equal. That is when the number of linear regions of $f_\theta$ containing examples of $X$ diminishes. For instance, adding a new example with the same activation pattern as an example already in $X$ does not increase $|\cA(X,\theta)| $. The geometry-induced regularization described in \Cref{regul-implicit-sec-1} and \Cref{regul-implicit-sec-2} favors larger linear zones, with a fixed activation pattern, containing many examples rather than the multiplication of small linear zones, containing few or no examples.

\subsection{Experiments on the Recovery of Continuous Piecewise-Linear Functions}\label{recovery-CPL-sec}

In this section, we illustrate the geometry-induced regularization described in \Cref{implicit-regule-theory}, in light of \Cref{shallow-case-thm}. Our experiment is designed to visualize this regularization effect and to demonstrate the ability of a shallow neural network to recover a piecewise-linear target function with only a small number of segments.

The univariate scalar target function $f^*$ that we aim to recover consists of three linear segments and is shown in gray in Figure~\ref{fig:examples-different-losses}. We sample 25 independent inputs uniformly from the interval $[1,20]$. They are gathered in $X \in \RR^{1 \times 25}$ and we set $Y = f^*(X) \in \RR^{1 \times 25}$. We then use the MSE loss and train a shallow neural network with $10$ hidden neurons to fit this dataset.\footnote{With $10$ hidden neurons, the critical points are not all global minima and they exhibit greater diversity. This setting illustrates more aspects of the geometry-induced regularization.}

To better illustrate the diversity of critical points—and therefore the geometry-induced regularization (see \Cref{regul-implicit-sec-1} and \Cref{regul-implicit-sec-2})—the training is carried out using Adam in full-batch mode, with a learning rate of \(0.01\), and a stopping criterion of \(10^{-5}\) on the training loss. With this procedure, the network obtained at the end of training corresponds to a critical point of the training objective. We perform 50 training runs using the same dataset but with different random initializations (with the HeNormal initialization of Keras). Each run yields a critical point, for which we measure: the final training loss, the local dimension with respect to the training sample  \(X\), the number of activation patterns observed on \(X\) (``Seen regions''), corresponding to $|\cA(X,\theta)|$ appearing in \Cref{shallow-case-thm}, and the number of activation patterns observed on a very fine grid (``Total regions'').

\begin{figure}[ht]
\begin{subfigure}{0.5\textwidth}
       \includegraphics[width=7.5cm]{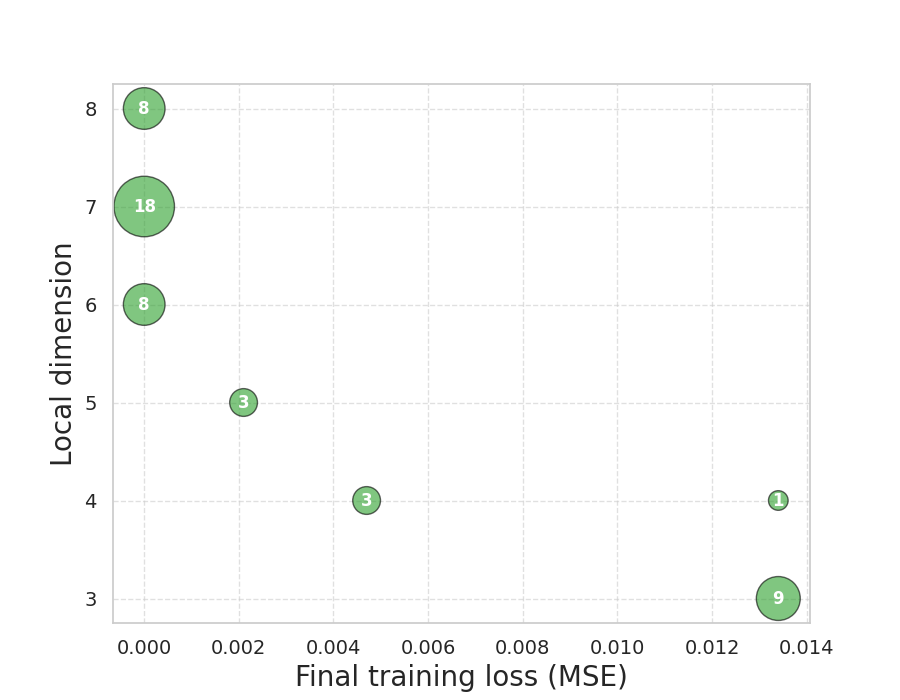}
    \caption{Local dimension vs Final training loss.}
    \label{fig:local-dim-vs-train-loss-a} 
\end{subfigure}
  \begin{subfigure}{0.5\textwidth}
      \includegraphics[width=7.5cm]{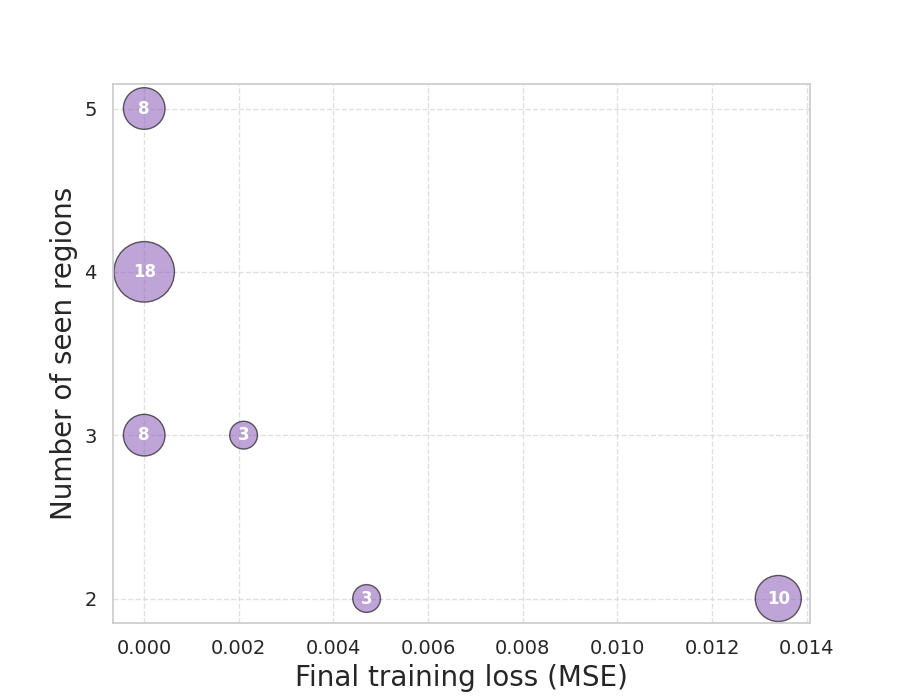}
    \caption{Seen regions vs Final training loss.}
    \label{fig:local-dim-vs-train-loss-b}
  \end{subfigure}
     \caption{Distribution of regularity as a function of final training loss for the 50 runs.}
\end{figure}

We first describe the different critical points revealed by the experiment in terms of the trade-off between training loss and regularity (measured either by the local dimension or by the number of seen regions). In \Cref{fig:local-dim-vs-train-loss-a}, each pair (training loss, local dimension) is represented by a disk whose radius is proportional to the number of runs that reached that pair. The latter is also written in the disk. The same visualization is provided in \Cref{fig:local-dim-vs-train-loss-b} for the pairs (training loss, seen regions).

The set of final training losses we observe is nearly discrete: up to variations smaller than \(10^{-5}\), we identify four distinct values: \(0\), \(2.1 \times 10^{-3}\), \(4.7 \times 10^{-3}\), and \(1.3 \times 10^{-2}\). These correspond to different critical values of the training loss. For the associated critical points, we observe varying local dimensions. In general, there is a negative correlation between these two quantities: lower training loss values correspond to higher local dimensions. In other words, networks with higher local dimensions have more degrees of freedom, enabling them to fit the training dataset more accurately.

We also observe that, due to the regularity in the training data, the training data can be fitted by a network with low local dimension. It is as if, in the example of \Cref{fig:exp1}-(b)--right, the red cross was lying on the pink line. Consequently, the local dimension of the trained networks is always smaller than \(8\), well below its maximum possible value of \(|E| + |B| - N_1 = 21\). Also, the number of seen regions is always smaller than $5$, which is much smaller than the size of the training sample $50$. In this sense, the regularity of the learned network is influenced by the regularity in the data. At last, for a shallow ReLU network of width 10, the estimated probability of achieving a training loss of \(0\) (i.e., reaching a global minimum) is approximately \(68\%\). We have observed in other experiments, not reported here, that this probability increases when training a wider network.

\begin{figure}[ht]
        \begin{subfigure}{.5\textwidth}
            \centering
            \includegraphics[width=0.95\textwidth]{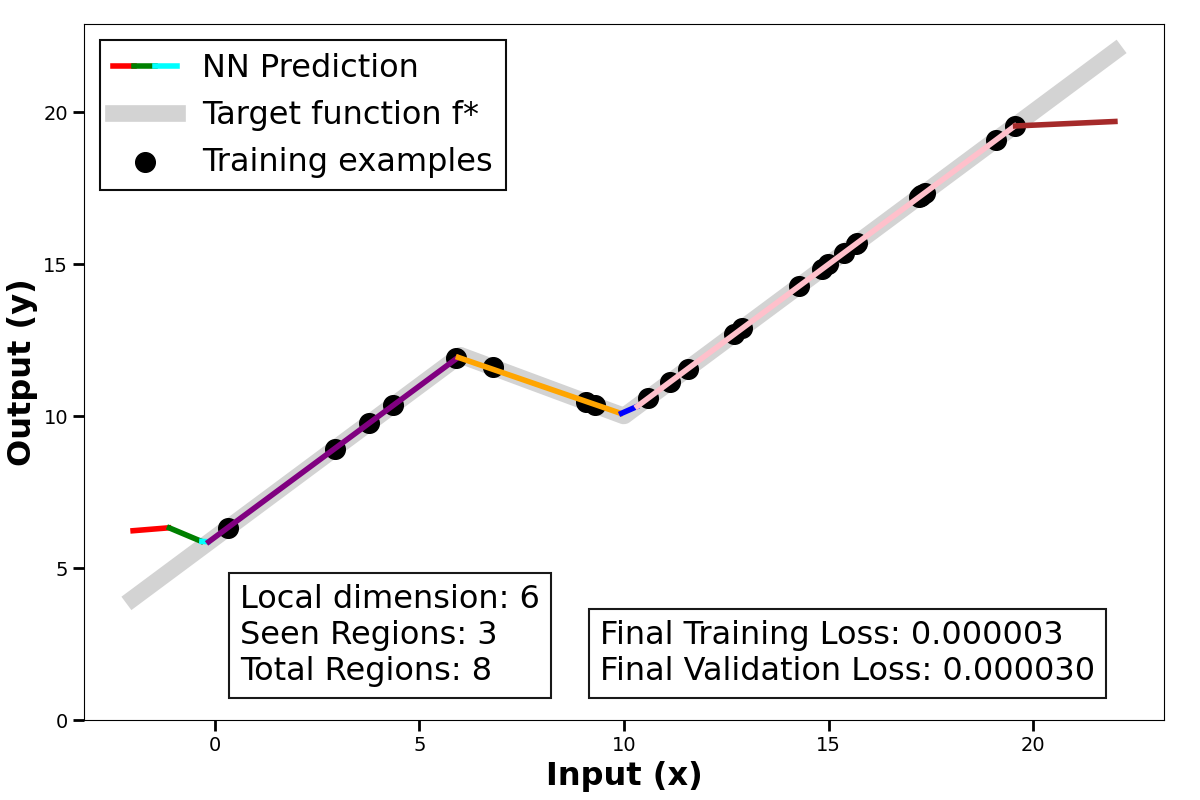}
            \caption{Final training loss: 3e-6}
        \end{subfigure}
         \begin{subfigure}{.5\textwidth}
            \centering
            \includegraphics[width=0.95\textwidth]{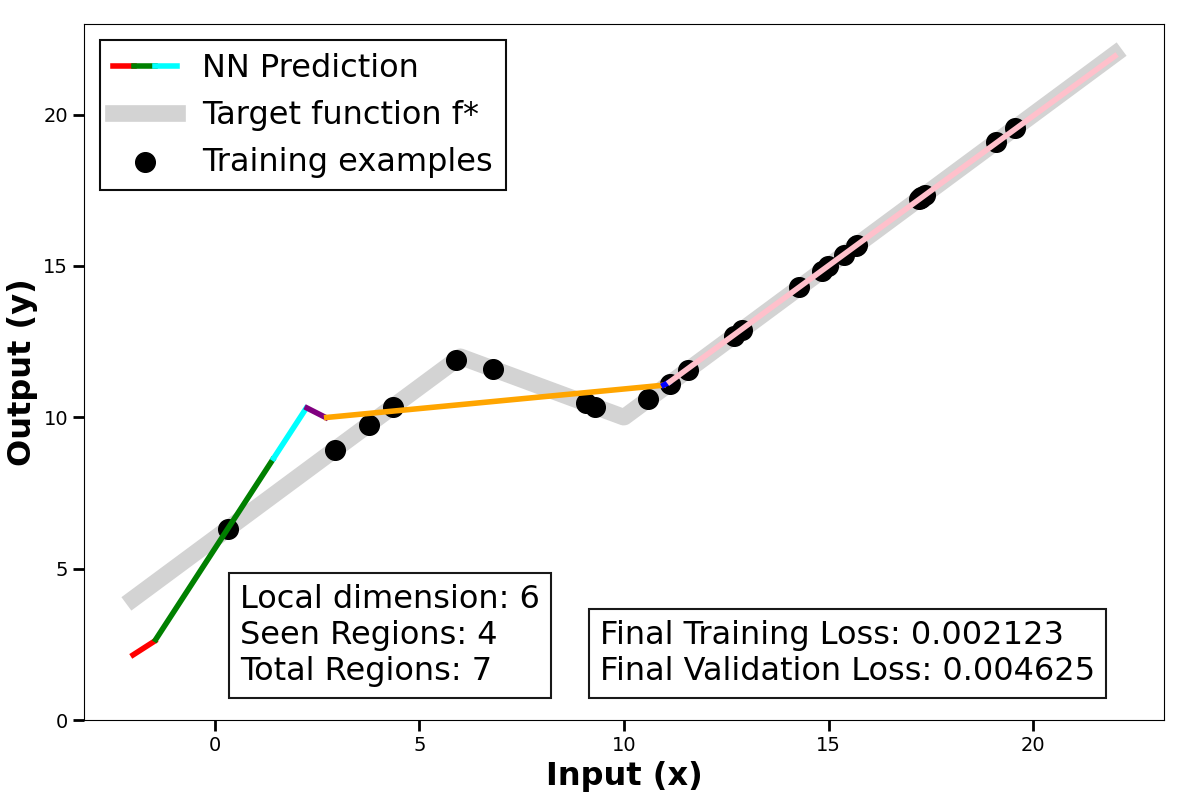}
            \caption{Final training loss: 2e-3}
        \end{subfigure}

         \begin{subfigure}{.5\textwidth}
            \centering
            \includegraphics[width=0.95\textwidth]{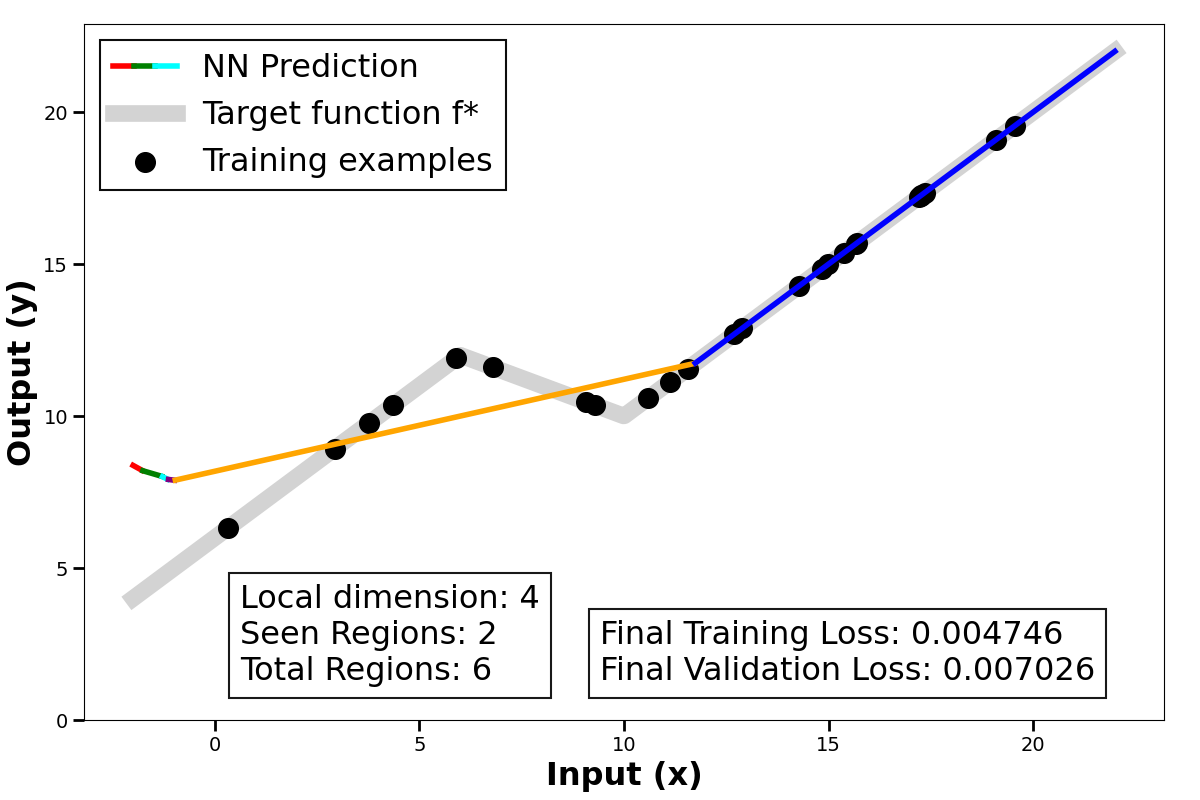}
            \caption{Final training loss: 4e-3}
        \end{subfigure}
         \begin{subfigure}{.5\textwidth}
            \centering
            \includegraphics[width=0.95\textwidth]{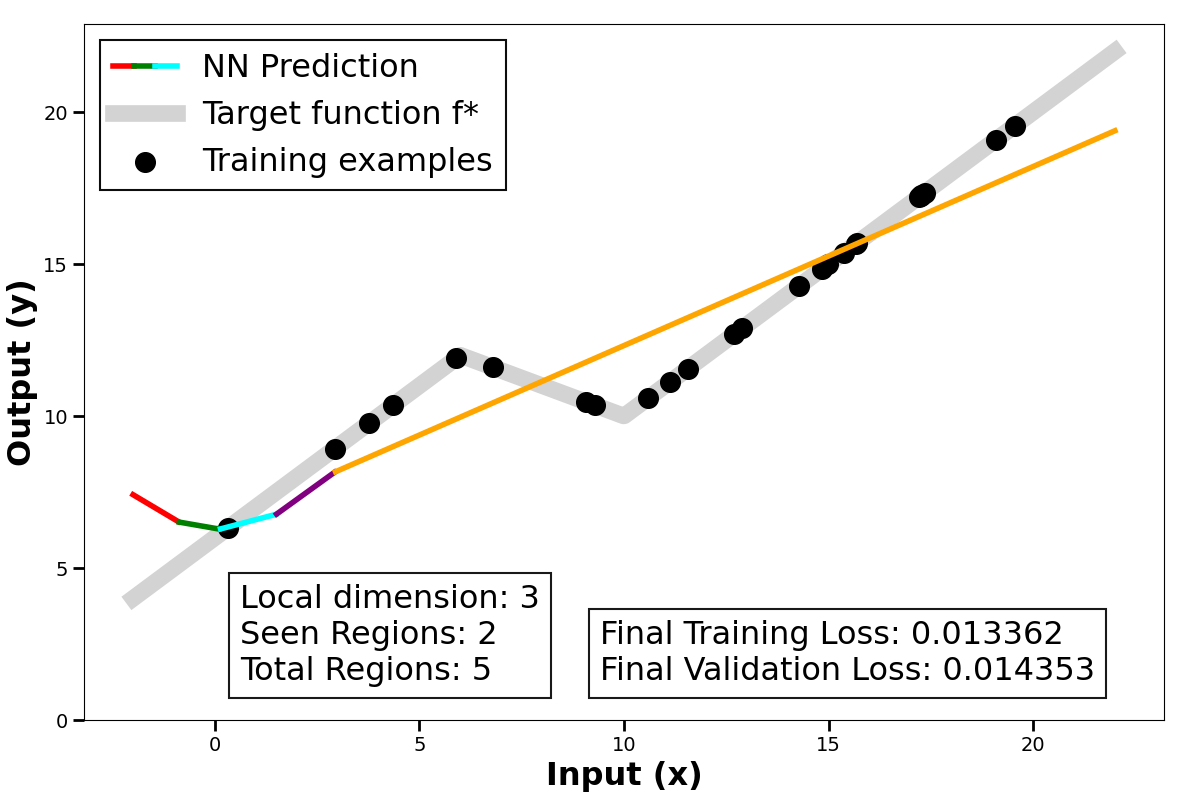}
            \caption{Final training loss: 1e-2}
        \end{subfigure}
        \caption{Examples of neural networks after training. Each subfigure corresponds to a different final training loss. The function computed by the network is plotted using multiple colors, each representing a distinct activation pattern (i.e., a linear region of the network).}
        \label{fig:examples-different-losses}
\end{figure}

To illustrate how geometry-induced regularization influences the trained network, we show in \Cref{fig:examples-different-losses} the target function \(f^*\) alongside four examples of networks after training. Each subfigure of \Cref{fig:examples-different-losses} represents a network achieving a specific training loss. 
The function computed by the network is plotted using multiple colors, each representing a distinct activation pattern (i.e., a linear region of the network).
From this figure, we observe that each linear region of the network closely approximates the linear regression of the subset of the dataset whose inputs share the corresponding activation pattern. In particular, when all these data points lie within a region where \(f^*\) is linear, the network accurately recovers \(f^*\) between the data points. These phenomena act as a form of regularization for the networks. The regularization arises when the networks’ linear regions are large and contain many examples, which occurs due to the geometry-induced regularization mechanism described in \Cref{regul-implicit-sec-1}, \Cref{regul-implicit-sec-2}, and \Cref{shallow-case-thm}.
We also note that some linear regions of the learned networks are not visited by \(X\); in these regions, we do not anticipate any clear connection between the network's behavior and the geometry-induced regularization studied in this article.

\begin{figure}[ht]
    \centering
    \includegraphics[width=0.7\textwidth]{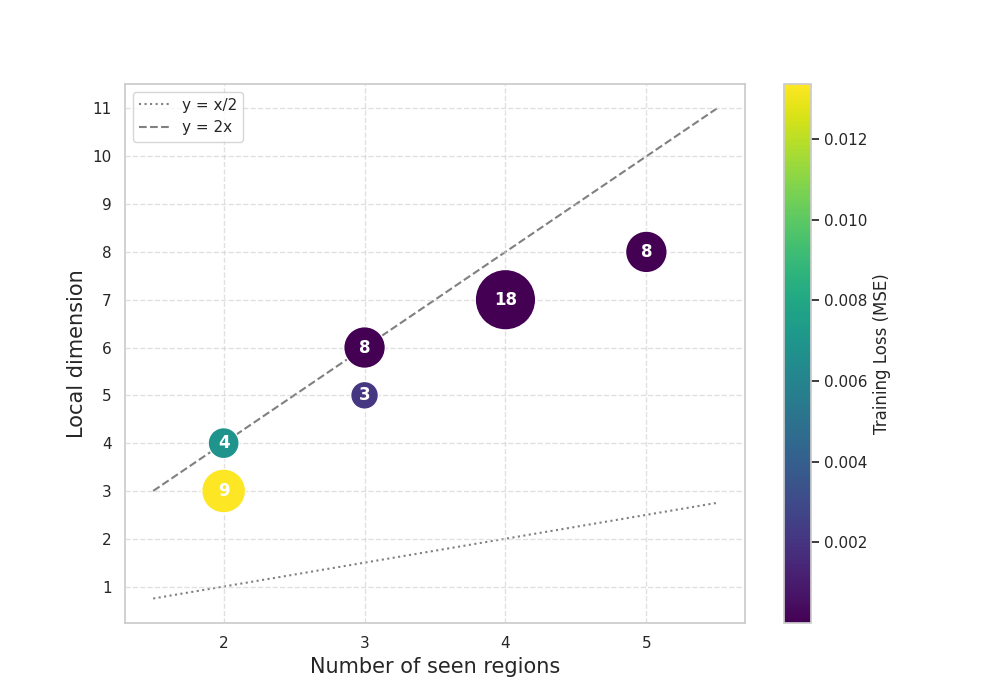}
    \caption{Local dimension versus number of seen regions for the optimized networks. The numbers indicate how many of the 50 runs produced the corresponding pair (seen regions, local dimension).}
    \label{fig:local-dim-vs-patterns}
\end{figure}

In Figure \ref{fig:local-dim-vs-patterns}, we plot, for each of the $50$ networks obtained in the experiment, the local dimension against the number of seen regions--in formula, the quantity $|\cA(X,\theta)|$ from \Cref{shallow-case-thm}. The size of the disks, and the numbers they contain, indicate how many of the 50 runs produced the corresponding pair (seen regions, local dimension). The color of the point indicates the value of the learning objective. We also plot the bounds $2|\cA(X,\theta)|$ and $\tfrac{1}{2}|\cA(X,\theta)|$ from \eqref{upper_bound_rank_shallow} of \Cref{shallow-case-thm}. We observe that the local dimension indeed remains within these bounds and is approximately proportional to the number of seen regions.

%%%%%%%%%%%%%%%%%%%%%%%%%%%%%%%%%%%%%%%%%%%%%%%%%%%%%%%%%%%%

\section{How to Compute  \texorpdfstring{$\RK{\DT f_\theta(X)}$}{the Local Dimension}}\label{sec:numerical:computations}

In this section, we describe how one can efficiently compute $\RK{\DT f_\theta(X)}$ for a given $X$ and given $\theta$.

%\subsection{How to Compute  \texorpdfstring{$\RK{\DT f_\theta(X)}$}{the Local Dimension}}\label{compute-sec-1}

For a given $X\in\Xspace$ and a given $\theta\in\Par$, $\RK{\DT f_\theta(X)}$ is computed using the backpropagation and numerical linear algebra tools computing the rank of a matrix. To justify the computations, let us first recall the classical backpropagation algorithm for computing the gradients with respect to the parameters of the network, for a given loss $R: \RR^{N_L} \longrightarrow \RR$. We will then describe how to use the backpropagation to compute $\RK{\DT f_\theta(X)}$. We conclude with implementation recommendations.

For a given input $x \in \RR^{N_0}$, backpropagation computes the gradient $\nabla \cL(\theta)$ of the function $\theta \longmapsto \cL(\theta) = R(f_\theta(x))$. To do so, it first computes $f_\theta(x)$ and stores the intermediate pre-activation values $(y^\ell_{\theta})_v = \sum_{v' \in V_{\ell-1}} w_{v'\rightarrow v} (f_\theta^{\ell-1}(x))_{v'} + b_v $, for $\ell \in \lb 1, L \rb$ and $v\in V_\ell$. This is known as the \lq forward pass\rq. Then, backpropagation computes the vector of errors $\eta^L_\theta$ defined by
\begin{equation}\label{oeinrbpnibt}
\eta^L_\theta = \left(J\sigma_L (y_\theta^L)\right)^T\frac{\partial R}{\partial y} (f_\theta(x)), 
\end{equation}
where $\frac{\partial R}{\partial y}(f_\theta(x))\in\RR^{N_L}$ is the gradient of $y \longmapsto R(y)$, at the point $f_\theta(x)$, and $J\sigma_L (y^L_\theta)$ is the Jacobian matrix of $y^L \mapsto \sigma_L(y^L)$, at $y^L_\theta$.
This vector is then backpropagated, from $\ell=L$ to $\ell=1$ thanks to the equation
\begin{equation} \label{backpropagation} 
\forall v'\in V_{\ell-1}\qquad \left(\eta^{\ell-1}_{\theta}\right)_{v'} =   \sigma'\left(\left(y^{\ell-1}_{\theta}\right)_{v'}\right)  \sum_{v\in V_{\ell}}  w_{v'\rightarrow v} \left(\eta^\ell_{\theta}\right)_{v} 
\end{equation}
where $\sigma'(t) = 1$ if $t>0$ and $\sigma'(t) = 0$ if\footnote{Neural networks libraries such as \texttt{Tensorflow} set $\sigma'(0) = 0$ and we adopt this convention in this calculation. Due to numerical imprecision, we rarely have $(y^{\ell-1}_{\theta})_{v'}=0$ in practice. In the theoretical sections of this article, the situation $\sigma'(0)$ never occurs for the cases where $\DT f_\theta(X)$ is considered.} $t\leq 0$.
This allows to recursively obtain the error vectors $\eta^\ell_\theta \in \RR^{N_\ell}$, for all $\ell \in \lb 1 , L \rb$. We deduce the partial derivatives thanks to the formulas
\begin{equation*}
\forall \ell \in \lb 1 , L \rb, \forall v' \in V_{\ell-1}, \forall v \in V_\ell, \qquad \frac{\partial R(f_\theta(x))}{\partial w_{v' \rightarrow v}} = \sigma\left(\left(y^{\ell-1}_{\theta}\right)_{v'}\right) \left(\eta^\ell_{\theta}\right)_{v}
\end{equation*}
and
\begin{equation*} \forall \ell \in \lb 1, L \rb, \forall v \in V_\ell, \qquad \frac{\partial R(f_\theta(x))}{\partial b_v} = \left(\eta^\ell_{\theta}\right)_{v}. 
\end{equation*}
This allows computing the gradients for one example $x$. For a batch, the algorithm is repeated for each example $x^{(i)}$, and the average of the so obtained gradients is computed. 

Let us now make the connection between backpropagation and the computation of $\RK{\DT f_\theta(X)}$. Vectorizing both the input and output spaces of $\theta \longmapsto f_\theta(X)$, we first notice that $\RK{\DT f_\theta(X)} = \RK{J f_\theta(X)}$, where the Jacobian matrix $J f_\theta(X) \in\RR^{nN_L \times (|E|+|B|)}$ takes the form
\[
J f_\theta(X) = \left(\begin{array}{c} 
J f_\theta(x^{(1)}) \\
\vdots \\
J f_\theta(x^{(n)})
\end{array}\right)
\]
and, for all $i\in\lb 1, n \rb$, $J f_\theta(x^{(i)}) \in \RR^{N_L \times (|E|+|B|)}$ is the Jacobian matrix of $\theta \longmapsto f_\theta(x^{(i)})$. We construct the matrix $J f_\theta(X)$ by successively computing each of its rows. This is achieved by computing each row of $J f_\theta(x^{(i)})$ for all $i\in\lb 1, n \rb$, with the method described below. 

For a given $i\in\lb 1, n \rb$ and $v\in V_L$, the line corresponding to $v$ of $J f_\theta(x^{(i)})$  is indeed simply obtained as the transpose of $\nabla R_v(f_\theta(x^{(i)}))$ for the function $R_v: \RR^{N_L}\longrightarrow \RR$ defined by $R_v(y) = y_v$, for all $y\in \RR^{N_L}$. We indeed have $ R_v(f_{\theta'}(x^{(i)})) = f_{\theta'}(x^{(i)})_v$ for all $\theta'$. The gradient $\nabla R_v(f_\theta(x^{(i)}))$ is obtained using the backpropagation algorithm described above. Notice that when $\sigma_L$ is the identity, for a given $v\in V_L$, using the definition of $R_v$ and \eqref{oeinrbpnibt}, we always have $(\eta^L_\theta)_v = 1$ and $(\eta^L_\theta)_{v'} = 0$ for all $v'\neq v$. We need however to compute the forward pass in order to compute the vectors $y^\ell_\theta$, for $\ell\in\lb 0,L-1\rb$. Finally, once $J f_\theta(X)$ is computed its rank is obtained using standard linear algebra algorithms.

 Our implementation uses the existing automatic differentiation of \texttt{Tensorflow}. It is possible to call the method \texttt{GradientTape.gradients}, which computes $J f_\theta(x)$ for a single example $x$, and to repeat it for each example $x^{(i)}$. However, it is more efficient to use \texttt{GradientTape.jacobian} which allows to compute directly $J f_\theta(X)$. We do not report the details of the experiments here but we found even more efficient to cut $X$ in sub-batches and repeatedly call \texttt{GradientTape.jacobian}, when appropriately choosing the size of the sub-batches.

Once $J f_\theta(X)$ built, the value of $\RK{J f_\theta(X)}$ can be computed with the \texttt{np.linalg.rank} function of \texttt{Numpy}, or using the accelerated rank computation of \texttt{Pytorch} with a GPU, which improves the speed by some factors. Note that the limiting factor when computing $\RK{J f_\theta(X)}$ for large networks and/or $n$ large is the computation of the rank and not the construction of $J f_\theta(X)$.

The codes are available at \citep{code_calcul_rang}.

%%%%%%%%%%%%%%%%%%%%%%%%%%%%%%%%%%%%%%%%%%%%%%%%%%%%%%%%%%%%
%%%%%%%%%%%%%%%%%%%%%%%%%%%%%%%%%%%%%%%%%%%%%%%%%%%%%%%%%%%%

%%%%%%%%%%%%%%%%%%%%%%%%%%%%%%%%%%%%%%%%%%%%%%%%%%%%%%%%%%%%
\section{Experiments on MNIST Dataset}\label{sec:experiments}

%When finalizing the Ph.D. report, we realized there was a small discrepancy between the network considered in the experiments and the networks studied in the article but did not have time to correct it. In all the experiments, the last layer includes a softmax activation function that is taken into account by the backpropagation and therefore affects the \NAME s .
The experiments provide evidence that {\it geometry-induced regularization} occurs on the MNIST dataset. They further highlight that the regularization observed during training also manifests at inference time on the test set.

The setting of the experiments is described in Section \ref{expe-description}. In Section \ref{expe-w_varies}, we describe the results of an experiment in which we compute the \NAME{} as the number of parameters of the network grows. In Section \ref{expe-rank_diminu}, we compute the \NAME{} throughout the learning phase.  
%In Section \ref{expe-X-corrupted}, we investigate the impact of the corruption of the inputs of the learning sample on the \NAME{}. In Section \ref{expe-Y-corrupted}, we perform the same experiment but corrupt the outputs of the learning sample.

The Python codes implementing the experiments described in this section are available at \citep{code_calcul_rang}.

\subsection{Experiments Description}\label{expe-description}

In the experiments of Sections \ref{expe-w_varies} and \ref{expe-rank_diminu}
%, \ref{expe-X-corrupted} and \ref{expe-Y-corrupted}
, we evaluate the behavior of different complexity measures for the classification of a subpart of the MNIST data set. 

We consider a fully-connected feed-forward ReLU network of depth $L=4$, of widths $(N_0, N_1, N_2, N_3, N_4) = (784,w,w,w,10)$, for different values of $w\in\lb1,85\rb$. The tested values of $w$ depend on the experiment/section. The hidden layers ($1$, $2$, $3$) include a ReLU activation function. The last layer includes a soft-max activation function. We randomly extract a training sample $(\Xtr,\Ytr)$, containing $6~000$ images and a test sample $(\Xts,\Yts)$ containing $20~000$ images from MNIST.

%The sizes of the samples depend on the experiment/section. They are tuned so that the computing time of each experiment remains reasonable.

For given $w$ and $(\Xtr,\Ytr)$, we tune the parameters of the network to minimize the cross-entropy. This is achieved using the Glorot uniform initialization for the weights while initializing the biases to $0$, and using the stochastic gradient descent \lq sgd\rq~as optimizer with a learning rate of $0.1$ and a batch size of $256$. The number of epochs depends on the experiment/section.

In the figures presenting the results of the experiments, we display the following quantities:
\begin{itemize}
\item Max rank: the maximal theoretically possible value of $\RK{\DT f_\theta(X)}$ for any sample $X$ and parameter $\theta$. It is equal to $|E|+|B| - N_1 - \dots - N_{L-1} = N_0 N_1 + N_1 N_2 + \dots + N_{L-1} N_L + N_L$ (see the bound provided by \citet{grigsby2022functional}, Theorem 7.1). With the architecture described above, for a given $w$, the Max rank is equal to $2 w^2 +794w + 10$. This is very close to the number of parameters $2 w^2 + 797 w + 10$. Furthermore, with the values of $w$ considered in the forthcoming experiments, the predominant term is $794w$.
%an estimation of the \NAMEMAX, according to the statement of Section \ref{sec:how:to:compute:r}, by computing $\RK{\DT f_\theta(\Xrd)}$ with a random i.i.d. sample $\Xrd$, where each example of the sample is a Gaussian random vector. The number of examples is equal to $20\,000$ or $40\,000$ depending on the experiment/section.
\item Rank X\_train: It corresponds to $\RK{\DT f_\theta(\Xtr)}$, where $\Xtr$ is the training sample of size $6~000$ mentioned above. This quantity is the \NAME{}.
\item Rank X\_test: It corresponds to $\RK{\DT f_\theta(\Xts)}$, where $\Xts$  is the test sample of size $20~000$ introduced above. The motivation for considering this quantity is to demonstrate that the geometry-induced regularization put to evidence on the training set is sufficiently strong to influence the network’s regularity when measured on the test sample.
\item Train loss: the cross-entropy loss value, evaluated on the training sample at the end of training (resp. at the current epoch) in Sections \ref{expe-w_varies}
%, \ref{expe-X-corrupted} and \ref{expe-Y-corrupted} 
(resp. in Section \ref{expe-rank_diminu}).
\item Test error: the proportion of images of $\Xts$ that are misclassified by the network.
\item Train error: the proportion of images of $\Xtr$ that are misclassified by the network.
\end{itemize}
Note that the test set is bigger than the train set, in contrast to classical settings. Indeed, the test set serves two purposes here: it is classically used to compute the classification accuracy, but it is also meant to provide an estimation of the \NAME{} when computed on test sample.

\subsection{Behavior of the Local Dimensions as the Network Width Increases}\label{expe-w_varies}

In this experiment, we evaluate the \NAME s  when the width $w$ varies between $1$ and $85$. More precisely, we test all $w$ between $1$ and $9$, then all $w$ between $10$ and $18$ with an increment of $2$, and then all $w$ between $20$ and $85$ with an increment of $5$. Overall, the number of parameters of the network varies between $809$ and $82~205$.

The setting of the experiment is described in Section \ref{expe-description}. We optimize the network parameters during $1~000$ epochs. 

\begin{figure}
\begin{subfigure}{.5\textwidth}
\centering
    \includegraphics[width=.98\textwidth]{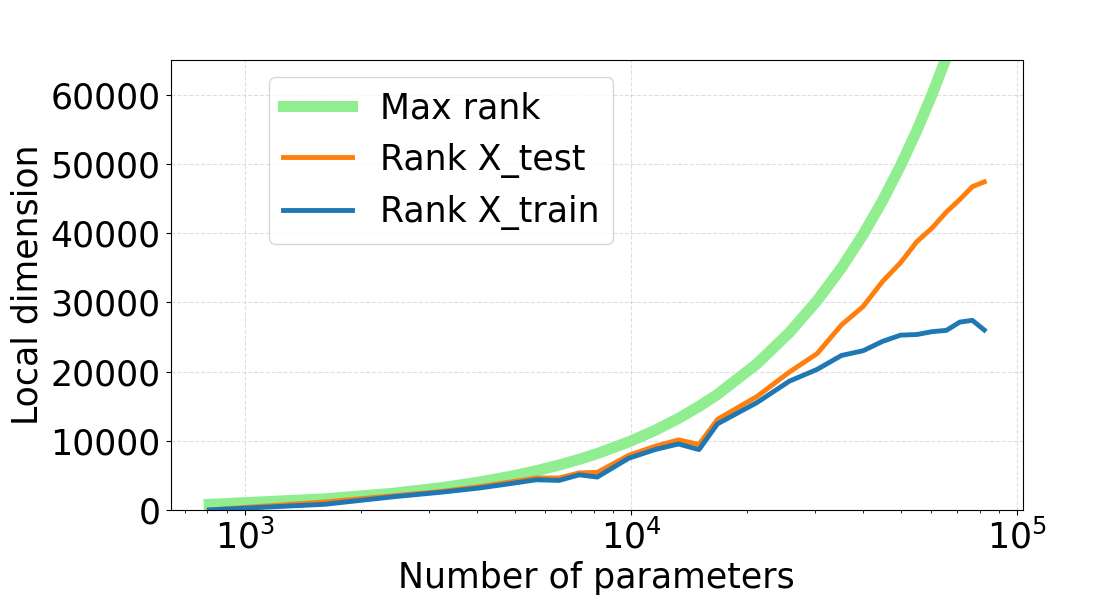}
    \caption{Local dimensions vs. the number of parameters}
\end{subfigure}
\begin{subfigure}{.5\textwidth}
\centering
        \includegraphics[width=.98\textwidth]{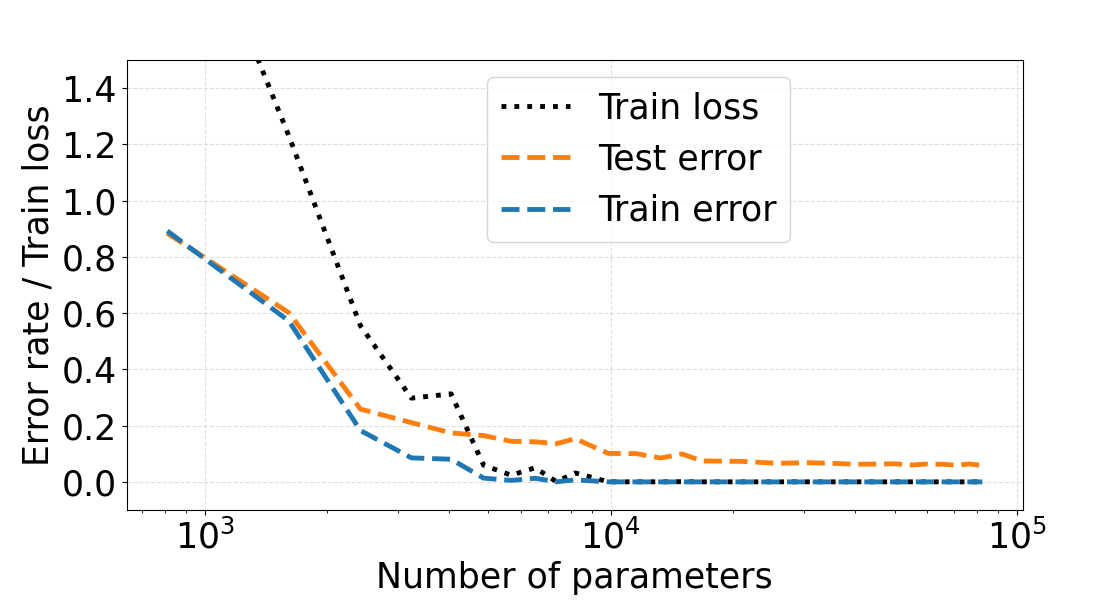}
        \caption{Loss and errors vs. the number of parameters}
\end{subfigure}
\caption{Behavior of different complexity measures as the size of the network increases.
} 
\label{Width-figure}
\end{figure}

\begin{comment}
The results of the experiment are in Figure \ref{Width-figure}. When increasing the number of parameters, the train loss, the train error and the test error decrease. For $w \geq 14$, i.e. when the number of parameters is superior or equal to $11~560$, the train error is equal to $0$: the network is able to fit perfectly the training images. However, the test error continues to decrease even after the train error reaches $0$: from $0.068$ when $w=14$ to $0.035$ when $w=85$.
\end{comment}

The results of the experiment are in Figure \ref{Width-figure}. When increasing the number of parameters, the train loss, the train error and the test error decrease. For $w \geq 12$, i.e. when the number of parameters is superior or equal to $9~862$, the train error is equal to $0$: the network is able to fit perfectly the training images. However, the test error continues to decrease even after the train error reaches $0$: from $0.101$ when $w=12$ to $0.058$ when $w=85$.

The ranks $\RK{\DT f_\theta(\Xtr)}$ and $\RK{\DT f_\theta(\Xts)}$ are nearly equal when the number of parameters is smaller than $21\,185$ ($w = 25$). Given the size of the test sample, this seems to indicate that the network is regularized on the whole support of the input distribution. Also, adding MNIST images to $\Xtr$ would not increase $\RK{\DT f_\theta(\Xtr)}$. Since $\RK{\DT f_\theta(\Xtr)}$ and $\RK{\DT f_\theta(\Xts)}$ are strictly less than Max rank, according to \cite{bona2022local}, this also shows that $\theta$ is not identifiable from $\Xtr$ nor $\Xts$. This suggests that, for these networks, using only samples of the input distribution does not allow to identify the parameters of a network. Asserting whether it is possible to identify them locally using examples outside the input distribution remains an open question.

Then, for more than $21\,185$ parameters (i.e. $w \geq 25$), a gap appears between the two ranks $\RK{\DT f_\theta(\Xtr)}$ and $\RK{\DT f_\theta(\Xts)}$, which in particular implies that $\RK{\DT f_\theta(\Xtr)}$ is smaller than the \NAME{} over the distribution of the inputs. Furthermore, while both ranks are not far from the maximum rank for $w < 25$, this other gap also increases with the number of parameters, to the point where the shapes of the curves seem to diverge: while the maximum rank is nearly proportional to the number of parameters, the ranks $\RK{\DT f_\theta(\Xtr)}$ and $\RK{\DT f_\theta(\Xts)}$ seem to increase less and less with the number of parameters. This shows that the geometry-induced regularization occurs and is more significant for larger networks. As the curve $\RK{\DT f_\theta(\Xts)}$ indicates, the regularization on the training sample also applies to the test sample, and thus—given the size of the test sample—extends to nearly the entire support of the input distribution.

\subsection{Behavior of the Local Dimensions During Training}\label{expe-rank_diminu}

\begin{figure}
\begin{subfigure}{.5\textwidth}
\centering
    \includegraphics[width=.98\textwidth]{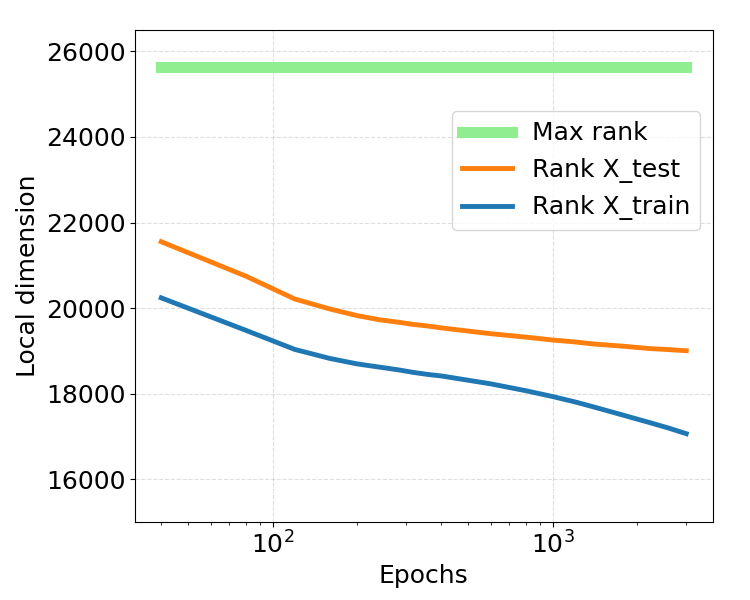}
    \caption{Local dimensions vs. epoch number}
\end{subfigure}
\begin{subfigure}{.5\textwidth}
\centering
        \includegraphics[width=.98\textwidth]{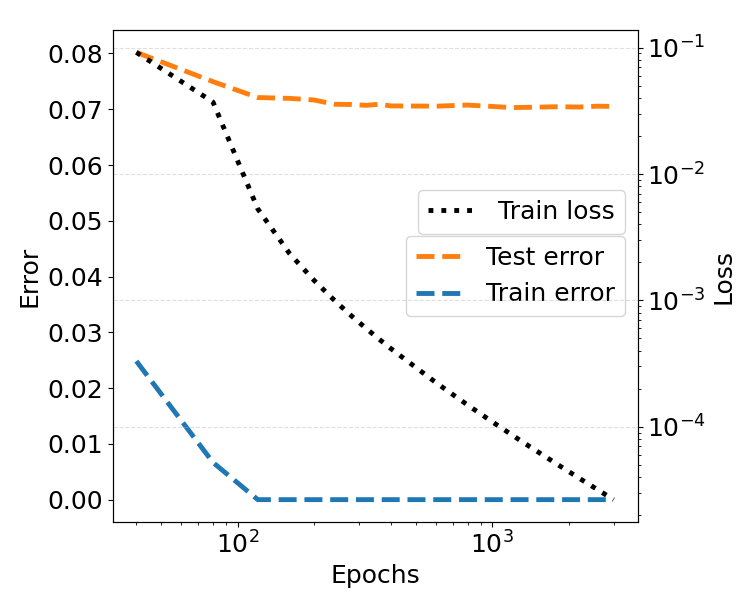}
        \caption{Loss and errors vs. epoch number}
\end{subfigure}
\caption{Behavior of different complexity measures during training.
} 
\label{epoch-figure}
\end{figure}
%\begin{figure}
%\centering
%   \begin{tabular}{ll}
%       \includegraphics[scale=0.38]{Figures/Epoch-ranks-review.png}
%        &  \includegraphics[scale=0.38]{Figures/Epoch-error-loss-review.png}
%\end{tabular}
%\caption{Behavior of different complexity measures during training. The values of the different ranks are to be read on the left axis, titled \lq \NAME{}\rq. The values of the train loss (on the left figure), and the values of the test and train errors (on the right figure) are to be read on the right axis.}\labels
%\end{figure} 

We consider the same setting described in Section \ref{expe-description}, with $w=30$. The quantities plotted in the previous experiment (see Figure \ref{Width-figure}) are computed after the training is done. In contrast, here, we fix a total number of epochs to $3~000$ and we compute the same quantities during training, throughout the epochs.

More precisely, we study the quantities Max rank, Rank X\_test, Rank X\_train, Train loss, Test error and Train error, as described in Section \ref{expe-description}. They are computed at the epochs $\{40, 80, 120, 160, 200, 240, 280, 320, 360, 400\}\cup \{ 600, 800, 1~000, 1~200, 1~400\} \cup \{ 1~800, 2~200, 2~600, 3~000 \}$. We plot these quantities in Figure \ref{epoch-figure}.

We plot the train loss (on the left), which decreases throughout the epochs, and the train error (on the right), which decreases and reaches $0$ at epoch $120$, after which all training images are always correctly classified. The test error decreases the most in the first 80 epochs, after which it continues to decrease, although at a slower pace.

We observe that the value of $\RK{\DT f_\theta(\Xtr)}$ consistently decreases during training. 
The value of $\RK{\DT f_\theta(\Xts)}$ also decreases, with a more gentle slope. This indicates that the geometry-induced regularization occurring on the training sample is \lq communicated\rq~to the test sample.

%%%%%%%%%%%%%%%%%%%%%%%%%%%%%%%%%%%%%%%%%%%%%%%%%%%%%%%%%%%%
%%%%%%%%%%%%%%%%%%%%%%%%%%%%%%%%%%%%%%%%%%%%%%%%%%%%%%%%%%%%

\section{Conclusion and Perspectives}

In this article, we study the local geometry of deep ReLU neural networks. We show that the image of a sample $X$ under such networks, for a fixed architecture, forms a set whose local dimension\footnote{Referred to as the batch functional dimension by \cite{grigsby2022functional}.} may vary. The parameter space is partitioned into regions within which the \NAME~remains constant. This local dimension is invariant under the natural symmetries of ReLU networks, namely positive rescalings and neuron permutations. Our analysis reveals that the geometry of deep ReLU networks gives rise to a regularization phenomenon, where the regularity criterion is essentially captured by the local dimension. We establish connections between the local dimension, the flatness of minima, and saddle-to-saddle dynamics. For shallow ReLU networks, we further show that the local dimension is directly related to the number of linear pieces perceived by the sample $X$, thereby shedding light on the effect of regularization. Finally, we investigate the practical computation of the local dimension and present experiments on the MNIST dataset that highlight the role of geometry-induced regularization.

This work opens several perspectives for deep learning theory. A formal connection between geometry-induced regularization and generalization guarantees is still lacking; establishing such a link could provide a theoretical foundation for the remarkable performance of deep learning. From a practical standpoint, it would be valuable to investigate geometry-induced regularization empirically on higher-dimensional datasets. Developing algorithms with lower computational complexity for estimating local dimensions is another important direction. In particular, since we have shown that the \NAME{} is almost surely determined by activation patterns, it would be natural to compute it directly from activation patterns rather than from gradients. Furthermore, designing a test to evaluate the notion of flatness of minima introduced in this article would provide additional insights. Finally, extending this geometric analysis to other network architectures remains an interesting avenue for future research.

%%%%%%%%%%%%%%%%%%%%%%%%%%%%%%%%%%%%%%%%%%%%%%%%%%%%%%%%%%%%
%%%%%%%%%%%%%%%%%%%%%%%%%%%%%%%%%%%%%%%%%%%%%%%%%%%%%%%%%%%%

% Acknowledgements and Disclosure of Funding should go at the end, before appendices and references

\acks{This work has benefited from the AI Interdisciplinary Institute ANITI. ANITI is funded by the French ``Investing for the Future – PIA3” program under the Grant agreement n°ANR-19-PI3A-0004.

The authors gratefully acknowledge the support of the DEEL project\footnote{\url{https://www.deel.ai/}}. They would like to thank Daniele Cannarsa and the anonymous reviewers for their contributions.}


%%%%%%%%%%%%%%%%%%%%%%%%%%%%%%%%%%%%%%%%%%%%%%%%%%%%%%%%%%%%
%%%%%%%%%%%%%%%%%%%%%%%%%%%%%%%%%%%%%%%%%%%%%%%%%%%%%%%%%%%%

\newpage

\appendix

\section{Proofs of Section \ref{section:rank:properties}}

This appendix is devoted to the proofs of Section \ref{section:rank:properties}. In Section \ref{theorem:constant:rank-proof}, we prove Theorem \ref{theorem:constant:rank}, and in Section \ref{app:proof:rank:invariant:rescaling} we prove Proposition \ref{prop:rank:invariant:rescaling}.

 \subsection{Proof of Theorem \ref{theorem:constant:rank}}\label{theorem:constant:rank-proof}

For any $x\in\RR^{N_0}$, $\ell \in \lb 1 , L-1 \rb$, and $v \in V_{\ell}$, let us define the set of parameters for which the activation of neuron $v$ changes:
\begin{equation}\label{def_Txv}
\cT^x_v = \Bigg \{ 
\theta \in \RR^E \times \RR^B
~|~
\sum_{v' \in V_{\ell-1}} 
w_{v' \to v}
\left( f_\theta^{\ell-1}(x) \right)_{v'}
+
b_v
=0
\Bigg \},
\end{equation}
and let
\begin{equation}\label{def_Tx}
    \cT^x
=
\cup_{\ell=1}^{L-1}
\cup_{v \in V_{\ell}} \cT^x_v.
\end{equation}

 \begin{lem} \label{lemma:valeurs:de:a:et:polynome}
For any given $x \in \RR^{N_0}$, the three following items hold:

\begin{itemize}
    \item  The function $ \Par \ni \theta \longmapsto a(x , \theta)$  exactly takes $2^{N_1 + \dots + N_{L-1}}$ distinct values. 
    \item For any $\delta \in \{0,1 \}^{N_1+\cdots + N_{L-1}} $, we write
    \begin{equation}\label{A_delta_x}
    A^x_\delta = \{  \theta \in \RR^E \times \RR^B ~|~ a(x,\theta) = \delta \}.
    \end{equation}
    Then: On $A^x_\delta$, the function $\theta \longmapsto f_{\theta}(x)$ is polynomial of degree $L$, when $\sigma_L=Id$, and it is analytic otherwise.
    
    \item The set $\cT^x$ is closed and has Lebesgue measure zero and $\cup_{\delta \in \{0,1 \}^{N_1+\cdots + N_{L-1}}} \partial A^x_\delta = \cT^x$. 
    Therefore, for any $\delta \in \{0,1 \}^{N_1+\cdots + N_{L-1}}$, $\partial A^x_\delta$ is a closed set with Lebesgue measure zero in $\RR^E \times \RR^B$.
\end{itemize}

 \end{lem}

%In statement $(ii)$ of Lemma \ref{lemma:valeurs:de:a:et:polynome}, the quantity $B^x_\delta$ depends on $x$. We drop the dependence in the notation for convenience.

\begin{proof}[Proof of Lemma \ref{lemma:valeurs:de:a:et:polynome}]

Throughout the proof, we consider a fixed $x\in\RR^{N_0}$.

We first prove the first item, i.e. we prove that all activation patterns are reached. The set $\{0,1\}^{N_1 + \cdots + N_{L-1}} $ is finite and its cardinal is $2^{N_1 + \dots + N_{L-1}}$. Observe that for any $\delta \in \{0,1\}^{N_1 + \cdots + N_{L-1}}$, by taking $\theta \in \Par$ such that $w_{v \rightarrow v'} = 0$ for any $(v \rightarrow v') \in E$, $b_v = 0$ for $v \in V_L$ and $b_v = (-1)^{1+\delta_v}$ for any $v \in V_1 \cup \dots \cup V_{L-1} $, then, for any $v \in V_1 \cup \dots \cup V_{L-1}$, we have $a_v(x, \theta) = \delta_v$, i.e. $a(x, \theta) = \delta$.

In order to prove the second item, i.e. that the function $ \theta \longmapsto f_{\theta}(x)$ is polynomial of degree $L$ on $A^x_\delta$, when $\sigma_L=Id$, and it is analytic otherwise, we remind the definition of $f_\theta^\ell$, in \eqref{deff_l}, and we define for all $\theta$
\[
a_{\leq \ell}(x,\theta)
=
\begin{cases}
(a_v(x , \theta))_{v \in V_1 \cup \cdots \cup V_{\ell}}
& ~ \text{if} ~ \ell \geq 1, 
\\
1 & ~ \text{if} ~ \ell = 0.
\end{cases}
\]
We prove  by induction that the assertion
\[
H_{\ell} 
: \left\{ \begin{array}{l} 
\forall D \subseteq \RR^E \times \RR^B, ~ ~
\text{if $\theta \longmapsto a_{\leq \ell}(x,\theta)$
is constant on $D$, then } 
\\ 
\text{
$\theta \longmapsto f_{\theta}^{\ell}(x)$ is polynomial of degree $\ell$ on $D$
}
\end{array}\right.
\]
holds, for all $\ell \in \lb 0 , L-1 \rb$.

The assertion $H_0$ indeed holds because $f_\theta^0(x) = x$ is polynomial in $\theta$ (of degree $0$) on any subset of $\RR^E \times \RR^B$.
Assume now that $H_{\ell-1}$ holds, for some $\ell \in \lb 1 , L-1 \rb$, and let us prove $H_{\ell}$. 

Let $D \subseteq \RR^E \times \RR^B$ such that $\theta \longmapsto a_{\leq \ell} (x , \theta)$ is constant on $D$. For $\theta \in D$ and $v \in V_{\ell}$, using \eqref{a=relu}, we have 
\begin{align*}
    \left(  f_{\theta}^{\ell}(x) \right)_v 
    = &% \sigma  \left(  \sum_{v' \in V_{\ell-1}}  w_{v' \to v} \left( f_{\theta}^{\ell - 1}(x) \right)_{v'}+ b_v \right) \\ = &
a_v(x , \theta)
    \left( 
\sum_{v' \in V_{\ell-1}} 
w_{v' \to v} 
\left( 
f_{\theta}^{\ell - 1}(x) 
\right)_{v'}
+ b_v
    \right).
\end{align*}

The quantity $a_{\leq \ell-1}( x , \theta )$ is constant on $D$ and thus from $H_{\ell-1}$, for all $v'\in V_{\ell-1}$,  $\theta \longmapsto (  f_{\theta}^{\ell - 1} (x) )_{v'}$ is a polynomial function of $\theta$, of degree $\ell-1$, on $D$. Since $a_v(x , \theta)$ is constant on $D$, $ \theta \longmapsto   \left(  f_{\theta}^{\ell}(x) \right)_v $ is a polynomial function of $\theta$, of degree $\ell$. This concludes the proof  by induction that $H_{\ell}$ holds for all $ \ell \in \lb 0, L-1 \rb $. 
    
If we recall from \eqref{deff_l} that $y^L_\theta(x) \in \RR^{N_L}$ is the vector satisfying, for all $v \in V_L$,
\[
(y_\theta^L(x))_v
=
\sum_{v' \in V_{L-1}} 
w_{v' \to v}
(f_\theta^{L-1}(x))_{v'} 
+ b_v,
\]
we have
\[ f_\theta(x) = \sigma_L(y_\theta^L(x)).\]
We recall the definition of $A^x_\delta$, for $\delta \in \{ 0,1 \}^{N_1 + \dots + N_{L-1}}$, in \eqref{A_delta_x}. For $\delta \in \{ 0,1 \}^{N_1 + \dots + N_{L-1}}$, $a_{\leq L-1}(x , \theta) = a(x , \theta)$ is constant on $A^x_\delta$ and thus from $H_{L-1}$, $\theta \longmapsto f_\theta^{L-1}(x)$ is polynomial, of degree $L-1$, on $A^x_\delta$. As a consequence, $\theta \longmapsto y^L_\theta(x)$ is polynomial, of degree $L$, on $A^x_\delta$. When $\sigma_L\neq Id$, $\sigma_L$ is analytic, and $\theta\longmapsto f_\theta(x)$ is a composition of analytic functions and is analytic on $A^x_\delta$. This proves the second item of Lemma \ref{lemma:valeurs:de:a:et:polynome}.

Let us now show the third item, which states that $\cT^x$ has Lebesgue measure zero.
For that, let us show that for all $\ell \in \lb 1, L-1 \rb$ and $v \in V_{\ell}$, $\cT^x_v$ has Lebesgue measure zero. To do so, since $\cup_\delta A^x_\delta = \RR^E \times \RR^B$, we consider $\ell \in \lb 1, L-1 \rb$ and $v \in V_{\ell}$, and prove that, for all $\delta \in \{ 0,1 \}^{N_1 + \dots + N_{L-1}}$, $\cT^x_v \cap A^x_\delta$
has Lebesgue measure zero.
For $\delta \in \{ 0,1 \}^{N_1 + \dots + N_{L-1}}$, $a_{\leq \ell-1}(x , \theta)$ is constant on $A^x_\delta$ and thus from $H_{\ell-1}$, $\theta \longmapsto f_\theta^{\ell-1}(x)$ is a polynomial function of $\theta$ on $A^x_\delta$ and thus $\sum_{v' \in V_{\ell-1}} 
w_{v' \to v}
\left( f_\theta^{\ell-1}(x) \right)_{v'}
+
b_v$ also is. Since the variable $b_v$ is not present in the expression of $ f_\theta^{\ell-1}(x)$, it only appears in a single monomial of degree and coefficient $1$ of $\sum_{v' \in V_{\ell-1}} 
w_{v' \to v}
\left( f_\theta^{\ell-1}(x) \right)_{v'}
+
b_v$. The latter polynomial function is therefore non-constant. Hence the set  $\cT^x_v \cap A^x_\delta$, constituted by the zeros of this polynomial function, has Lebesgue measure zero. Since $ \cup_\delta A^x_\delta =\RR^E \times \RR^B$, we finally conclude that, for any $\ell \in \lb 1 , L-1 \rb$ and $v \in V_{\ell}$, $\cT^x_v$ has Lebesgue measure zero.

The set 
\[
\cT^x 
=
\cup_{\ell=1}^{L-1}
\cup_{v \in V_{\ell}} \cT^x_v
\]
is thus also of Lebesgue measure zero.

Let us now prove the set equality:
\begin{equation}\label{expression:of:T}
\bigcup_\delta \partial A^x_\delta = \cT^x . 
\end{equation}

We  first show the inclusion $\bigcup_\delta \partial A^x_\delta \subseteq \cT^x$. Consider $\delta \in \{ 0,1 \}^{N_1 + \dots + N_{L-1}}$ and let us now show that $\partial A^x_\delta \subseteq \cT^x$. To do so, consider $\theta \in \partial A^x_\delta $. Since  $\theta \not\in \Inter(A^x_\delta)$ and $\cup_{\delta'} A^x_{\delta'} = \RR^E \times \RR^B$, for any $\varepsilon$ there exists $\delta_\varepsilon \neq \delta$ such that $B(\theta, \varepsilon) \cap A_{\delta_\varepsilon} \neq \emptyset$. Since the set of all possible $\delta_\varepsilon$ is finite, we are sure that there exists $\delta' \neq \delta$ such that $\theta \in \overline{A}_{\delta'}$. Let $\ell \in \lb 1 , L-1 \rb$ and $v \in V_{\ell}$ such that $\delta_v \neq \delta'_v$. We assume without loss of generality that $\delta_v= 0$. The proof  is indeed similar when $\delta_v= 1$. There exists $(\theta_n)_{n\in\NN} \in (A^x_\delta)^{\NN}$ such that $\theta_n \to \theta$ as $n \to \infty$ and there exists $(\theta'_n)_{n\in\NN}  \in (A^x_{\delta'})^{\NN}$ such that $\theta'_n \to \theta$ as $n \to \infty$. We have $a_v(x , \theta_n) = 0$ and $a_v(x , \theta'_n) = 1$ for all $n$.
%(Note that is is possible that only the opposite signs $a_v(x_n , \theta_n) = 1$ and 
%$a_v(x'_n , \theta'_n) = 0$ can hold for some $j \neq i$, but this case would be handled similarly as the case $a_v(x_n , \theta_n) = 0$ and 
%$a_v(x'_n , \theta'_n) = 1$ that we treat below.)

Using that $\theta \longmapsto \sum_{v' \in V_{\ell-1}} 
w_{v' \to v}
\left( f_\theta^{\ell-1}(x) \right)_{v'}
+
b_v$ is continuous and taking the limit of this function at $\theta_n$, as $n$ goes to infinity, we obtain that $\sum_{v' \in V_{\ell-1}} 
w_{v' \to v}
\left( f_\theta^{\ell-1}(x) \right)_{v'}
+
b_v \leq 0 $. Reasoning similarly with the sequence $ (\theta'_n)_{n\in\NN}$ we obtain the reverse inequality and conclude that $\sum_{v' \in V_{\ell-1}} 
w_{v' \to v}
\left( f_\theta^{\ell-1}(x) \right)_{v'}
+
b_v = 0$. This shows that $\theta \in \cT^x_v \subseteq \cT^x$. This finishes the proof of $\partial A^x_\delta \subseteq \cT^x$. 

Let us now show the reciprocal inclusion $ \cT^x \subseteq \bigcup_\delta \partial A^x_\delta$. Indeed, let $\theta \in \cT^x$. There exist $\ell \in \lb 1,L-1 \rb$ and $v \in V_\ell$ such that $\theta \in \cT^x_v$. There also exists $\delta \in \{ 0,1 \}^{N_1 + \dots + N_{L-1}}$ such that $\theta \in A^x_\delta.$ In particular, since $\sum_{v' \in V_{\ell-1}} 
w_{v' \to v}
\left( f_\theta^{\ell-1}(x) \right)_{v'}
+
b_v
=0$, we have $\delta_v = a_v(x, \theta) = 1$. For any $\varepsilon > 0$, by replacing $b_v$ by $b_v - \varepsilon$, we obtain a $\theta_\varepsilon$ satisfying $\| \theta - \theta_\varepsilon \| \leq \varepsilon$ and $a_v(x, \theta_\varepsilon) = 0 \neq \delta_v$, which shows $ \theta_\varepsilon \not\in A^x_\delta$. This shows $\theta \in \partial A^x_\delta \subseteq \bigcup_\delta \partial A^x_\delta$. This shows the desired inclusion, and thus the equality \eqref{expression:of:T}.

 For all $\delta \in \{ 0,1 \}^{N_1 + \dots + N_{L-1}}$, $\partial A^x_\delta$ is closed by definition of a boundary. Therefore $ \cT^x = \bigcup_\delta \partial A^x_\delta$ is also closed. Also, since $ \cT^x$ has been shown to have Lebesgue measure zero, $\partial A^x_\delta\subset \cT^x$ has Lebesgue measure zero for all $\delta$.
 
 This concludes the proof of  Lemma \ref{lemma:valeurs:de:a:et:polynome}.
\end{proof}

We state and prove another lemma before proving \Cref{theorem:constant:rank}. The lemma resembles \Cref{theorem:constant:rank} but does not include the statements on $\RK{\DT f_\theta(X)}$. 

For $n \in \NN^*$ and $X\in\RR^{N_0\times n}$, we define
\begin{equation}\label{def-ctX}
   \cT^X = \cup_{i=1}^n  \cT^{x^{(i)}},
\end{equation}

where $\cT^x$ is defined in \eqref{def_Tx} and \eqref{def_Txv}.

\begin{lem} \label{lemma:valeurs:de:a:et:polynome:deux}

        For all $n \in \NN^*$, for all $X \in \Xspace$, the sets $\widetilde{\cU}^X_1,\ldots,\widetilde{\cU}^X_{{p_X}}$ defined in \eqref{def:u:j:tilde} are non-empty, open and disjoint, and they satisfy
        \begin{itemize}
            \item $\left( \cup_{j=1}^{{p_X}} \widetilde{\cU}^X_j \right)^c = \cT^X$, and in particular the complement $\left( \cup_{j=1}^{{p_X}} \widetilde{\cU}^X_j \right)^c$ is a closed set with Lebesgue measure zero;
            \item for all $j \in \lb 1, {p_X} \rb$, the function $\theta \longmapsto a(X,\theta)$ is constant on each $\widetilde{\cU}^X_j$ and takes ${p_X}$ distinct values on $\cup_{j=1}^{p_X} \widetilde{\cU}^X_j$;
            \item for all $j \in \lb 1,{p_X} \rb$, the mapping $\theta \longmapsto f_{\theta}(X)$ is polynomial of degree $L$ on $\widetilde{\cU}^X_j$, when $\sigma_L=Id$, and it is analytic otherwise. 
        \end{itemize}

\end{lem}

\begin{proof}[Proof of Lemma \ref{lemma:valeurs:de:a:et:polynome:deux}]

Throughout the proof we consider a fixed $n \in \NN^*$ and a fixed $X \in \Xspace$.

By definition, see \eqref{def:u:j:tilde}, the sets $\widetilde{\cU}^X_1 , \ldots  , \widetilde{\cU}^X_{{p_X}}$ are non-empty, open and disjoint. Before proving the first item of the lemma, let us notice that $\cT^X$ is closed and of Lebesgue measure zero since, for all $i\in\lb1,n\rb$, the third item of Lemma \ref{lemma:valeurs:de:a:et:polynome} states that $\cT^{x^{(i)}}$ is closed and has Lebesgue measure zero. Let us also write the following useful characterization: thanks to the characterization of $\cT^x$ in the third item of Lemma \ref{lemma:valeurs:de:a:et:polynome}, we have
\begin{equation}\label{charac:of:T_n}
\cT^X
=
\cup_{i=1}^n 
\cup_{\delta \in \{ 0,1 \}^{N_1 + \dots + N_{L-1}}}
\partial A^{x^{(i)}}_\delta.
\end{equation}

Let us now prove that  $ ( \cup_{j=1}^{{p_X} }\widetilde{\cU}^X_j )^c = \cT^X$.

To do so, let us first show that $ ( \cup_{j=1}^{{p_X} }\widetilde{\cU}^X_j )^c \subseteq \cT^X$.
Let $\theta \in ( \cup_{j=1}^{{p_X} }\widetilde{\cU}^X_j )^c $. Consider the $\Delta^X_1, \dots, \Delta^X_{q^X}$ defined just before \eqref{def:u:j:tilde}. There exists $j \in \lb 1 , q_X \rb$ such that $a(X , \theta) = \Delta^X_j$. Since $\theta \not \in \widetilde{\cU}^X_j$, there exists a sequence $(\theta_k)_{k\in\NN}$ such that $\theta_k \rightarrow \theta$, as $k \rightarrow \infty$ and $a(X , \theta_k) \neq \Delta^X_j$, for all $k$. Modulo the extraction of a sub-sequence, we can assume that there exists $i\in\lb 1, n \rb$ such that for all $k \in \NN$, $a(x^{(i)} , \theta_k)  \neq \delta$, where $x^{(i)}$ is the $i^{th}$ column of $X$, and $\delta$ is the $i^{th}$ column of $\Delta^X_j$.
Thus, we have $\theta_k \not\in A^{x^{(i)}}_{\delta}$, for all $k$,  and since $\theta\in A^{x^{(i)}}_{\delta}$, we conclude $\theta \in \partial A^{x^{(i)}}_{\delta}$. The characterization \eqref{charac:of:T_n} thus shows $ \theta \in \cT^X$. This shows $ ( \cup_{j=1}^{{p_X} }\widetilde{\cU}^X_j )^c \subseteq \cT^X$.

Let us now show that $\cT^X \subseteq  ( \cup_{j=1}^{{p_X} }\widetilde{\cU}^X_j )^c $. If $\theta \in \cT^X$, there exists $i \in \lb 1, n \rb$ and $\delta \in \{ 0,1 \}^{N_1 + \dots + N_{L-1}}$ such that $\theta \in \partial A^{x^{(i)}}_\delta$. Thus, for any $\varepsilon >  0$, $\theta' \longmapsto a(x^{(i)}, \theta')$ is not constant over $B(\theta, \varepsilon)$. As a consequence, $\theta$ does not belong to any of the open sets $\widetilde{\cU}^X_j$. This finishes the proof of $\cT^X =  ( \cup_{j=1}^{{p_X} }\widetilde{\cU}^X_j )^c $. 

Since, as said above, $\cT^X$ is closed and of Lebesgue measure zero, $( \cup_{j=1}^{{p_X} }\widetilde{\cU}^X_j )^c$ is too. This ends the proof of the first item of the lemma.

The second item holds by definition of $\widetilde{\cU}^X_1 , \ldots , \widetilde{\cU}^X_{{p_X}}$. 

Let us now show the third item. Let $j \in \lb 1, {p_X} \rb$. The function $\theta \longmapsto a(X,\theta)$ is constant on $\widetilde{\cU}^X_j$.
The set $\widetilde{\cU}^X_j$ is associated to $\Delta^X_j$ in \eqref{def:u:j:tilde} and  the latter is of the form $(\delta^{1} , \ldots ,\delta^{n})$ with $\delta^{i} \in \{ 0,1 \}^{N_1 + \dots + N_{L-1}}$, for $i \in \lb 1,n \rb$.
Fix $i' \in \lb 1,n \rb$.
Then for $X = (x^{(i)})_{i \in \lb 1,n \rb}$
with $\theta \in \widetilde{\cU}^X_j$, $ \theta \in A^{x^{(i')} }_{\delta^{i'}}$.
Hence, Lemma \ref{lemma:valeurs:de:a:et:polynome} second item shows that $\theta \longmapsto f_{\theta}(x^{(i')})$ is a polynomial function of degree $L$, when $\sigma_L=Id$, or an analytic function of $\theta$. The quantity $f_{\theta}(X)$ is a matrix whose columns are $f_{\theta}(x^{(i)})$, $i \in \lb 1,n \rb$. Hence $\theta \longmapsto f_{\theta}(X)$ is a polynomial function of degree $L$, when $\sigma_L=Id$, or is an analytic function on $\widetilde{\cU}^X_j$. 

This concludes the proof of Lemma \ref{lemma:valeurs:de:a:et:polynome:deux}. 
\end{proof}

\begin{proof}[Proof of Theorem \ref{theorem:constant:rank}]

    Consider $n \in \NN^*$ and $X\in\Xspace$. The sets ${\cU}^X_1, \dots, {\cU}^X_{p_X}$ are non-empty by definition of $r^X_1, \dots, r^X_{p_X}$, and they are disjoint because of the inclusion ${\cU}^X_j \subseteq \widetilde{\cU}^X_j$, for all $j$, and because the sets $\widetilde{\cU}^X_1,\ldots,\widetilde{\cU}^X_{{p_X}}$ are disjoint as shown in Lemma \ref{lemma:valeurs:de:a:et:polynome:deux}. 
    Hence the first item holds.
    The second item is a direct consequence of the definition of $\widetilde{\cU}^X_j$, in \eqref{def:u:j:tilde}.
    The third item holds by the definition of ${\cU}^X_j$, in \eqref{defU}.
    
    To see that ${\cU}^X_j$ is open, for all $j$, first recall that $\widetilde{\cU}^X_j$ is open, then note that since the function $\theta \longmapsto f_{\theta}(X)$ is polynomial or analytic over $\widetilde{\cU}^X_j$ (by Lemma \ref{lemma:valeurs:de:a:et:polynome:deux}, third item), the function $\theta \longmapsto \DT f_{\theta}(X)$ is continuous over $\widetilde{\cU}^X_j$. Consider $\theta\in{\cU}^X_j$, since the rank is lower semicontinuous and $\RK{\DT f_\theta(X)} = r^X_j $, there exists $\epsilon > 0$ such that for any  $
    \theta'\in B(\theta, \epsilon)$, we have $\RK{\DT f_{\theta'}(X)} \geq r^X_j-\frac{1}{2}$, which using the fact that $\RK{.}$ takes integer values and the maximality of $r^X_j$, is equivalent to $\RK{\DT f_{\theta'}(X)} = r^X_j$ and to $ \theta' \in {\cU}^X_j$. Summarizing, for any $\theta\in{\cU}^X_j$, there exists $\epsilon > 0$ such that $B(\theta, \epsilon)\subset {\cU}^X_j$. This shows that ${\cU}^X_j$ is open. Hence Item 4 holds.

    Item 6, stating that $\theta \longmapsto f_{\theta}(X)$ is polynomial or analytic on $\widetilde{\cU}^X_j$, comes directly from the last item of Lemma \ref{lemma:valeurs:de:a:et:polynome:deux} and from the inclusion ${\cU}^X_j \subseteq \widetilde{\cU}^X_j$.
    
    To finish the proof, we need to prove Item 5. The fact that $\left( \cup_{j=1}^{p_X} \widetilde{\cU}^X_j \right)^c$ is a closed set with Lebesgue measure zero follows from \Cref{lemma:valeurs:de:a:et:polynome:deux}, first item. The set $\left( \cup_{j=1}^{p_X} {\cU}^X_j \right)^c$ is closed because, as already proved, ${\cU}^X_j$ is open for all $j$. 
    
    Let us prove  that
    $\left( \cup_{j=1}^{p_X} {\cU}^X_j \right)^c$ is of Lebesgue measure zero. We consider a basis $\cB$ of $\RR^E \times \RR^B$ and a basis $\cB'$ of $\Yspace$. Let us write $Jf_{\theta}(X)$ for the matrix of the differential $\DT f_{\theta}(X)$ of the function $\theta \longmapsto f_\theta(X)$ in these two bases. Then $\theta \longmapsto Jf_{\theta}(X)$ is an analytic function on $\widetilde{\cU}^X_j$. Recall the notation $r^X_j = \max_{ \theta \in \widetilde{\cU}^X_j} \RK{\DT f_\theta(X)}$, and let $\theta' \in \widetilde{\cU}^X_j$ such that $\RK{ \DT f_{\theta'}(X) } = r^X_j$. We thus have $\RK{ Jf_{\theta'} (X)} = r^X_j$, and thus there exists a sub-matrix $N_{\theta'}(X)$ of $Jf_{\theta'}(X)$, of size $r^X_j \times r^X_j$, such that $\det N_{\theta'}(X) \neq 0$. The function $\theta \longmapsto Jf_{\theta}(X)$ is a polynomial or analytic function on $\widetilde{\cU}^X_j$ and thus $\theta\longmapsto \det (N_{\theta}(X))$ also is. This latter function is not uniformly zero on $\widetilde{\cU}^X_j$ and thus the set of its zeros, which we write $\cY_j$, is a closed set of Lebesgue measure zero \citep{mityagin2015zero}.

    For all $\theta \in  \widetilde{\cU}^X_j \setminus \cY_j$, we have $\det N_{\theta}(X) \neq 0$ and thus $\RK{ N_{\theta}(X) } = r^X_j$ and thus $\RK{ Jf_{\theta}(X) } \geq r^X_j$. We also have $\RK{ Jf_{\theta}(X) } = \RK{ \DT f_{\theta}(X) } \leq r^X_j$ by definition of $r^X_j$. Hence $\RK{ \DT f_{\theta}(X) } = r^X_j$. This shows $\widetilde{\cU}^X_j \setminus \cY_j \subseteq {\cU}^X_j$.

Finally, 
    \begin{align*}
        \left( 
        \cup_{j=1}^{p_X} {\cU}^X_j
        \right)^c 
        = & 
        \cap_{j=1}^{p_X} \left({\cU}^X_j\right)^c 
        \\ 
        \subseteq & 
        \cap_{j=1}^{p_X} 
         \left( 
       \widetilde{\cU}^X_j \setminus \cY_j
        \right)^c 
        \\ 
        = & 
         \cap_{j=1}^{p_X} 
         \left( 
       \left(\widetilde{\cU}^X_j\right)^c \cup  \cY_j
        \right) 
        \\ 
        \subseteq & 
        \cap_{j=1}^{p_X} 
         \left( 
         \left(\widetilde{\cU}^X_j\right)^c 
         \cup 
         \left(
         \cup_{j'=1}^{p_X}
       \cY_{j'} 
       \right)
        \right) 
        \\
        = & 
        \left( \cap_{j=1}^{p_X} 
        \left(\widetilde{\cU}^X_j \right)^c 
        \right) 
        \cup 
                \left( \cup_{j=1}^{p_X} 
\cY_j
        \right) 
        \\ 
         = & 
        \left( \cup_{j=1}^{p_X} 
        \widetilde{\cU}^X_j
        \right)^c  
        \cup 
                \left( \cup_{j=1}^{p_X} 
\cY_j
        \right). 
    \end{align*}
    We know from  Lemma \ref{lemma:valeurs:de:a:et:polynome:deux} first item that $\left( \cup_{j=1}^{p_X} 
        \widetilde{\cU}^X_j 
        \right)^c $ has Lebesgue measure zero. Also each $\cY_j$ has Lebesgue measure zero, thus $\cup_{j=1}^{p_X} 
\cY_j$ has Lebesgue measure zero. Hence, $ \left( 
        \cup_{j=1}^{p_X} {\cU}^X_j
        \right)^c $ has Lebesgue measure zero.
        
        This concludes the proof of Theorem \ref{theorem:constant:rank}. 
\end{proof}

\subsection{Proof of Proposition \ref{prop:rank:invariant:rescaling}}
\label{app:proof:rank:invariant:rescaling}

Let $\widetilde \theta \sim \theta$, let $n \in \NN^*$ and let $X \in \Xspace$. 

By definition of the relation $\sim$, in Section \ref{ReLU networks-sec-main}, there is an invertible linear map $M: \Par \longrightarrow \Par$ such that $ \widetilde\theta = M \theta $. Note that when expressed in the canonical basis of $\Par$, the matrix corresponding to $M$ is the product of a permutation matrix and a diagonal matrix, with strictly positive diagonal components whose values are given by \eqref{eq:scaling:un} and \eqref{eq:scaling:deux}. Notice that since $M$ corresponds to positive rescalings and neuron permutations, as discussed after \eqref{eq:scaling:deux}, we have, 
\begin{equation}\label{etoutgnj}
\mbox{for any }\theta'\in\Par, \qquad f_{\theta'}(X) = f_{M\theta'}(X).
\end{equation}

Assume that $\DT f_{\theta}(X)$ is well-defined, i.e. the map $\theta' \mapsto f_{\theta'}(X)$ is differentiable at $\theta$. Then, for all $u\in\Par$, the following calculation holds, using the fact that $M$ is invertible, using \eqref{etoutgnj} and using
\eqref{diff-def-eq},
\begin{align*}
    f_{\widetilde \theta + u}(X) =f_{M \theta + u}(X)
    =& 
    f_{M (\theta + M^{-1} u)}(X) 
    \\
     = &
     f_{\theta + M^{-1} u}(X) 
     \\ 
     = & f_{\theta}(X) + \DT f_{\theta}(X) (M^{-1} u  ) 
     + 
     o( \| M^{-1} u \|  )
     \\ 
     =& 
     f_{\theta}(X) + \DT f_{\theta}(X) (M^{-1} u  ) 
     + 
     o( \|  u \|  ).
\end{align*}
Hence, $\theta' \longmapsto f_{\theta'}(X)$ is differentiable at $\widetilde\theta$ and for all $u \in \Par$,
\[
\DT f_{\widetilde\theta}(X) ( u  )
=
\DT f_{\theta}(X) (M^{-1} u  ).
\]
Since $M^{-1}$ is invertible, it follows that 
 $\RK{ \DT f_{\widetilde \theta}(X) } = \RK{ \DT f_{\theta}(X) }$. 

This concludes the proof of Proposition \ref{prop:rank:invariant:rescaling}.

\section{Proofs and Calculations of \texorpdfstring{\Cref{section:geometric:interpretation}}{section 4}}

\subsection{Applying the Constant Rank Theorem to Obtain \Cref{coro-optim}}\label{appendix-constant-rank-thm} 

Since this is the central argument linking the regularity of the learned neural network to the flatness of the objective function, we recall, for completeness, the classical geometric reasoning leading to \Cref{coro-optim}.

Let us first recall the constant rank theorem.
\begin{thm}[Constant Rank Theorem]\label{cste-rk-thm}
Let \( U \subset \mathbb{R}^n \) be an open set, \( a \in U \), and let
\( g : U \to \mathbb{R}^p \) be a \(\mathcal{C}^1\) mapping.
If the differential of \( g \) has constant rank \( r \) on \( U \), then there exist:
\begin{itemize}
    \item a \(\mathcal{C}^1\)-diffeomorphism
    \( \varphi \) from an open set \( V \subset \mathbb{R}^n \) containing \( 0 \)
    onto an open subset of \( U \), with \( \varphi(0) = a \), and
    \item a \(\mathcal{C}^1\)-diffeomorphism
    \( \psi \) from an open subset of \( \mathbb{R}^p \) containing \( g(\varphi(V)) \)
    onto an open subset of \( \mathbb{R}^p \), with \( \psi(g(a)) = 0 \),
\end{itemize}
such that for all \( x = (x_1, \ldots, x_n) \in V \),
\begin{equation} \label{eq:constant:rank:ccl}
(\psi \circ g \circ \varphi)(x) = (x_1, \ldots, x_r, 0, \ldots, 0).
\end{equation}
\end{thm}

In the context of our problem, set:
\[
n = |E| + |B|, \quad
U = \cU_j^X, \quad
a = \theta, \quad
g : \theta' \mapsto f_{\theta'}(X), \quad
p = n n_L, \quad
r = \RK{ \DT f_{\theta}(X) }.
\]
\Cref{theorem:constant:rank} guarantees that the hypotheses of \Cref{cste-rk-thm} hold.

Let \( \varepsilon_{X,\theta} \) be such that \( B(\theta, \varepsilon_{X,\theta}) \subset \varphi(V) \), and define \( V' = \varphi^{-1}\left(B(\theta, \varepsilon_{X,\theta})\right)
\subseteq V
\).
Then, to prove the first item of \Cref{coro-optim}, it suffices to show that \( \psi^{-1} \) is a smooth chart from
\[
\psi \circ g \circ \varphi\left( V' \right),
\]
which, using \eqref{eq:constant:rank:ccl}, satisfies $\psi \circ g \circ \varphi\left( V' \right) = W \times \{0\} $ for  an open set $W$ of $\mathbb{R}^r$ containing $0$,
onto
\[
\bigl\{ f_{\theta'}(X) \in \Yspace \mid \|\theta' - \theta\| < \varepsilon_{X,\theta} \bigr\}.
\]

Let us prove this. Indeed,  $\psi^{-1}$ is smooth and invertible by definition. Let us verify that it indeed maps the two sets mentioned above to one another.

For any $y \in \psi \circ g \circ \varphi\left( V'  \right)$, there is $x \in V'$ such that $y = \psi \circ g \circ \varphi (x)$ and thus $\psi^{-1}(y) = g \circ \varphi (x)= f_{\varphi(x)}(X) \in \{ f_{\theta'} (X) \mid \theta' \in\Par\}
$. 
Also, because $x \in V'$, we have $\varphi (x) \in B(\theta, \varepsilon_{X,\theta})$ and therefore $\psi^{-1}(y) = f_{\varphi(x)}(X) \in \bigl\{ f_{\theta'}(X)  \mid \|\theta' - \theta\| < \varepsilon_{X,\theta} \bigr\}$.

Conversely, let $y \in \bigl\{ f_{\theta'}(X)  \mid \|\theta' - \theta\| < \varepsilon_{X,\theta} \bigr\}$. Then, there is $\theta'$ with $\| \theta - \theta' \| < \varepsilon_{X,\theta} $ such that $y = f_{\theta'} (X) = g(\theta')$ and $(\psi^{-1})^{-1} (y) = \psi \circ g (\theta')$. We can write $\theta' =  \varphi (x)$ with $x \in V' $
and so $(\psi^{-1})^{-1} (y) = \psi \circ g \circ \varphi (x) \in \psi \circ g \circ \varphi (V')$. This concludes the proof of the first item of \Cref{coro-optim}.

To prove the second item of \Cref{coro-optim}, it suffices to show that the map $\varphi$ 
is a smooth chart from
\[
V' \cap \left( \{0\} \times \mathbb{R}^{n-r} \right)
\]
onto
\[
\bigl\{ \theta' \in \mathbb{R}^E \times \mathbb{R}^B \mid f_{\theta'}(X) = f_{\theta}(X) \text{ and } \|\theta' - \theta\| < \varepsilon_{X,\theta} \bigr\}.
\]

Let us prove this. Indeed, $\varphi$ is smooth and invertible.  Let us verify that it indeed maps the two sets mentioned above to one another.

Consider $x \in V' \cap \left( \{0\} \times \mathbb{R}^{n-r} \right)$ and denote $\theta' = \varphi(x)$. Using \eqref{eq:constant:rank:ccl}, we have $\psi \circ g \circ \varphi(x) = 0$, and thus, using  \Cref{cste-rk-thm} again, we have
\[f_{\theta'}(X) = g \circ \varphi(x) = \psi^{-1}(0)= g(a) = f_{\theta}(X).
\]
Also, since $x \in V'$, $\|\theta' - \theta\|\leq \varepsilon_{X,\theta}$. We finally conclude that
$\varphi(x) \in \bigl\{ \theta' \in \mathbb{R}^E \times \mathbb{R}^B \mid f_{\theta'}(X) = f_{\theta}(X) \text{ and } \|\theta' - \theta\| < \varepsilon_{X,\theta} \bigr\}$.

Conversely, for $y \in \bigl\{ \theta' \in \mathbb{R}^E \times \mathbb{R}^B \mid f_{\theta'}(X) = f_{\theta}(X) \text{ and } \|\theta' - \theta\| < \varepsilon_{X,\theta} \bigr\}$, we have $g(y)=g(a)$ and, using \Cref{cste-rk-thm}, $\psi \circ g (y) =\psi (g (a)) =0$. Let $x = (x_1,\ldots,x_n) = 
\varphi^{-1}(y) \in V'$. From \eqref{eq:constant:rank:ccl}, $0 = \psi \circ g (y) = \psi \circ g \circ \varphi(x) = (x_1,\ldots,x_r,0 , \dots , 0) $ and thus $x \in \{0\} \times \mathbb{R}^{n-r}$. Thus, $\varphi^{-1}(y) \in V' \cap \left( \{0\} \times \mathbb{R}^{n-r} \right)$. This concludes the proof of the second (and last) item of \Cref{coro-optim}.
 
\subsection{Calculations for the Example in Section \ref{example:image:preimage:dimdeux}}\label{appendix-example}

We provide in this appendix, the calculations permitting to construct Figure \ref{dessin-exemple}. We consider a one-hidden-layer neural network of widths $N_0=N_1=N_2=1$, with the identity activation function on the last layer. To simplify notations, we denote the weights and biases $\theta = (w,v,b,c)\in\RR^4$ so that $f_\theta(x) = v \sigma(wx+b)+c$, for all $x\in\RR$. We consider $X=(0, 1, 2)\in\RR^{1\times 3}$ and 
\[f_\theta(X) = \bigl(v \sigma(b) + c ~,~ v \sigma(w+b) + c ~,~v \sigma(2w+b) + c\bigr).
\]
%\[
%f_\theta(X)^T = \left(\begin{array}{c}
%v \sigma(b) + c \\
%v \sigma(w+b) + c \\
%v \sigma(2w+b) + c
%end{array}\right).
%\]

The boundaries of the sets $\widetilde \cU_j^X$, corresponding to the parameters having the same activation pattern, are defined by the equation $b=0$, $w+b=0$ and $2w+b = 0$. There are $6$ possible activation patterns corresponding to the zones represented, in the $(w,b)$ plane, on the left of Figure \ref{dessin-exemple}.

Since the sets $f_{\widetilde \cU^X_j}(X) = \{f_\theta(X) ~|~ \theta \in \widetilde \cU^X_j\}$, for $j\in \lb1,6\rb$, are invariant to translations by vectors $(c,c,c)$, for $c\in\RR$, we consider the plane $\cP$ orthogonal to the vector $(1,1,1)$ and parameterize its elements using the mapping
\begin{eqnarray*}
p : \RR^2 & \longrightarrow & \cP \\    
(x,y) & \longmapsto & \frac{x}{\sqrt{6}} (1,1,-2) + \frac{y}{\sqrt{2}} (-1,1,0).
\end{eqnarray*}
Instead of representing $f_{\widetilde \cU^X_j}(X)$, we represent on the right of Figure \ref{dessin-exemple} its intersection with $\cP$, formally defined as the set $\cV_j \subseteq \RR^2$ such that
\[
f_{\widetilde \cU^X_j}(X) = \left\{ p(x,y) + (z,z,z) \in \RR^{1\times 3}~|~ (x,y) \in \cV_j \mbox{ and } z\in\RR\right\}.
\]
Below, we construct the sets $\cV_j$, for $j\in \lb1,6\rb$.
\paragraph{Case $j=1$:}We have $b\leq 0$, $2w+b \leq 0$ and therefore $w+b\leq 0$. This leads to $f_{(w,v,b,c)}(X) = (c,c,c)$ and $\cV_1=\{(0,0)\}$. 
\paragraph{Case $j=2$:}We have $b\geq 0$, $w+b \leq 0$ and therefore $2w+b\leq 0$. This leads to $f_{(w,v,b,c)}(X) = (vb+c,c,c)$ and
\[\cV_2 = \bigl\{(x,y) \in\RR^2~|~ \exists (w,v,b,c)\in \widetilde \cU^X_2, p(x,y) = (vb+c,c,c)\bigr\}.
\]
Solving
\[\left\{\begin{array}{rrl}
(1) : &vb+c & =\frac{x}{\sqrt{6}} - \frac{y}{\sqrt{2}} \\
(2) : &c & = \frac{x}{\sqrt{6}} + \frac{y}{\sqrt{2}} \\
(3) : &c & =  -2 \frac{x}{\sqrt{6}}
\end{array}\right. \qquad\Longleftrightarrow \qquad
\left\{\begin{array}{rrl} 
(1)-(2):& - \sqrt{2} y &= 
vb \\
\sqrt{2}((2)-(3)):&y &= - \sqrt{3} x\\
(3) :& c &= -2 \frac{x}{\sqrt{6}}
\end{array}\right.
\] 
 and we obtain
\[\cV_2 = \bigl\{ (x,y) \in \RR^2 ~|~ \sqrt{3} x + y = 0\bigr\}.
\]

\paragraph{Case $j=3$:}We have $b\geq 0$, $w+b \geq 0$ and  $2w+b\leq 0$. This leads to $f_{(w,v,b,c)}(X) = (vb+c,v(w+b)+c,c)$ and
\[\cV_3 = \bigl\{(x,y) \in\RR^2~|~ \exists (w,v,b,c)\in \widetilde \cU^X_3, p(x,y) = (vb+c,v(w+b)+c,c)\bigr\}.
\]
We have
\[\left\{\begin{array}{rrl}
(1): &vb+c & =\frac{x}{\sqrt{6}} - \frac{y}{\sqrt{2}} \\
(2): &v(w+b)+c & = \frac{x}{\sqrt{6}} + \frac{y}{\sqrt{2}} \\
(3): &c & =  -2 \frac{x}{\sqrt{6}}
\end{array}\right. \qquad \Longleftrightarrow \qquad
\left\{\begin{array}{rrl}
\sqrt{2}((1)-(3)):& \sqrt{3} x &= y+ \sqrt{2} vb \\
(2) - (1):& \sqrt{2} y& = vw \\
(3):& c &= -2 \frac{x}{\sqrt{6}}
\end{array}\right. .
\]
Using $(w,v,b,c) \in\widetilde \cU_3^X$, we obtain $b\in[-w, -2w]$, where we recall that $w\leq 0$.
\begin{itemize}
\item Taking, for simplicity, $v=1$ and choosing the value of $w$, the second equation shows that we can reach any $y = \frac{w}{\sqrt{2}}\leq 0$. Moreover, as $b$ goes through $[-w, -2w]$, $\sqrt{2} vb$ goes through $[-\sqrt{2}w, -2\sqrt{2}w]=[-2y, -4y]$. Therefore, we see with the first equation that $\sqrt{3}x$ goes through $[-y,-3y]$, that is $x$ goes through $[-\frac{y}{\sqrt{3}}, -\sqrt{3}y]$. It is not possible to reach other values for other values when $v\geq 0$.
\item Similarly, taking $v=-1$ and choosing the value of $w$, the second equation shows that we can reach any $y = -\frac{w}{\sqrt{2}}\geq 0$. Moreover, as $b$ goes through $[-w, -2w]$, $\sqrt{2} vb$ goes through $[ 2\sqrt{2}w, \sqrt{2}w]=[ -4y, -2y]$. Therefore, we see with the first equation that $\sqrt{3}x$ goes through $[-3y,-y]$, that is $x$ goes through $[-\sqrt{3}y,-\frac{y}{\sqrt{3}}]$. Again, it is not possible to reach other values for other values when $v\leq 0$.
\end{itemize}

Finally, the set $\cV_3$ is the set in between the two lines $x+\sqrt{3}y=0$ and $\sqrt{3}x+y=0$, as on the right of Figure \ref{dessin-exemple}.

\paragraph{Case $j=4$:}We have $b\geq 0$, $w+b \geq 0$ and  $2w+b\geq 0$. This leads to $f_{(w,v,b,c)}(X) = (vb+c,v(w+b)+c,v(2w+b)+c)$ and
\[\cV_4 = \bigl\{(x,y) \in\RR^2~|~ \exists (w,v,b,c)\in \widetilde \cU^X_4, p(x,y) = (vb+c,v(w+b)+c,v(2w+b)+c)\bigr\}.
\]
We have
\[\left\{\begin{array}{rrl}
(1): &vb+c & =\frac{x}{\sqrt{6}} - \frac{y}{\sqrt{2}} \\
(2): &v(w+b)+c & = \frac{x}{\sqrt{6}} + \frac{y}{\sqrt{2}} \\
(3): &v(2w+b)+c & =  -2 \frac{x}{\sqrt{6}}
\end{array}\right. \qquad \Longleftrightarrow \qquad
\left\{\begin{array}{rrl}
(2)-(1): &\sqrt{2} y &= vw \\
(3)-(2): & - 3 \frac{x}{\sqrt{6}} - \frac{y}{\sqrt{2}} &= vw\\
(3):&v(2w+b)+c &= -2 \frac{x}{\sqrt{6}}
\end{array}\right. 
\]
which is equivalent to
\[
 \left\{\begin{array}{rrl}
(1): & \sqrt{2} y &= vw \\
\sqrt{2}( (2) - (1) ) : & 3y &= - \sqrt{3} x \\
(3): & v(2w+b)+c &= -2 \frac{x}{\sqrt{6}}
\end{array}\right. 
\]
This leads to
\[\cV_4 = \bigl\{ (x,y) \in \RR^2 ~|~   x + \sqrt{3} y = 0\bigr\}.
\]

\paragraph{Case $j=5$:}We have $b\leq 0$, $w+b \geq 0$ and therefore $2w+b\geq 0$. This leads to $f_{(w,v,b,c)}(X) = (c,v(w+b)+c,v(2w+b)+c)$ and
\[\cV_5 = \bigl\{(x,y) \in\RR^2~|~ \exists (w,v,b,c)\in \widetilde \cU^X_5, p(x,y) = (c,v(w+b)+c,v(2w+b)+c)\bigr\}.
\]
We have
\[\left\{\begin{array}{rrl}
(1): & c & =\frac{x}{\sqrt{6}} - \frac{y}{\sqrt{2}} \\
(2): &v(w+b)+c & = \frac{x}{\sqrt{6}} + \frac{y}{\sqrt{2}} \\
(3): &v(2w+b)+c & =  -2 \frac{x}{\sqrt{6}}
\end{array}\right. \qquad \Longleftrightarrow \qquad
\left\{\begin{array}{rrl}
(1): & \frac{x}{\sqrt{6}} - \frac{y}{\sqrt{2}} &= c \\
(2) - (1): &\sqrt{2} y &= v(w+b)\\
(3)-(1): &- 3 \frac{x}{\sqrt{6}} + \frac{y}{\sqrt{2}} &= v(2w+b)
\end{array}\right. .
\]
Using $v$ and $w+b\geq 0$, we see with the second equation that $y$ can take any value in $\RR$. Let us consider an arbitrary fixed $y \in \RR$. In fact, there are infinitely many choices for $v, w$ and $b$ corresponding to this value. Taking $v=sign(y)$, we have $w+b = sign(y) \sqrt{2} y = \sqrt{2} |y|$ and, since $b\leq 0$, $w$ can take any value in $[\sqrt{2} |y|, +\infty)$. Therefore, $2w+b=w+(w+b)$ can take any value in $[2\sqrt{2} |y|, +\infty)$.
\begin{itemize}
    \item If $y\geq 0$: $- 3 \frac{x}{\sqrt{6}} + \frac{y}{\sqrt{2}} = 2w+b$ goes through $[2\sqrt{2} y, +\infty)$. Therefore, $- 3 \frac{x}{\sqrt{6}} $ goes through $[3 \frac{y}{\sqrt{2}}, +\infty)$, which means $x$ goes through $(-\infty, -\sqrt{3} y ]$.
    \item If $y\leq 0$: $- 3 \frac{x}{\sqrt{6}} + \frac{y}{\sqrt{2}}=-(2w+b)$ goes through $( -\infty, -2\sqrt{2}  |y| ]$. Therefore, $- 3 \frac{x}{\sqrt{6}} $ goes through  $( -\infty, 3 \frac{y}{\sqrt{2} } ]$, which means $x$ goes through $[-\sqrt{3} y, +\infty)$.
\end{itemize}

Finally, the set $\cV_5$ is the set in between the two lines $x+\sqrt{3}y=0$ and $y=0$, as on the right of Figure \ref{dessin-exemple}.

\paragraph{Case $j=6$:}We have $b\leq 0$, $w+b \leq 0$ and $2w+b\geq 0$. This leads to $f_{(w,v,b,c)}(X) = (c,c,v(2w+b)+c)$ and
\[\cV_6 = \bigl\{(x,y) \in\RR^2~|~ \exists (w,v,b,c)\in \widetilde \cU^X_6, p(x,y) = (c,c,v(2w+b)+c)\bigr\}.
\]
We have
\[\left\{\begin{array}{rrl}
(1): & c & =\frac{x}{\sqrt{6}} - \frac{y}{\sqrt{2}} \\
(2): &c & = \frac{x}{\sqrt{6}} + \frac{y}{\sqrt{2}} \\
(3): &v(2w+b)+c & =  -2 \frac{x}{\sqrt{6}}
\end{array}\right. \qquad \Longleftrightarrow \qquad
\left\{\begin{array}{rrl}
(1): &\frac{x}{\sqrt{6}} - \frac{y}{\sqrt{2}} &= c \\
((2)-(1))/\sqrt{2}: &y &= 0\\
(3)-(1): &- 3 \frac{x}{\sqrt{6}} + \frac{y}{\sqrt{2}} &= v(2w+b)
\end{array}\right. .
\]
Using either $c$ or $v(2w+b)$, $x$ can take any value in $\RR$ and
\[\cV_6 = \bigl\{(x,y) \in\RR^2~|~ y=0\bigr\}.
\]

\subsection{Proof of \texorpdfstring{\Cref{regul-implicit-sec-1}}{corollary 4} }\label{proof-sec4-app-1}

Throughout the proof, we consider $\theta^* \in \cup_{j=1}^{p_X} \cU_j^X$. We denote $j^*\in\lb1,p_X\rb$ such that $\theta^*\in \cU_{j^*}^X$ and $k= \dim^-(\theta^*,X) = \RK{ \DT f_{\theta^*}(X)}$. We define, for all $\theta\in  \Par$, the function  $h(\theta) = \dim^-(\theta,X) - k$. Using \Cref{theorem:constant:rank},  we have $h(\theta)= 0$ for all $\theta\in\cU_{j^*}^X$. Since $\theta^*\in \cU_{j^*}^X$ and since, using \Cref{theorem:constant:rank}, $\cU_{j^*}^X$ is open, the function $h$ equals $0$ in $\cU_{j^*}^X$. It is also differentiable at $\theta^*$ and $\nabla h (\theta^*) = 0$.

Let us first prove that if $\theta^*$ is a critical point of $(P)$, then $(\theta^*,1)$ satisfies the KKT conditions of $(P_k)$. Assume that $\theta^*$ is a critical point of $(P)$. Denoting $\cL(\theta) = R(f_\theta(X))$,  we have $\nabla \cL(\theta^*) = 0$. Since $\nabla h (\theta^*) = 0$, we have
\[\nabla \cL(\theta^*) + \nabla h(\theta^*) = 0.
\]
Using $h(\theta^*) = 0$, we conclude that $(\theta^*,1)$ satisfies the KKT conditions of $(P_k)$.

Let us now prove that if $(\theta^*,1)$ satisfies the KKT conditions of $(P_k)$ then $\theta^*$ is a critical point of $(P)$.
Indeed, if the former holds, 
\[\nabla \cL(\theta^*) + \nabla h(\theta^*) = 0.
\]
Using $\nabla h(\theta^*) = 0$, we deduce that $\nabla \cL(\theta^*) = 0$ and conclude that $\theta^*$ is a critical point of $(P)$.

This concludes the proof.

\subsection{Proof of \texorpdfstring{\Cref{regul-implicit-sec-2}}{Corollary 5} }\label{proof-sec4-app-2}

We detail the proof for local minimizers but, because it uses similar arguments, we omit the proof of the statement for saddle points.

For simplicity, we denote for all $\theta\in\Par$, $\cL(\theta) = R(f_\theta(X))$. 

We first consider $\theta^*\in\Par$, a local minimizer of $(P)$, and prove that $\theta^*$ is a local minimizer of $(P_k)$, for $k=\dim^+(\theta^*,X)$. 
From the definition \eqref{dim_def},
we have $\dim^-(\theta^*,X) \le \dim^+(\theta^*,X) = k$. Also, the hypothesis on $\theta^*$ guarantees that there exists $\varepsilon >0$ such that for all $\theta\in B(\theta^*,\varepsilon)$, $\cL(\theta^*)\leq \cL(\theta)$. A fortiori, for all $\theta\in B(\theta^*,\varepsilon)$ such that $\dim^-(\theta,X) \le k$, $\cL(\theta^*)\leq \cL(\theta)$, and since $\dim^-(\theta^*,X) \le k $, then $\theta^*$ is a local minimizer of $(P_k)$.

Let us now prove the converse statement.
Let $\theta^*\in\Par$ be a local minimizer of $(P_k)$ for $k=\dim^+(\theta^*,X)$.
There exists $\varepsilon>0$ such that for all $\theta \in B(\theta^*,\varepsilon)$ satisfying $\dim^-(\theta,X) \leq k$, we have $\cL(\theta^*) \leq \cL(\theta)$.
Using that the $\rk$ takes integer values and the definition of $\dim^+(\theta^*,X)$ in \eqref{dim_def}, there exist $\varepsilon'>0$  and $j^*\in\lb1,p_X\rb$ such that $ B(\theta^*,\varepsilon') \cap \cU_{j^*}^X \neq \emptyset$, $k=\dim^+(\theta^*,X) = r_{j^*}^X$, and $r_{j}^X\leq k$ for all $j$ such that $ B(\theta^*,\varepsilon') \cap \cU_{j}^X \neq \emptyset$. Using the definitions of $\dim^-(\theta^*,X)$, we know that for all $\theta \in B(\theta^*,\varepsilon')$ we have $\dim^-(\theta,X) \leq r_{j^*}^X = k$.

Denoting $\epsilon=\min(\varepsilon, \varepsilon')$, it follows that for all $\theta \in B(\theta^*,\epsilon)$, we have  $\dim^-(\theta,X) \leq  k$ and thus $\cL(\theta^*)\leq \cL(\theta)$. As a consequence, $\theta^*$ is a local minimizer of $(P)$. 

This concludes the proof.

\section{Proof of Theorem \ref{shallow-case-thm}}\label{shallow-case-app}

The proof of Theorem \ref{shallow-case-thm} is decomposed into the detailed study of the architecture $(N_0,N_1,N_2) = (1, 1, 1)$, in \Cref{archi-OneOneOne}, and the proof in the general case, in \Cref{archi-ONENN}. Notice that the results of  \Cref{archi-OneOneOne} extend the results described in the example in \Cref{example:image:preimage:dimdeux}.

\subsection{Architecture \texorpdfstring{$(N_0,N_1,N_2) = (1, 1, 1)$}{(N0,N1,N2) = (1, 1, 1)}}\label{archi-OneOneOne}

 In this section, we investigate neural network functions with the architecture $(1, 1, 1)$. For simplicity, we assume  throughout the section that $X = (x^{(1)}, \dots, x^{(n)}) \in\RR^{1\times n}$ is such that 
 \begin{equation}\label{x_ordered}
     x^{(1)}< x^{(2)} < \dots < x^{(n)}.
 \end{equation}
 We also simplify notations and consider the neural network function applied to the sample \(X\) defined by 
\begin{equation}\label{arch-1-1-1}
	f_{(w,v,b,c)}(X) = v\sigma (wX + b\mathbf{1}) + c\mathbf{1} \in \mathbb{R}^{1 \times n}, \quad \forall w,b,v,c \in \mathbb{R},
\end{equation}
where all the components of $\One\in\RR^{1 \times n}$ equal $1$. We also adapt the notation given in Section \ref{activation-patterns-sec} and consider the activation pattern $a(X,w,b)\in\RR^{1 \times n}$ defined by
\[
	a(X, w, b)_{1,j} = 
			 \begin{cases}
			1 & \text{if } w x^{(j)} + b \geq 0,\\
			0 & \text{otherwise},
			\end{cases} \quad \forall j=1, \dots, n.
\]
We have,
\begin{equation}\label{property-activation}
    \sigma(w X + b\mathbf{1}) = a(X,w,b) \odot (w X + b \mathbf{1}),
\end{equation}
where $\odot$ stands for the Hadamard product.

Let us introduce the vectors \(\mathbf{1}_{1}, \dots, \mathbf{1}_{2n} \in \mathbb{R}^{1 \times n}\) by defining
\begin{equation}\label{eqn:1i-vectors}
\begin{aligned}	\mathbf{1}_{i} &= a\Big(X, 1, -x^{(i)}\Big), & & \text{for } i=1, \dots, n,\\
	\mathbf{1}_{n+i} &= a\Big(X, -1, x^{(i-1)}\Big), & & \text{for } i=2, \dots, n,
\end{aligned}
\end{equation}
and \(\mathbf{1}_{n+1} = \mathbf{0}\), that is, a vector of zeroes. As in \eqref{defDesOnes}, we have, for \(i=1, \dots, n\),
\begin{equation}
\label{eqn:1-order}
\mathbf{1}_{i} = (0,\dots, 0, \underset{\underset{\text{\(i\)}}{\uparrow}}{1}, \dots, 1), \quad 
	\mathbf{1}_{n+i} = (1,\dots, 1,  \underset{\underset{\text{\(i\)}}{\uparrow}}{0}, \dots, 0).
\end{equation} 
Given these notations, we consider the sets
\begin{equation}\label{defn:activations}
    \overline\cU^X_{i} = \{(w,b) \in \mathbb{R}^{2} \mid a(X, w, b) = \mathbf{1}_{i}\}, \quad \text{for }i=1,\dots, 2n.
\end{equation}
The use of the \lq overline\rq{} shall not be confused with the closure. As will be clarified in the sequel, the sets $\overline\cU^X_{i}$ can be closed, open, or neither. Considering the definition of $\widetilde\cU^X_{j}$ in \eqref{def:u:j:tilde}, for all $j\in\lb1,p_X\rb$ there exists $i\in\lb1,2n\rb$  such that, modulo a change of order of the components\footnote{Because it simplifies notations and is harmless, we will make this abuse of notation throughout the section.}, $\widetilde\cU^X_{j} = \Inter{\left(\overline\cU^X_{i}\times\RR^2\right)}$.

The following lemma shows that the sets $\overline\cU^X_{i}$ constitute a partition of \(\mathbb{R}^{2}\) into constant components for the activation function with a fixed sample \(X\). Moreover, we give a geometric parameterization of these regions as the cone generated by segments between some chosen parameter pairs. 

The parameterization uses the notation \(({y}, {z}]\) which is defined as \(\{(1-t) {y} + t{z} \mid 0 < t \leq 1\}\) and represents the open-closed line segment between vectors \({y}\) and \({z}\), with similar interpretations for other combinations of brackets. Also, for a subset $\mathcal{V}$ of a vector space, we define $\mathbb{R}_{>0} \mathcal{V}$ as the set $\{\lambda v \mid v \in \mathcal{V}, \lambda > 0\}$. Similarly, $\mathbb{R}_{\geq 0} \mathcal{V}=\{\lambda v \mid v \in \mathcal{V}, \lambda \geq 0\}$. A subset $\mathcal{V}$ of a vector space is recognized as a positive cone if it satisfies $\mathbb{R}_{>0}\mathcal{V} \subseteq \mathcal{V}$.

An illustration of \Cref{lem:sets} is in \Cref{ima:regions}.
\begin{lem}\label{lem:sets}
Assume that \(n \geq 2\) and the sample \(X\in\RR^{1 \times n}\) satisfies \eqref{x_ordered}.
Then, the activation regions \(\overline\cU^X_{1}, \dots, \overline\cU^X_{2n}\) are a partition of \(\mathbb{R}^{2}\) into  convex positive cones; precisely, each region is characterized by 
\begin{equation}\label{eqn:Ui-vectors}
\left\{
\begin{aligned}
    \overline\cU^X_{1} &= \mathbb{R}_{\geq 0}\left[(-1, x^{(n)}), (1, -x^{(1)})\right], \\
    \overline\cU^X_{i} &= \mathbb{R}_{>0} \left((1,-x^{(i-1)}),(1,-x^{(i)})\right], & \quad &\text{for } i=2, \dots, n, \\
    \overline\cU^X_{n+1} &= \mathbb{R}_{>0} \left((1,-x^{(n)}),(-1,x^{(1)})\right),\\
    \overline\cU^X_{n+i} &= \mathbb{R}_{>0} \left[(-1,x^{(i-1)}),(-1,x^{(i)})\right), & \quad &\text{for } i=2, \dots, n.
\end{aligned}
\right.
\end{equation}
\end{lem}

\begin{figure}[ht]
\centering
\begin{tikzpicture}[>=Stealth]
    % Highlighted sectors
    \fill[red!20] (0,0) -- (-1.5,3) -- (2,3) -- cycle;
    \fill[green!20] (0,0) -- (2,3) -- (4,3) -- (4,-3) -- (3.5,-3) -- cycle;
    \fill[blue!20] (0,0) -- (3.5,-3) -- (1.5,-3)-- cycle;
    \fill[gray!20] (0,0) -- (1.5,-3) -- (-2,-3)-- cycle;
    \fill[blue!20] (0,0) -- (-2,-3) -- (-4,-3) -- (-4,-2.5) -- cycle;
    \fill[green!20] (0,0) -- (-4,-2.5) -- (-4,3) -- (-1.5,3) -- cycle;

    % Coordinate axes
    \draw[->] (-4,0) -- (4,0) node[right] {\(w\)};
    \draw[->] (0,-3) -- (0,3) node[above] {\(b\)};

    % Thresholds and sectors
    \draw[red, thick] (0,0) -- (2,3);
    \draw[blue, thick] (0,0) -- (-2,-3);
    \draw[green!80!black, thick] (0,0) -- (4,2.5);
    \draw[green!80!black, thick] (0,0) -- (-4,-2.5);
    \draw[green!80!black, thick] (0,0) -- (-3.5,3);
    \draw[green!80!black, thick] (0,0) -- (3.5,-3);
    \draw[red, thick] (0,0) -- (-1.5,3);
    \draw[blue, thick] (0,0) -- (1.5,-3);
    \fill[red] (0,0) circle (2pt);

    % Lines and points
    \draw[black, <-] (1,3) -- (1,-3) node[below] {\(w=1\)};
    \draw[black, <-] (-1,3) -- (-1,-3) node[below] {\(w=-1\)};

    % Points
    \fill[black] (1,1.5) circle (2pt) node[right] {\(-x^{(1)}\)};
    \fill[black] (1,0.62) circle (2pt) node[right] {\(-x^{(2)}\)};
    \fill[black] (1,-0.85) circle (2pt) node[right] {\(-x^{(n-1)}\)};
    \fill[black] (1,-2.02) circle (2pt) node[right] {\(-x^{(n)}\)};
    \fill[black] (-1,-1.5) circle (2pt) node[left] {\(x^{(1)}\)};
    \fill[black] (-1,-0.62) circle (2pt) node[left] {\(x^{(2)}\)};
    \fill[black] (-1,0.85) circle (2pt) node[left] {\(x^{(n-1)}\)};
    \fill[black] (-1,2.02) circle (2pt) node[left] {\(x^{(n)}\)};

    % Labels for sectors
    \node at (-1.8,2.5) {\(\overline\cU^X_{2n}\)};
    \node at (2.5,2.5) {\(\overline\cU^X_{2}\)};
    \node at (0.5,2.5) {\(\overline\cU^X_{1}\)};
    \node at (2.1,-2.5) {\(\overline\cU^X_{n}\)};
    \node at (-0.5,-2.5) {\(\overline\cU^X_{n+1}\)};
    \node at (-2.7, -2.5) {\(\overline\cU^X_{n+2}\)};

    % Dots for continuity
    \foreach \y in {0.1,-0.1,-0.3, -0.5}{
      \fill[black] (1.5,\y) circle (1pt);
    }
    \foreach \y in {0.5,0.3,0.1,-0.1}{
      \fill[black] (-1.5,\y) circle (1pt);
    }
\end{tikzpicture}
\caption{\label{ima:regions}Illustration of the activation regions for a neural network of architecture $(1,1,1)$ with sample \(X = (x^{(1)}, \dots, x^{(n)})\) satisfying \eqref{x_ordered}. The coloring of the activation regions corresponds to different local behaviors of the neural network function described in \Cref{lem:sets}.}
\end{figure}

\begin{proof}
We start by demonstrating that the sets on the right-hand side of \eqref{eqn:Ui-vectors} are subsets of their respective \(\overline\cU^X_{i}\). 

For $i=1$, parameters  $(w,b)\in\RR_{\geq0}\left[(-1, x^{(n)}), (1, -x^{(1)})\right]$ can be expressed as 
\[
	(w,b)=\lambda \left((1-t)\left(-1, x^{(n)}\right) + t\left(1, -x^{(1)}\right) \right)\quad \text{for $t\in [0,1]$ and $\lambda \geq 0$.}
\] 
The preactivation hidden layer's content for \(X\) with these parameters is
\begin{align*}
	w X + b \vb 1
	&= \big(\lambda(1-t)(-1) + \lambda t\big)X +  \big(\lambda(1-t)x^{(n)} + \lambda t(-x^{(1)})\big) \vb 1\\
	&= \lambda\left((1-t)\left( x^{(n)} \vb 1 - \vphantom{x^{(1)}}X\right) + t\left(X - x^{(1)}\vb 1\right)\right),
	\end{align*}
which has non-negative components, since \(x^{(n)} \vb 1 - X\) and $X - x^{(1)}\vb 1$ are nonnegative. Therefore, the activation pattern for these parameters is \(\mathbf{1} = \mathbf{1}_{1}\), implying that \((w,b) \in \overline\cU^X_{1}\), and therefore $\RR_{\geq0}\left[(-1, x^{n}), (1, -x^{1})\right]\subset \overline\cU^X_{1}$.

Similarly, for \(i = 2, \dots, n\), the preactivation hidden layer's content for \(X\) with parameters $(w,b)$ in $\mathbb{R}_{>0}\big((1,-x^{(i-1)}),(1,-x^{(i)})\big]$ is 
\[
 wX +b\vb 1=\lambda\left((1-t)\left( X - x^{(i-1)} \vb 1\right) + t\left(X - x^{(i)}\vb 1\right)\right), \quad \text{for $t\in (0,1]$ and $\lambda>0$.}
\]
Exactly the first \(i-1\) components of these vectors are negative. This arises because, for all $t\in (0,1]$, the term $(1-t)(x^{(j)}-x^{(i-1)})+t(x^{(j)}-x^{(i)})$ yields a negative value for $j\leq i-1$ and is nonnegative for $j\geq i$. With the help of expression \eqref{eqn:1-order}, we recognise that the activation pattern for these parameters is \(\mathbf{1}_{i}\) and therefore $\mathbb{R}_{>0}\big((1,-x^{(i-1)}),(1,-x^{(i)})\big]\subset \overline\cU^X_{i}$, by the definition of $\overline\cU^X_{i}$.

For \(i = n+1\), the preactivation hidden layer's content for \(X\) with parameters  $(w,b)$ in $\mathbb{R}_{>0}\big((1,-x^{(n)}),(-1,x^{(1)})\big)$ is
\[
	wX +b\vb 1= \lambda\left((1-t)\left(X - x^{(n)}\vb 1\right)+ t\left( x^{(1)} \vb 1 - X\right) \right)\quad \text{for $t\in (0,1)$ and $\lambda>0$.}
\]
These vectors have negative components. This is because the only non-negative components in   \(x^{(1)} \vb 1 - X\) and $X - x^{(n)}\vb 1$ are, respectively, the first and the last which are zero, with all other components being negative. Therefore, their strict convex combination also yields negative values. It follows that the activation pattern is \(\mathbf{1}_{n+1} = \mathbf{0}\) and that $\mathbb{R}_{>0}\big((1,-x^{(n)}),(-1,x^{(1)})\big)\subset \overline\cU^X_{n+1}$, by definition of $\overline\cU^X_{n+1}$.

Finally, for \(i = 2, \dots, n\), the preactivation hidden layer's content for \(X\) with parameters $(w,b)$ in \(\mathbb{R}_{>0}\big[(-1,x^{(i-1)}),(-1,x^{(i)})\big)\) is
\[
	wX +b\vb 1= \lambda\left((1-t)\left( x^{(i-1)}\vb 1- X\right)+ t\left( x^{(i)} \vb 1 - X \right) \right) \quad \text{for $t\in [0,1)$ and $\lambda>0$.}
\]
Exactly the components \(j = i, \dots, n\) of these vectors are negative. This is because, for all $t\in [0,1)$, the term $(1-t)(x^{(i-1)}-x^{(j)})+t(x^{(i)}-x^{(j)})$ yields a negative value for $j\geq i$ and is nonnegative for $j\leq i-1$. It follows that the activation pattern is \(\mathbf{1}_{n+i}\) and that $\mathbb{R}_{>0}\big[(-1,x^{(i-1)}),(-1,x^{(i)})\big) \subset \overline\cU^X_{n+i}$, by definition of $\overline\cU^X_{n+i}$. 

 To establish the inclusion  of the activation regions $\overline\cU^X_{j}$ in their respective sets  in \eqref{eqn:Ui-vectors}, since,  by definition, the activation regions are disjoint, it is sufficient  to demonstrate that the subsets on the right-hand side of \eqref{eqn:Ui-vectors} cover the entire $\RR^{2}$. This will also ensure that the activation regions partition $\RR^{2}$.
 
To proceed with this, let us consider any point \((w,b)\) in \(\mathbb{R}^{2}\). 

 If $w=0$, since $x^{(n)} - x^{(1)} >0$, $(w,b)$ belongs either to $\RR_{\geq0}\left[(-1, x^{(n)}), (1, -x^{(1)})\right]$, if $b\geq 0$, or to $\mathbb{R}_{>0}\big((1,-x^{(n)}),(-1,x^{(1)})\big)$, if $b<0$.
 
If $w>0$, several cases arise:
\begin{itemize} 
\item If $-x^{(1)}\leq b/w$, then we decompose $(w,b)=w(1,-x^{(1)}) + w(0, b/w + x^{(1)})$, which belongs to $\RR_{\geq0}\left[(-1, x^{(n)}), (1, -x^{(1)})\right]$ since the latter is a convex cone to which both $(1,-x^{(1)})$ and $(0, b/w + x^{(1)})$ belong.
\item If $-x^{(i)}\leq b/w <-x^{(i-1)}$ with $i=2, \dots, n$, then we recognize that $(w,b) = w(1,b/w)\in \mathbb{R}_{>0}\big((1,-x^{(i-1)}),(1,-x^{(i)})\big]$ because $b/w\in [-x^{(i)}, -x^{(i-1)})$ and $w>0$.
\item If $b/w< -x^{(n)}$, then we decompose  $(w,b)= w(1,-x^{(n)}) + w(0, b/w + x^{(n)})$, which belongs to $\mathbb{R}_{>0}\big((1,-x^{(n)}),(-1,x^{(1)})\big)$ since the latter is a convex cone, to which $(0, b/w + x^{(n)})$ belongs, $(1,-x^{(n)})$ is in its closure, and $w>0$.
 \end{itemize}
 Similarly, if $w<0$, several cases arise:
\begin{itemize} 
\item If $-b/w< x^{(1)}$, then we decompose $(w,b) = -w(-1,x^{(1)}) + -w(0, -b/w - x^{(1)})$ belongs to $\mathbb{R}_{>0}\big((1,-x^{(n)}),(-1,x^{(1)})\big)$ since the latter is a convex cone, to which $(0, -b/w - x^{(1)})$ belongs, $(-1,x^{(1)})$ is in its closure, and $-w>0$.
\item If $x^{(i-1)}\leq -b/w <x^{(i)}$ with $i=2, \dots, n$, then we recognize that $(w,b) = -w(-1,-b/w)\in \mathbb{R}_{>0}\big[(-1,x^{(i-1)}),(-1,x^{(i)})\big)$ because $-b/w\in [x^{(i-1)}, x^{(i)})$.
\item If $x^{(n)}\leq -b/w$, then we decompose $(w,b)= -w(-1,x^{(n)}) -w(0, -b/w - x^{(n)})$, which belongs to $\RR_{\geq0}\left[(-1, x^{(n)}), (1, -x^{(1)})\right]$ since the latter is convex cone to which both $(-1,x^{(n)})$ and (since $-b/w - x^{(n)}\geq 0$) $(0, -b/w - x^{(n)})$ belong, and $-w> 0$.
 \end{itemize}
 This concludes the proof.
 \end{proof}

Since we have shown that activation regions as defined in  \eqref{defn:activations} partition $\RR^{2}$, one can observe that, for the architecture $(1,1,1)$, when $X$ satisfies \eqref{x_ordered} and $(w,b)$ varies, the only achievable activation patterns are the row vectors $\vb 1_{1}, \dots, \vb 1_{2n}$ defined in \eqref{eqn:1i-vectors}, or equivalently \eqref{eqn:1-order}.

The next lemma provides a simple parameterization of the sets $\mathcal V_{i}$, defined for all $i=1,\dots,2n$, by
\begin{equation}\label{def_Vi}
    \mathcal V_{i}= \{\sigma(wX+b) \in \RR^{1\times n} \mid (w,b) \in \overline\cU^X_{i}\}.
\end{equation}
The union of these sets constitutes the image of $X$ in the hidden layer.
To parameterize the sets, we define, for each $i=1, \dots, n$, the vectors $\mathbf{e}_{i}$ and $\mathbf{e}_{n+i} \in \RR^{1\times n}$ by
\begin{equation}\label{def_ei_proof}
    	\mathbf{e}_{i} = \sigma\big( X -x^{(i)} \vb 1\big)
	\quad \text{and} \quad
	\mathbf{e}_{n+i} =  \sigma\big( x^{(i)}\vb 1 - X\big).
\end{equation}
We also set $\mathbf{e}_{0} = \mathbf{e}_{2n}$. These vectors correspond to the vectors defined in \eqref{def_bfe_0}. Since the sample \(X=(x^{(1)},\ldots,x^{(n)})\in\RR^{1 \times n}\) satisfies \eqref{x_ordered}, the vectors $\vb e_{i}$ are such that, for all $i \in \lb 1, n\rb$,
\begin{multline}\label{qrounetivb}
%   \left\{\begin{array}{l}
   \mathbf{e}_{i} = (0,\dots, \underset{\underset{\text{\(i\)}}{\uparrow}}{0},  x^{(i+1)} - x^{(i)}, \dots, x^{(n)}-x^{(i)}), \\ \mbox{and} \quad
	\mathbf{e}_{n+i} = (x^{(i)}-x^{(1)}, \dots, x^{(i)}-x^{(i-1)},  \underset{\underset{\text{\(i\)}}{\uparrow}}{0}, \dots, 0).
%   \end{array}\right.
\end{multline}

In particular, $\mathbf{e}_{n} =\mathbf{e}_{n+1} =0$.

The following lemma is illustrated in Figure \ref{ima:images}.
\begin{lem}\label{lem:images}
Assume that \(n \geq 2\) and the sample \(X\in\RR^{1 \times n}\) satisfies \eqref{x_ordered}. We have
 \begin{equation}\label{eqn:images}
 \mathcal V_{i} = \left\{\begin{array}{ll}
\RR_{\geq 0} \left[\vb e_{i-1}, ~ \vb e_{i}\right] & \text{, for }i=1 \\
\RR_{>0} \left(\vb e_{i-1} ,~ \vb e_{i}\right] & \text{, for }i=2,\dots,n,\\
\RR_{>0} \left(\vb e_{i-1} ,~ \vb e_{i}\right) & \text{, for }i=n+1\\
\RR_{>0} \left[  \vb e_{i-1},~  \vb e_{i}\right) & \text{, for } i=n+2,\dots,2n.
 \end{array} \right.
 \end{equation}
 \end{lem}

\begin{figure}[ht]
\begin{center}
  \begin{tikzpicture}[>=Stealth]
% Define the coordinates of the vertices
\coordinate (Origin) at (0, 0, 0);
\coordinate (X1) at (5, 0, 0);
\coordinate (X2) at (0, 0, -5);
\coordinate (X3) at (0, 5, 0);
\coordinate (P0) at (0, 0, 0);
\coordinate (P1) at (5, 0, 0);
\coordinate (P2) at (5, 0,3.125);
\coordinate (P3) at (0, 5, 1.875);
\coordinate (P4) at (0, 5, 0);

% Draw the coordinates axes
\draw[->] (0,0,-1) -- (0, 0, 4) node[left, black] {$x^{2}$};
\draw[->] (-1,0,0) -- (5, 0, 0) node[above right, black] {$x^{1}$};
\draw[->] (0,-0.5,0) -- (0, 5, 0) node[right, black] {$x^{3}$};

% Draw over the axes
\draw[->] (Origin) -- (X1) node[anchor=north west]{$\RR_{>0}\vb e_{n+2}  $} ;% = \RR_{>0} \vb 1_{n+2}$} ;
\draw[->] (Origin) -- (X3) node[anchor=east]{$\RR_{>0}\vb e_{n-1} $};%= \RR_{>0} \vb 1_{n}$};

%% To see the range of the reds
\draw[dashed] (0,4,0) -- (0, 4, 1.5);
\draw[dashed] (0,0,1.5) -- (0, 4, 1.5);
\draw[dashed] (4,0,0) -- (4, 0,2.5);
\draw[dashed] (0,0,2.5) -- (4, 0,2.5);

% Draw the polygon
\draw[fill=green!20, opacity=0.5, draw=none] (Origin) -- (P1) -- (P2) -- cycle;
\draw[fill=green!20, opacity=0.5, draw=none] (Origin) -- (P3) -- (P4) -- cycle;
\draw[fill=red!20, opacity=0.5, draw=none] (Origin) -- (P2) -- (P3) -- cycle;

\fill[gray]  (0,0,0)   circle (3pt)  node[above right, black] {$\vb e_{n} = \vb e_{n+1}=\vb 0$};

% Draw the lines on the regions
\draw[thick, red] (Origin) -- (P2) ;
\draw[thick, red] (Origin) -- (P3);
\draw[thick, red,->] (Origin) -- (2.5,2.5,2.5) node[right, black]{$\vb 1$};
\draw[very thick, blue] (Origin) -- (P1);
\draw[very thick, blue] (Origin) -- (P4);
\draw[green!80!black] (Origin) -- (P4);
\draw[green!80!black] (Origin) -- (P1);

% draw the coordinates indications
\fill[black]  (P2)  circle (0pt) node[below right] {$\RR_{>0}\vb e_{2n}$};
\fill[black]  (P3)  circle (0pt) node[below left] {$\RR_{>0}\vb e_{1}$};
\end{tikzpicture}
\caption{\label{ima:images}Illustration of $\{\sigma(wX + b\vb 1)\mid (w,b)\in \RR^{2}\}\subset \RR^{1\times n}$ for $n=3$ and sample set $X$ satisfying \eqref{x_ordered}.
The colours in this figure correspond to those in Figure \ref{ima:regions} for the partition of $\RR^{2}$ in activation regions $\overline\cU^X_{1}, \dots, \overline\cU^X_{2n}$.
}

\end{center}
\end{figure}

 \begin{proof}
 Due to \eqref{property-activation} and the definition of $\overline\cU^X_{i}$, in \eqref{defn:activations}, we have for all $ i=1, \dots ,2n$,
\begin{equation}\label{images-vi}
	\sigma(w X + b\vb 1) = a(X, w, b)\odot (w X +b \vb 1)= \vb 1_{i} \odot (wX+ b\vb 1), \quad \forall (w,b) \in \overline\cU^X_{i},
\end{equation}
which is linear in $(w,b)$.

Adapting the following arguments to other values of $i=1,\dots,2n$, and therefore other brackets and inequality signs, will lead to an analogue of \eqref{res_intermediaire} for all $i=1,\dots,2n$. For simplicity, we only detail the proof of \eqref{res_intermediaire} for an arbitrary $i=2,\dots,n$.
Considering \eqref{eqn:Ui-vectors}, we have \(\overline\cU_i^X = \mathbb{R}_{>0}(y, z]\), with $y=(y_1,y_2)=(1,-x^{(i-1)})$ and $z=(z_1,z_2)=(1,-x^{(i)})$. Using also the linearity obtained from \eqref{images-vi}, we obtain
\begin{align*}
	& \{\sigma(wX+b\One) \in \RR^{1\times n} \mid (w,b) \in \overline\cU^X_{i}\} \\
		&= \left\{\sigma\Big(\lambda \big((1-t)y_1+tz_1\big)~ X + \lambda \big((1-t)y_2+tz_2\big )~\One\Big) \in \RR^{1\times n} \mid \lambda>0 \mbox{ and }t\in(0,1]\right\} \\
		&= \left\{\lambda (1-t)~\sigma \big( y_1X+y_2 \One\big) + \lambda t ~\sigma \big(z_1X +z_2\One\big) \in \RR^{1\times n} \mid \lambda>0 \mbox{ and }t\in(0,1]\right\}.
\end{align*}
Using \eqref{def_ei_proof}, we find that $\sigma \big( y_1X+y_2 \One\big) =\mathbf{e}_{i-1} $ and $\sigma \big(z_1X +z_2\One\big) = \mathbf{e}_{i}$, which leads to 
\begin{equation}\label{res_intermediaire}
    \{\sigma(wX+b\One) \in \RR^{1\times n} \mid (w,b) \in \overline\cU^X_{i}\} = 
\RR_{>0}(\mathbf{e}_{i-1},\mathbf{e}_{i} ].
\end{equation}
As already said, adaptations to sets of the form $\mathbb{R}_{>0}[y, z)$ and $\mathbb{R}_{\geq0}[y, z]$ are straightforward. Using \eqref{def_ei_proof}, we find that $\sigma(y_1X+y_2\One) =\mathbf{e}_{i} $, when $(y_1,y_2) = (1,-x^{(i)})$, and $\sigma(y_1X+y_2\One) =\mathbf{e}_{n+i} $, when $(y_1,y_2) = (-1,x^{(i)})$.

This concludes the proof.
\end{proof}
Notice that, since $\mathbf{e}_{n}=\mathbf{e}_{n+1}=0$, we have
\[
 \mathcal V_{n} = \RR_{\geq 0}  \vb e_{n-1}, \quad  \mathcal V_{n+1} = \{0\}, \qquad\text{ and }\quad  \mathcal V_{n+2} = \RR_{\geq 0}  \vb e_{n+2}.
\]

The following proposition is not required for the proof of Theorem \ref{shallow-case-thm}. However, we present it as an illustration—a generalization of the example described in Section \ref{example:image:preimage:dimdeux}. A visual representation of the proposition is provided in Figure \ref{ima:output}.

In the proposition, we denote for all $i=1,\dots,2n$,
\[
f_{\overline\cU^X_i \times \RR^2} = \big\{ v\sigma(wX+b\One)+c\One \in\RR^{1\times n} \mid (w,b,v,c)\in \overline\cU^X_i \times \RR^2 \big\}.
\]
\begin{prop}[Architecture $ (N_0,N_1,N_2)= (1,1,1)$]\label{prop:1-1-1} Assume that \(n \geq 2\) and the sample \(X\in\RR^{1 \times n}\) satisfies \eqref{x_ordered}. We have, for all $i=1,\dots,2n$,
\begin{eqnarray*}
    f_{\overline\cU^X_i \times \RR^2} & = & \RR \mathcal V_{i} + \RR \vb 1 \\
& = &\left\{\begin{array}{ll}
\RR \left[\vb e_{i-1}, ~ \vb e_{i}\right] + \RR \vb 1 & \text{, for }i=1 \\
\RR \left(\vb e_{i-1} ,~ \vb e_{i}\right] + \RR \vb 1& \text{, for }i=2,\dots,n,\\
\RR \left(\vb e_{i-1} ,~ \vb e_{i}\right) + \RR \vb 1& \text{, for }i=n+1\\
\RR \left[  \vb e_{i-1},~  \vb e_{i}\right)+ \RR \vb 1 & \text{, for } i=n+2,\dots,2n.
 \end{array} \right.
\end{eqnarray*}
\end{prop}

\tdplotsetmaincoords{80}{120}%{55}{135}

\begin{figure}[ht]
\begin{center}
  \begin{tikzpicture}[>=Stealth, tdplot_main_coords]
  
% Define the length of the axes
\def\axisscale{5}

% Draw the main coordinate system axes
\draw[] (-7,0,0) -- (0,0,0);% node[anchor=north east]{$x^{1}$};
\draw[->] (0,-\axisscale,0) -- (0,\axisscale,0) node[anchor=north west]{$x^{2}$};
\draw[->] (0,0,-3) -- (0,0,\axisscale) node[anchor=south]{$x^{3}$};

% Radius of the gray circle
\def\radius{2}

% Set rotated coordinates for circle plane orthogonal to vector (1,1,1)
\tdplotsetrotatedcoords{45}{54.73561}{0}

%\draw[very thick, blue] (-3.8165, -3.8165, -1.36701,)  -- (2.1835, 2.1835, 4.63299); %node[below, black] {$f_{\widetilde\cU^X_{n}\times \RR^{2}}(X)$};
%\draw[very thick, red] (-3.18732, -4.31122, -1.50146)  -- (2.81268, 1.68878, 4.49854); %node[below, black] {$f_{\widetilde\cU^X_{n}\times \RR^{2}}(X)$};

\draw[tdplot_rotated_coords,fill=blue!20,fill opacity=0.5, draw=none] (0,0,5.2) -- ({\radius*cos(180)},{\radius*sin(180)},5.2) --({\radius*cos(180)},{\radius*sin(180)},-5.2) -- (0,0,-5.2) -- cycle;

\draw[tdplot_rotated_coords,fill=blue!20,fill opacity=0.5, draw=none] (0,0,5.2) -- ({\radius*cos(60)},{\radius*sin(60)},5.2) --({\radius*cos(60)},{\radius*sin(60)},-5.2) -- (0,0,-5.2) -- cycle;

\draw[tdplot_rotated_coords,fill=green!45,fill opacity=0.5, draw=green]  ({\radius*cos(180)},{\radius*sin(180)},5.2) arc (180:240:\radius)  -- ({\radius*cos(240)},{\radius*sin(240)},-5.2) arc (240:180:\radius) --  cycle ;

% Draw the red fill
%\draw[tdplot_rotated_coords,fill=green!40,fill opacity=0.5] (0,0,-5.2) -- ({\radius*cos(0)},{\radius*sin(0)},-5.2) arc (0:60:\radius) -- cycle;
\draw[tdplot_rotated_coords,fill=green!40,fill opacity=0.5, draw=none] (0,0,-5.2) -- ({\radius*cos(180)},{\radius*sin(180)},-5.2) arc (180:240:\radius) -- cycle;
\draw[fill=red!20,fill opacity=0.5, draw=none] (-3.18732, -4.31122, -1.50146)  -- (2.81268, 1.68878, 4.49854) -- (3.18732, 4.31122, 1.50146)  --  (-2.81268, -1.68878, -4.49854)  -- cycle;

% Draw the green fill
\draw[tdplot_rotated_coords,fill=green!40,fill opacity=0.5, draw=green] (0,0,5.2) -- ({\radius*cos(0)},{\radius*sin(0)},5.2) arc (0:60:\radius) -- cycle;
\draw[tdplot_rotated_coords,fill=green!40,fill opacity=0.5, draw=green] (0,0,5.2) -- ({\radius*cos(180)},{\radius*sin(180)},5.2) arc (180:240:\radius) -- cycle;

\draw[tdplot_rotated_coords,fill=green!45,fill opacity=0.5, draw=green]  ({\radius*cos(0)},{\radius*sin(0)},5.2) arc (0:60:\radius)  -- ({\radius*cos(60)},{\radius*sin(60)},-5.2) arc (60:0:\radius) --  cycle ;

% Draw the blue fill
\draw[tdplot_rotated_coords,fill=blue!15,fill opacity=0.5, draw=none] (0,0,5.2) -- (0,0,-5.2)-- ({\radius*cos(0)},{\radius*sin(0)},-5.2) -- ({\radius*cos(0)},{\radius*sin(0)},5.2) --cycle;

\draw[tdplot_rotated_coords,fill=blue!20,fill opacity=0.5, draw=none] (0,0,5.2) -- ({\radius*cos(240)},{\radius*sin(240)},5.2) --({\radius*cos(240)},{\radius*sin(240)},-5.2) -- (0,0,-5.2) -- cycle;

\draw[very thick, blue] (-3.8165, -1.36701, -3.8165)  -- (2.1835, 4.63299, 2.1835); %node[below, black] {$f_{\widetilde\cU^X_{n}\times \RR^{2}}(X)$};
\draw[very thick, blue] (-2.1835, -4.63299, -2.1835)  -- ( 3.8165, 1.36701, 3.8165); %node[below, black] {$f_{\widetilde\cU^X_{n}\times \RR^{2}}(X)$};
\draw[very thick, blue] (2.1835, 4.63299, 2.1835)  -- ( 3.8165, 1.36701, 3.8165); %node[below, black] {$f_{\widetilde\cU^X_{n}\times \RR^{2}}(X)$};

\draw[very thick, blue] (-2.1835, -2.1835, -4.63299)  -- ( 3.8165, 3.8165, 1.36701); %node[below, black] {$f_{\widetilde\cU^X_{n}\times \RR^{2}}(X)$};
\draw[very thick, blue] (2.1835, 2.1835, 4.63299)  -- ( 3.8165, 3.8165, 1.36701); %node[below, black] {$f_{\widetilde\cU^X_{n}\times \RR^{2}}(X)$};

%\draw[very thick, blue] (-2.1835, -4.63299, -2.1835) -- (-3,-3,-3) -- (-2.1835, -2.1835, -4.63299); %node[below, black] {$f_{\widetilde\cU^X_{n}\times \RR^{2}}(X)$};

\draw[very thick, red] (3.18732, 4.31122, 1.50146)  -- (2.81268, 1.68878, 4.49854); %node[below, black] {$f_{\widetilde\cU^X_{n}\times \RR^{2}}(X)$};

\draw[very thick, red] (-2.81268, -1.68878, -4.49854)  -- (3.18732, 4.31122, 1.50146); %node[below, black] {$f_{\widetilde\cU^X_{n}\times \RR^{2}}(X)$};

\draw[very thick, gray, ->] (-3,-3,-3) -- (4,4,4) node[anchor=south west, black]{$\vb 1$};

\draw[->] (-0,0,0) -- (7,0,0) node[anchor=north east]{$x^{1}$};

\draw[] (0,0,-4.5) -- (0,0,-2.5);% node[anchor=south]{$x^{3}$};

\draw[tdplot_rotated_coords, fill=white,fill opacity=0] (0,0,5.2) --  ({\radius*cos(60)},{\radius*sin(60)},5.2) arc (60:180:\radius) --  cycle ;
\draw[tdplot_rotated_coords, fill=white,fill opacity=0] (0,0,5.2) --  ({\radius*cos(240)},{\radius*sin(240)},5.2) arc (240:360:\radius) --  cycle ;

\draw[tdplot_rotated_coords, dashed, blue] (0,0,5.2) --  ({\radius*cos(300)},{\radius*sin(300)},5.2)  --  ({\radius*cos(120)},{\radius*sin(120)},5.2) ;

\end{tikzpicture}
\caption{\label{ima:output}Illustration of $\{f_{(w,v,b,c)}(X)\in\RR^{1\times n} \mid (w,v,b,c)\in\RR^4\}$ for $n=3$ and the sample $X = (2,0,5)$. 
Only a cylindrical section of the output is illustrated, with the vector $\vb 1$ as the axis of the cylinder and circular section in the plane orthogonal to the vector $\vb 1$.
The colors in this image correspond to those in Figure \ref{ima:regions} for the partition of $\RR^{2}$ in activation regions $\overline\cU^X_{1}, \dots, \overline\cU^X_{2n}$.
}

\end{center}
\end{figure}

\begin{proof} 
The proposition is a direct consequence of the definition 
\[
	f_{(w,v,b,c)}(X) = v \sigma(wX+b\One)+c\vb 1, \qquad \forall (w,v,b,c)\in \RR^{4}.
\] 
the definition of $\mathcal V_i$, in \eqref{def_Vi}, and  \Cref{lem:images}.
\end{proof}
In particular, using again that $\vb e_{n} =\vb e_{n+1} =0 $, we find
\[
f_{\overline\cU^X_{n}\times \RR^{2}}(X) = \RR \vb e_{n-1} \oplus \RR\vb 1 \quad, f_{\overline\cU^X_{n+1}\times \RR^{2}}(X) = \RR\One \quad \text{, and } \quad f_{\overline\cU^X_{n+2}\times \RR^{2}}(X) = \RR \vb e_{n+2} \oplus \RR\vb 1.
\]
We can also simplify the description of  \(f_{\overline\cU^X_{1}\times \RR^{2}}(X)\). Noting that \(\mathbf{e}_{0} =\mathbf{e}_{2n} = \sigma(x^{(n)}\vb 1 - X) = x^{(n)} \mathbf{1} - X\) and \(\mathbf{e}_{1} = \sigma(X - x^{(1)}\vb 1)=X - x^{(1)} \mathbf{1}\), both being in $\RR X \oplus \RR\vb 1$, we find 
\[
f_{\overline\cU^X_{1}\times \RR^{2}}(X) =\RR \left[\vb e_{0}, ~ \vb e_{1}\right] + \RR \vb 1 = \RR X \oplus \RR\vb 1.
\]

We can also state the following corollary which provides the dimensions of the different sets. In the corollary, we denote by $\dim$ the dimension in the manifold sense and the notation of $\Inter$ for the interior of the set.

 The four distinct behaviors, delineated in the following corollary, correspond to the activation regions depicted in Figure \ref{ima:regions}, color-coded in red, green, blue, and gray, respectively. Before stating the corollary, we remind that for all $j\in\lb1,p_X\rb$, there exists $i\in\lb1,2n\rb$ such that $\widetilde\cU^X_{j} = \Inter\left(\overline\cU^X_{i} \times\RR^2\right)$.

\begin{coro}[Architecture $ (N_0,N_1,N_2)= (1,1,1)$]\label{cool-dimension-1}
Assume that \(n \geq 2\) and the sample \(X\in\RR^{1 \times n}\) satisfies \eqref{x_ordered}. We have, 
\[
	\dim f_{\Inter\left(\overline\cU^X_{i} \times\RR^2\right)}(X) 
    = \left\{\begin{array}{lll}
    2 &\mbox{, if } i =1,\\
	 3 &\mbox{, if } i =2, \dots n-1, n+3, \dots 2n,\\
	 2 &\mbox{, if } i =n, n+2,\\
	 1 &\mbox{, if } i =n+1.
    \end{array}\right.
\]
\end{coro}

%%%%%%%%%%%%%%%%%%%%%
\subsection{Proof of Theorem \ref{shallow-case-thm}, for the Architecture \texorpdfstring{$(1,N_1,1)$}{(1,N1,N2)}}\label{archi-ONENN}

We extend the results from the previous section to neural networks with architecture $(1,N_{1},1)$, for a positive integer $N_{1}\geq 1$. The sample \(X = (x^{(1)}, \dots, x^{(n)}) \in\RR^{1 \times n} \) is assumed to satisfy \eqref{x_ordered}. We simplify notations  and, throughout the section, denote the parameters of the neural network $w,b \in \RR^{N_{1}}$, $V\in \RR^{1\times N_{1}}$ and $c \in \mathbb{R}^{1}$. The image of $X$ is defined by
\begin{equation}\label{arch-1-d1-d2}
	f_{(w,V,b,c)}(X) = V\sigma (wX + b\mathbf{1}) + c\mathbf{1} \in \mathbb{R}^{1 \times n},
\end{equation}
where $\One \in\RR^{1 \times n}$.

Expanding this equation to explicitly represent the vector-matrix multiplication of the second layer yields
\begin{equation}\label{eqn:fexp-d1-d2}
    f_{(w,V,b,c)}(X) = 
    \sum_{i=1}^{N_{1}} V_{1,i} \sigma\left( w_{i} X + b_{i} \vb 1\right) + c  \vb 1
\end{equation}
with $\One \in\RR^{1 \times n}$. As in the previous section, we also explicit the important parameters for the activation pattern. The latter is given by the matrix $a(X, w, b) \in \RR^{N_{1}\times n}$ defined, for all $i=1, \dots, N_{1}$ and $j=1, \dots, n$ by
\[
	a(X, w, b)_{i,j} = 
			 \begin{cases}
			1 & \text{if } w_{i} x^{(j)} + b_{i} \geq 0,\\
			0 & \text{otherwise}.
			\end{cases}
\]
Note that for each $i=1, \dots, N_{1}$, the row $a(X, w,b)_{i,:}$ coincides with the activation pattern $a(X, w_{i}, b_{i})$ for the $(1,1,1)$ architecture. As discussed in the previous section, since $X$ satisfies \eqref{x_ordered}, a consequence of \Cref{lem:sets} is that the achievable activation patterns associated with the $(1,1,1)$ architecture are the row vectors $\vb 1_{1}, \dots, \vb 1_{2n}$, as defined in \eqref{eqn:1i-vectors}. This leads us to introduce the following notation: for a $N_{1}$-tuple $\alpha = (\alpha_{1}, \dots, \alpha_{N_{1}})\in\{1, \dots, 2n\}^{N_{1}}$, we define $\vb 1_{\alpha}$ as 
	\[
	\vb 1_{\alpha} = \begin{pmatrix}
		\vb 1_{\alpha_{1}}\\
		\vdots\\
		\vb 1_{\alpha_{N_{1}}}
		\end{pmatrix} \in \RR^{N_{1}\times n}.
	\]
As for the architecture $(1,1,1)$, we define the sets
\begin{equation}\label{def-overlineU}
 	\overline{\mathcal U}^X_{\alpha} = \{(w,b)\in \RR^{N_{1}}\times \RR^{N_{1}} \mid a(X, w,b) = \vb 1_{\alpha} \}.   
\end{equation}

It is immediate to verify that the pair of vectors $(w,b)\in \RR^{N_{1}}\times \RR^{N_{1}}$ belongs to $\overline{\mathcal U}^X_{\alpha}$ if and only if, for all $i=\lb 1, N_{1}\rb$, the $2$ dimensional point $(w_{i}, b_{i})$ belongs to $\overline\cU^X_{\alpha_{i}}$. Therefore, due to \Cref{lem:sets}, it follows that the activation regions $\overline{\mathcal U}^X_{\alpha}$, for $\alpha \in\{1, \dots, 2n\}^{N_{1}}$, partition $\RR^{N_{1}}\times \RR^{N_{1}}$.

%For any $\alpha \in \{1, \dots, 2n\}^{N_{1}}$, we define
%\[
%f_{\overline{\mathcal U}^X_{\alpha}\times \RR^{1\times N_{1}} \times \RR} (X) = \big\{f_{(w,V,b,c)}(X) \in\RR^{1\times n} \mid (w,b) \in\overline{\mathcal U}^X_{\alpha},  V\in \RR^{1\times N_{1}}\mbox{ and } c\in \RR\big\}.
%\]
For any $\calC\subset \RR^{N_{1}}\times \RR^{N_{1}}$, we define
\[
f_{\calC \times \RR^{1\times N_{1}} \times \RR} (X) = \big\{f_{(w,V,b,c)}(X) \in\RR^{1\times n} \mid (w,b) \in \calC,  V\in \RR^{1\times N_{1}}\mbox{ and } c\in \RR\big\}.
\]
Using that $\bigcup_{\alpha \in \{1, \dots, 2n\}^{N_{1}}} \overline{\mathcal U}^X_{\alpha} = \RR^{N_{1}}\times \RR^{N_{1}}$, we have 
$$f_{\RR^{N_1}\times \RR^{N_1}\times \RR^{1\times N_{1}} \times \RR} (X) = \bigcup_{\alpha \in \{1, \dots, 2n\}^{N_{1}}} f_{\overline{\mathcal U}^X_{\alpha}\times \RR^{1\times N_{1}} \times \RR} (X).
$$

We now describe the output of neural networks with the architecture $(1,N_{1},1)$ across the activation regions $\overline{\mathcal U}^X_{\alpha}$, employing  the sets $\mathcal V_{i}\subset \RR^{1\times n}$ introduced in \eqref{def_Vi} and parameterized in \Cref{lem:images}, for the architecture $(1,1,1)$.

\begin{prop}[Architecture $(1,N_{1},1)$]\label{prop:1d1d2}
Assume that \(n \geq 2\) and the sample \(X\in\RR^{1 \times n}\) satisfies \eqref{x_ordered}. We have, for any $\alpha \in \{1, \dots, 2n\}^{N_{1}}$, 
\begin{equation}\label{param_image}
f_{\overline{\mathcal U}^X_{\alpha}\times \RR^{1\times N_{1}} \times \RR} (X)
	=    \RR \mathcal V_{\alpha_{1}} + \cdots + \RR \mathcal V_{\alpha_{N_{1}}} + \RR\vb 1 ,
\end{equation}
where $\vb 1 \in\RR^{1 \times n}$. 

Moreover for all $\theta\in \Inter{\left(\overline{\mathcal U}^X_{\alpha}\right)}\times \RR^{1\times N_{1}} \times \RR$, the mapping $\theta' \longmapsto f_{\theta'}(X)$ is differentiable at $\theta$ and
\begin{equation}\label{eruvnoetb}
    \RK{ \DT f_{\theta}(X) } 
	= \rank\left(\vb 1, \vb 1_{\alpha_{1}} \odot X, \vb 1_{\alpha_{1}}, \ldots,  \vb 1_{\alpha_{N_1}} \odot X, \vb 1_{\alpha_{N_1}}
    \right),
\end{equation}
where we remind that $\odot$ stands for the Hadamard product.

\end{prop}
\begin{proof}
To first prove that the set on the left of the equality sign of \eqref{param_image} is included in the set on the right, we consider  $\alpha \in \{1, \dots, 2n\}^{N_{1}}$ and an arbitrary $(w,b)\in \overline{\mathcal U}^X_{\alpha}$. For all $i=1, \dots, N_{1}$, we have $(w_{i}, b_{i})\in \overline{\mathcal{U}}^X_{\alpha_{i}}$,
and, using \eqref{def_Vi}, 
\[\sigma(w_{i}X + b_{i}\vb 1)\in \cV_i.
\]
Using \eqref{eqn:fexp-d1-d2}, we find
\begin{eqnarray*}
    	 f_{(w,V,b,c)}(X) & = &\sum_{i=1}^{N_{1}} V_{1,i} \sigma(w_{i}X + b_{i}\vb 1) + c \mathbf{1} , \quad \text{for }  V_{1,i}, \dots,  V_{1,N_{1}},c\in \RR, \\
         & \in & \RR \mathcal  V_{\alpha_{1}} + \cdots + \RR \mathcal V_{\alpha_{N_{1}}} + \RR\vb 1 .
\end{eqnarray*}
This proves that, for all $\alpha \in \{1, \dots, 2n\}^{N_{1}}$, 
\[f_{\overline{\mathcal U}^X_{\alpha}\times \RR^{1\times N_{1}} \times \RR} (X)
	~\subset~  \RR \mathcal V_{\alpha_{1}} + \cdots + \RR \mathcal V_{\alpha_{N_{1}}} + \RR\vb 1.
\]

Conversely, we consider $\alpha \in \{1, \dots, 2n\}^{N_{1}}$, arbitrary elements $y_i\in \cV_{\alpha_i}$, for all $i=1,\dots,N_1$, and arbitrary real numbers $ V_{1,1}, \dots,  V_{1,N_{1}}$, $c$. The point $\sum_{i=1}^{N_{1}} V_{1,i} y_i + c\One$ is therefore an arbitrary element of $\RR \mathcal  V_{\alpha_{1}} + \cdots + \RR \mathcal V_{\alpha_{N_{1}}} + \RR\vb 1.$
Using the definition of $\cV_i$, in \eqref{def_Vi}, we know that, for all $i$, there exists  $(w_i,b_i) \in \overline{\mathcal{U}}^X_{\alpha_{i}}$ such that 
\[y_i = \sigma(w_{i}X + b_{i}\vb 1).
\]
Setting $w=(w_i)_{i=1,\ldots,N_1}$, $b=(b_i)_{i=1,\ldots,N_1}$, and $V=(V_{1,i})_{i=1,\ldots,N_1}$, we find $(w,b) \in \overline{\mathcal{U}}^X_{\alpha}$ and 
\[\sum_{i=1}^{N_{1}} V_{1,i} y_i + c\One = f_{(w,V,b,c)}(X) \in f_{\overline{\mathcal U}^X_{\alpha}\times \RR^{1\times N_{1}} \times \RR} (X).
\]
This proves the converse inclusion and finishes the proof of \eqref{param_image}.

To prove the differentiability statement and \eqref{eruvnoetb}, we consider $\alpha \in \{1, \dots, 2n\}^{N_{1}}$, $\theta\in \Inter{\left(\overline{\mathcal U}^X_{\alpha}\right)}\times \RR^{1\times N_{1}} \times \RR$ and $\varepsilon>0$ such that $B(\theta,\varepsilon) \subset \Inter{\left(\overline{\mathcal U}^X_{\alpha}\right)}\times \RR^{1\times N_{1}} \times \RR$. We have, for all $\theta'=(w',V',b',c')\in B(\theta,\varepsilon)$,
\[    f_{\theta'}(X) = \sum_{i=1}^{N_{1}} V'_{1,i} \sigma(w'_{i}X + b'_{i}\vb 1) + c' \mathbf{1}.
\]
Using \eqref{property-activation} and \eqref{defn:activations} and  since $(w',b') \in \Inter{\left(\overline{\mathcal U}^X_{\alpha}\right)}$, we obtain
\begin{eqnarray}
    f_{\theta'}(X) & =&  \sum_{i=1}^{N_{1}} V'_{1,i} \Bigl(\vb 1_{\alpha_i}\odot (w'_{i}X + b'_{i}\vb 1)\Bigr) + c' \mathbf{1} \nonumber\\
    & =&  \sum_{i=1}^{N_{1}} \Bigl(V'_{1,i} w'_{i}~\bigl(\vb 1_{\alpha_i}\odot X\bigr) + V'_{1,i} b'_{i} ~\vb 1_{\alpha_i} + c' ~\mathbf{1} \Bigr). \label{ernunbt}
\end{eqnarray}  
The term on the right of the above equality sign is a polynomial in $\theta'$ and therefore  $\theta' \longmapsto f_{\theta'}(X)$ is differentiable at $\theta$. Also, given \eqref{ernunbt}, we can construct an open set $\cO_\varepsilon\subset \RR^{1\times n}$ such that
\[ \Bigl(\cO_\varepsilon \cap \calW_\alpha  \Bigr) ~\subset~ f_{B(\theta,\varepsilon)} (X) ~ \subset ~ \calW_\alpha,
\]
where
\[\calW_\alpha = \vect \left({\vb 1, \vb 1_{\alpha_{1}} \odot X, \vb 1_{\alpha_{1}}, \ldots,  \vb 1_{\alpha_{N_1}} \odot X, \vb 1_{\alpha_{N_1}} }
\right).
\]
Therefore $\Range(\DT f_{\theta}(X)) = \calW_\alpha$ and \eqref{eruvnoetb} holds.

This concludes the proof of \Cref{prop:1d1d2}.
\end{proof}

Notice that, in \eqref{param_image}, for all $\alpha \in \{1,\dots, 2n\}$, we deduce from \Cref{lem:images} that
\[
	\RR \mathcal V_{\alpha} + \RR \mathcal V_{\alpha} = \Span\{\vb e_{\alpha-1}, \vb e_{\alpha}\}.
\]

The next proposition makes the link with the notations in the main part of the article and makes a step towards the simplification of \eqref{eruvnoetb}.
\begin{prop}[Architecture $(1,N_{1},1)$]\label{cor:1d1d2}
Assume that \(n \geq 2\) and the sample \(X\in\RR^{1 \times n}\) satisfies \eqref{x_ordered}. For any $j\in\lb1,p_X\rb$, there exists $\alpha \in \{1, \dots, 2n\}^{N_{1}}$, such that
\[\widetilde{\mathcal U}^X_{j} = \Inter{\left(\overline{\mathcal U}^X_{\alpha}\right)}\times \RR^{1\times N_{1}} \times \RR
\]
and for all $\theta\in \widetilde{\mathcal U}^X_{j}$
\begin{equation}\label{rzpiqriuntb}
    	\RK{ \DT f_{\theta}(X) } 
	= \rank\left(\vb 1, \vb e_{\alpha_{1}-1}, \vb e_{\alpha_{1}}, \dots, \vb e_{\alpha_{N_{1}}-1}, \vb e_{\alpha_{N_{1}}}\right),
\end{equation}
where we remind that $\vb e_{0} = \vb e_{2n}$ and  $\vb e_{n} =\vb e_{n+1} =0 $.

As a consequence, $\widetilde{\mathcal U}^X_{j}=\mathcal U^X_{j}$.
\end{prop}
\begin{proof}
The first statement is a direct consequence of the definition of $\widetilde \cU^X_j$, in \eqref{def:u:j:tilde}, and the definition of $\overline{\mathcal U}^X_{\alpha}$, in \eqref{def-overlineU}.

To prove \eqref{rzpiqriuntb}, we prove in the following that, 
\begin{equation}\label{jetounbt}
    \vect\Bigl(\vb e_{i}, \vb e_{i-1} \Bigr) = \vect\Bigl(\vb 1_{i} \odot X, \vb 1_{i} \Bigr), \qquad \mbox{for all }i\in\{1,\ldots,2n\}.
\end{equation}
Once this is established, \eqref{rzpiqriuntb} is indeed a direct consequence of \eqref{eruvnoetb} and \eqref{jetounbt}.

To prove \eqref{jetounbt}, we distinguish four cases: $i=1$, $i\in\{2,\ldots,n\}$, $i=n+1$, and $i\in\{n+2,\ldots,2n\}$. All the cases rely on \eqref{def_bfe} and \eqref{defDesOnes}, which we remind here: For all $i\in\{1,\ldots, n\}$,
\[\left\{\begin{array}{l}
    \bfe_{i} = (0,\dots, \underset{\underset{i}{\uparrow}}{0},  x^{(i+1)} - x^{(i)}, \dots, x^{(n)}-x^{(i)}), \\  
	\bfe_{n+i} = (x^{(i)}-x^{(1)}, \dots, x^{(i)}-x^{(i-1)},  \underset{\underset{i}{\uparrow}}{0}, \dots, 0).
    \end{array}\right.
\]
We also set $\bfe_0 = \bfe_{2n}$. Finally, for all $i\in\{1,\ldots, n\}$,
\[\One_{i} = (0,\dots,0, \underset{\underset{i}{\uparrow}}{1},  \dots, 1)\in\RR^{1\times n} \quad \mbox{and}\quad
	\One_{n+i} = (1, \dots,1, \underset{\underset{i}{\uparrow}}{0}, \dots, 0)\in\RR^{1\times n}.
\]
\begin{itemize}
    \item Case $i=1$: We have $\bfe_1= \vb 1_{1} \odot X - x^{(1) }\vb 1_{1}$ and  $\bfe_{0} =\bfe_{2n} = x^{(n)} \vb 1_1 - \vb 1_1\odot X$. This proves that $\vect\bigl(\vb e_{1}, \vb e_{0} \bigr) \subset \vect\bigl(\vb 1_{1} \odot X, \vb 1_{1} \bigr)$.

    Conversely, since $x^{(n)} \neq x^{(1)}$, the equalities $(x^{(n)} - x^{(1)})\vb 1_{1} =  \bfe_1 + \bfe_0$ and $\vb 1_{1} \odot X= X = \vb e_1 + x^{(1)} \vb 1_{1} $ prove that $\vect\bigl(\vb 1_{1} \odot X, \vb 1_{1} \bigr) \subset \vect\bigl(\vb e_{1}, \vb e_{0} \bigr)$.
    \item Case $i\in\{2,\ldots,n\}$: We have $\bfe_i = \vb 1_{i} \odot X - x^{(i) }\vb 1_{i}$ and  $\bfe_{i-1} = \bfe_{i} + ( x^{(i)} - x^{(i-1)})\vb 1_{i}$, which proves that $\vect\bigl(\vb e_{i}, \vb e_{i-1} \bigr) \subset \vect\bigl(\vb 1_{i} \odot X, \vb 1_{i} \bigr)$. We conclude using that, since $X$ satisfies \eqref{x_ordered}, we have
    \begin{eqnarray*}
        \dim  \vect\bigl(\vb e_{i}, \vb e_{i-1} \bigr) & = & \left\{\begin{array}{ll}
        2 & \mbox{ if } i\in\{2,\ldots, n-1\} \\
        1 & \mbox{ if } i=n
    \end{array}\right. \\
    & = & \dim \vect\bigl(\vb 1_{i} \odot X, \vb 1_{i} \bigr).
    \end{eqnarray*}
    \item Case $i=n+1$: We have $\vb 1_{n+1} = \vb 1_{n+1} \odot X = \bfe_{n+1} = \bfe_n = 0$ and therefore  $\vect\bigl(\vb e_{n+1}, \vb e_{n} \bigr) = \vect\bigl(\vb 1_{n+1} \odot X, \vb 1_{n+1} \bigr)$.
    \item Case $i\in\{n+2,\ldots,2n\}$: The equalities $\bfe_i = x^{(i-n)} \vb 1_i - \vb 1_i \odot X$ and $\bfe_{i-1} =\bfe_{i} + \bigl(x^{(i-n-1)} - x^{(i-n)}\bigr) \vb 1_i $ lead to $\vect\bigl(\vb e_{i}, \vb e_{i-1} \bigr) \subset \vect\bigl(\vb 1_{i} \odot X, \vb 1_{i} \bigr)$.

    We conclude using that, since $X$ satisfies \eqref{x_ordered}, we have
    \begin{eqnarray*}
        \dim  \vect\bigl(\vb e_{i}, \vb e_{i-1} \bigr) & = & \left\{\begin{array}{ll}
        2 & \mbox{ if } i\in\{n+3,\ldots, 2n\} \\
        1 & \mbox{ if } i=n+2
    \end{array}\right. \\
    & = & \dim \vect\bigl(\vb 1_{i} \odot X, \vb 1_{i} \bigr).
    \end{eqnarray*}
\end{itemize}

The fact that $\widetilde{\mathcal U}^X_{j}=\mathcal U^X_{j}$ is then a direct consequence of the definition of $\mathcal U^X_{j}$, in \eqref{defU}, and \eqref{rzpiqriuntb}. This concludes the proof of Proposition \ref{cor:1d1d2}.
\end{proof}

Before proving Theorem \ref{shallow-case-thm}, we state and prove a similar theorem with more accurate, but less interpretable, upper and lower bounds. We will then deduce the bounds of Theorem \ref{shallow-case-thm} from \eqref{upper_bound_rank_shallow-accurate}.

\begin{thm}\label{shallow-case-thm-accurate}
    Consider any deep fully-connected ReLU network architecture $(E,V, Id)$, with $L=2$ and $N_0=N_2=1$. Consider $n\in\NN^*$, and a sample $X=(x^{(1)},x^{(2)}, \ldots,  x^{(n)})\in\RR^{1\times n}$ satisfying \eqref{x_ordered_main_part}. 

    For any $j\in\lb1,p_X\rb$, there exists $\alpha \in \{1,\dots,2n\}^{N_1}$ such that for all $\theta \in \widetilde\cU_j^X$ and all $k\in \lb1,N_1\rb$, $a(X,\theta)_{k,:}= \One_{\alpha_k} $, and 
    \begin{equation}\label{formule_rank_shallow-accurate}
    \RK{ \DT f_{\theta}(X) } = \RK{\One, \bfe_{\alpha_1 -1},\bfe_{\alpha_1 }, \ldots ,\bfe_{\alpha_{N_1} -1},\bfe_{\alpha_{N_1} }  }.
    \end{equation}
     As a consequence, $\widetilde\cU_j^X = \cU_j^X$ and for all $\theta \in \cU_j^X$
    \begin{equation}\label{upper_bound_rank_shallow-accurate}
     \Big(1+\frac{1}{2}\ell^0_{neurons}(\alpha) \Big) \leq 
    \RK{ \DT f_{\theta}(X) }  \\
    \leq  \min\Big(1+\ell^0_{neurons}(\alpha), \ell^0_{linear}(f_{\theta}, X) \Big),
    \end{equation}
    with
    \begin{multline}\label{def-l0-neurons}
    \ell^0_{neurons}(\alpha) = \\
    \left|\big\{l\in\lb0, 2n\rb\setminus\{n,n+1\} \mid \mbox{there exists } k\in\lb 1,N_1\rb, \alpha_k = l \mbox{ or }\alpha_k-1 = l\big\}\right|,     
    \end{multline}
    represents the number of effective neurons in the hidden-layer and
    \begin{equation}\label{def-l0-linear}
\ell^0_{linear}(f_{\theta}, X)= \sum_{\delta\in\{0,1\}^{N_1} } \min\Bigl(2, \Bigl|\Bigl\{i\in\lb1,n\rb \mid a(x^{(i)}, \theta)=\delta\Bigr\}\Bigr| \Bigr)
    \end{equation}
    is the number of linear regions of $f_\theta$ perceived by $X$.
\end{thm}

Note that we can also write $\ell^0_{neurons}$ as
\[
\ell^0_{neurons}(\alpha) =
\left| 
\left\{ 
\alpha_1,\alpha_1-1,
\ldots,
\alpha_{N_1},\alpha_{N_1}-1
\right\}
\backslash 
\left\{ 
n,n+1
\right\}
\right|.
\]

\begin{proof}%{\bf of Theorem \ref{shallow-case-thm-accurate}:}
Notice first that the hypotheses on the neural network architecture and the sample $X$ in Theorem \ref{shallow-case-thm-accurate} are identical to the hypotheses in this section.

Consider  $j\in\lb1,p_X\rb$. Proposition \ref{cor:1d1d2} guarantees that there exists $\alpha \in \{1, \dots, 2n\}^{N_{1}}$, such that
\[\widetilde{\mathcal U}^X_{j} = \Inter{\left(\overline{\mathcal U}^X_{\alpha}\right)}\times \RR^{1\times N_{1}} \times \RR.
\]
Using the definition of $\widetilde \cU^X_j$, in \eqref{def:u:j:tilde}, and the definition of $\overline{\mathcal U}^X_{\alpha}$, in \eqref{def-overlineU}, we also know that $a(X,\theta) = {\mathbf 1}_{\alpha}$ which implies that for all $\theta \in \widetilde\cU_j^X$ and all $k\in\lb1,N_1\rb$, 
\[a(X,\theta)_{k,:} = {\mathbf 1}_{\alpha_k}.
\]
The second statement of Proposition \ref{cor:1d1d2} is exactly \eqref{formule_rank_shallow-accurate}. This guarantees that the part of Theorem \ref{shallow-case-thm-accurate} until \eqref{formule_rank_shallow-accurate} holds.

Let us now prove \eqref{upper_bound_rank_shallow-accurate}.

To do so, we first remove the elements of the list $\left(\vb 1, \vb e_{\alpha_{1}-1}, \vb e_{\alpha_{1}}, \dots, \vb e_{\alpha_{N_{1}}-1}, \vb e_{\alpha_{N_{1}}}\right) $ which do not contribute to the rank (those which are repeated or null). We consider
    \begin{equation}\label{eriquvt}
        \cL(\alpha)= \big\{l\in\lb0, 2n\rb\setminus\{n,n+1\} \mid \mbox{there exists } k\in\lb1,N_1\rb, \alpha_k = l \mbox{ or }\alpha_k-1 = l\big\},
    \end{equation}
 and we have
    \[
    \RK{ \DT f_{\theta}(X) } = \RK{\{\One\}\cup \{\vb e _l\mid {l\in \cL(\alpha)}\}}.
    \]

    To express the right-hand side of this equation in matrix form, let $A \in \RR^{(1 + |\cL(\alpha)|) \times n}$ be the matrix whose first row is the vector $\One$, and whose remaining rows are the vectors $\{\vb e_l \mid l \in \cL(\alpha)\}$. We obtain
    \[\RK{ \DT f_{\theta}(X) } = \RK{A}.
    \]
Notice that, given \eqref{def-l0-neurons}, we have $\ell^0_{neurons}(\alpha) = |\cL(\alpha)|$. Given the definition of $\ell^0_{linear}(f_{\theta}, X)$, in \eqref{def-l0-linear}, to establish the upper-bound in \eqref{upper_bound_rank_shallow-accurate}, it suffices to prove that for any activation pattern $\delta\in\{0,1\}^{N_1}$, if
\[I(\delta) = \bigl\{i\in\lb1,n\rb \mid a(x^{(i)}, \theta)=\delta\bigr\}
\]
contains $3$ elements or more, the corresponding columns of $A$ are linearly dependent. If this property holds, indeed, we can remove the linearly dependent columns from $A$ and obtain a matrix $A'\in\RR^{(1 + |\cL(\alpha)|) \times \ell^0_{linear}(f_{\theta}, X)}$ such that $\RK{A'} = \RK{A}$ and deduce the upper-bound.

To prove the property, we consider $\delta\in\{0,1\}^{N_1}$. We assume that $|I(\delta)| > 2$ and we prove that, 
\begin{equation}\label{erttbunojqetb}
\mbox{for all }i\in I(\delta), \qquad A_{:,i} = v + x^{(i)} u,
\end{equation}
where the vectors $v\in\RR^{1+|\cL(\alpha)|}$ and $u\in\RR^{1+|\cL(\alpha)|}$ depend on $\delta$ but not on $i$. More precisely, denoting by $l_m$ the element of $\cL(\alpha)$ corresponding to the row $m\in\lb2, 1+|\cL(\alpha)|\rb$, in $A$, $i_{min} = \min \{i\in I(\delta)\}$ and $i_{max} = \max\{i\in I(\delta)\}$, we set for all $m\in\lb1,1+|\cL(\alpha)|\rb$
\[
v_m=\left\{
\begin{array}{ll}
1 & \mbox{if }m=1 \\
x^{(n)} & \mbox{if }l_m=0 \\
- x^{(l_m)} & \mbox{if }1\leq l_m\leq i_{min} \\
0 & \mbox{if } i_{max}\leq l_m \leq n+i_{min} \\
x^{(l_m-n)} & \mbox{if } n+i_{max}\leq l_m \leq 2n \\
\end{array}
\right. 
\mbox{ and }\quad
u_m=\left\{
\begin{array}{ll}
0 & \mbox{if }m=1 \\
-1 & \mbox{if }l_m=0 \\
1 & \mbox{if }1\leq l_m\leq i_{min} \\
0 & \mbox{if } i_{max}\leq l_m \leq n+i_{min} \\
-1 & \mbox{if } n+i_{max}< l_m \leq 2n \\
\end{array}
\right..
\]
Notice that, because of the definition of $I(\delta)$, the columns $a(X,\theta)_{:,i}$, for $i\in I(\delta)$, are all equal. Therefore, for all $k\in\lb1,N_1\rb$, given \eqref{defDesOnes} and since $a(X,\theta)_{k,:}=\One_{\alpha_{k}}$, we have $\alpha_k\not\in \lb i_{min}+1,i_{max}\rb \cup \lb n+i_{min}+1,n+i_{max}\rb $. Given the definitions of $\cL(\alpha)$ and $l_m$, we have $l_m=\alpha_k$ or $l_m=\alpha_k-1$, 
for some $k \in \lb 1,N_1\rb$,
for all row $m\in\lb2, 1+|\cL(\alpha)|\rb$, and therefore $l_m\not\in\lb i_{min}+1,i_{max}-1\rb \cup\lb n+i_{min}+1,n+i_{max}-1\rb $. 
As a consequence, all the components of $u$ and $v$ are properly defined by the above definition. The vectors $u$ and $v$ depend only on $\cL(\alpha)$ and $I(\delta)$, and are the same for all $i\in I(\delta)$.

To prove that \eqref{erttbunojqetb} holds, we consider $i\in I(\delta)$ and $m\in\lb1,1+|\cL(\alpha)|\rb$. Using the definition of $A$ and reminding that, using \eqref{def_bfe},
\[ \left\{\begin{array}{ll}   
\bfe_0 = \bfe_{2n} = (x^{(n)} - x^{(1)},x^{(n)} - x^{(2)}, \ldots, x^{(n)} - x^{(n)})& \mbox{ if } l_m = 0 \\
\bfe_{l_m} = (0,\dots, \underset{\underset{l_m}{\uparrow}}{0},  x^{(l_m+1)} - x^{(l_m)}, \dots, x^{(n)}-x^{(l_m)}) & \mbox{ if } 1 \leq l_m\leq n \\
\bfe_{l_m} = (x^{(l_m-n)}-x^{(1)}, \dots, x^{(l_m-n)}-x^{(l_m-n-1)},  \underset{\underset{l_m-n}{\uparrow}}{0}, \dots, 0)& \mbox{ if } n+1 \leq  l_m\leq 2n \\
\end{array} \right.
,
\]
we obtain
\[A_{m,i} = \left\{\begin{array}{lll}
 1    & = v_m + x^{(i)} u_m  &  \mbox{if }m=1 \\
x^{(n)} - x^{(i)}& = v_m + x^{(i)} u_m  & \mbox{if }l_m=0 \\
x^{(i)} - x^{(l_m)}& = x^{(i)} u_m +  v_m  & \mbox{if }1\leq l_m\leq i_{min} \\
0 & = v_m + x^{(i)} u_m  & \mbox{if } i_{max}\leq l_m \leq n+i_{min} \\
x^{(l_m-n)} - x^{(i)}& = v_m + x^{(i)} u_m  & \mbox{if } n+i_{max}\leq l_m \leq 2n \\
\end{array}
\right. .
\]
This concludes the proof of \eqref{erttbunojqetb}. As already said, this proves that, for all $\delta\in\{0,1\}^{N_1}$ such that $|I(\delta)|\geq 3$, we can remove $I(\delta)-2$ columns from $A$ without changing its rank. Doing so for all $\delta$, we obtain a matrix $A'\in\RR^{(1 + |\cL(\alpha)|) \times \ell^0_{linear}(f_{\theta}, X)}$ such that $\RK{A'} = \RK{A}$. Reminding that $|\cL(\alpha)| = \ell^0_{neurons}(\alpha)$, we then deduce the upper-bound of \eqref{upper_bound_rank_shallow-accurate}:
\[\RK{ \DT f_{\theta}(X) } = \RK{A'} \leq \min(1+\ell^0_{neurons}(\alpha),\ell^0_{linear}(f_\theta,X)) .
\]

To prove the lower-bound of \eqref{upper_bound_rank_shallow-accurate}, we remark that given the forms of the vectors $\vb e_l$, in \eqref{def_bfe}, up to re-ordering the lines of $A'$, if we draw positive values in blue and null values in red, $A'$ has the form
\[
A'= \begin{pmatrix}
\color{blue}{ \rule[0pt]{6cm}{3pt} } \\
\color{red}{ \rule[0pt]{0.5cm}{3pt} }\color{blue}{ \rule[0pt]{5.5cm}{3pt} } \\
\color{red}{ \rule[0pt]{2cm}{3pt} }\color{blue}{ \rule[0pt]{4cm}{3pt} } \\
\color{red}{ \rule[0pt]{3.5cm}{3pt} }\color{blue}{ \rule[0pt]{2.5cm}{3pt} } \\
\color{red}{ \rule[0pt]{4.5cm}{3pt} }\color{blue}{ \rule[0pt]{1.5cm}{3pt} } \\
\color{blue}{ \rule[0pt]{1.5cm}{3pt} }\color{red}{ \rule[0pt]{4.5cm}{3pt} } \\
\color{blue}{ \rule[0pt]{2.5cm}{3pt} }\color{red}{ \rule[0pt]{3.5cm}{3pt} } \\
\color{blue}{ \rule[0pt]{5cm}{3pt} }\color{red}{ \rule[0pt]{1cm}{3pt} } \\
\end{pmatrix} \in \RR^{(1+\ell^0_{neurons}(\alpha) )\times \ell^0_{linear}(f_\theta,X)}.
\]
We can extract from $A'$ a full row rank matrix by keeping its upper-triangular part or by keeping the part below the  upper-triangular part, which we can augment by the first line of $A$. The largest of the two matrices has more than $1+\frac{1}{2}\ell^0_{neurons}(\alpha)$ lines. This proves the lower-bound of \eqref{upper_bound_rank_shallow-accurate}.

This concludes the proof of Theorem \ref{shallow-case-thm-accurate}.
\end{proof}
%%%%%%%%%%%%%%%%%%%%%%%%%%%%%%%%%%%%%%%%%%%%%%%%%%%%%%%%%%%
% Premier Lemme sur A, L', L''
We denote, for all $X$ and $\theta$,
\begin{equation}\label{def_A_annexe}
    \cA(X,\theta) = \{\delta \in\{0,1 \}^{N_1} \mid \mbox{there exists }i\in\lb1,n\rb \mbox{, such that } a(x^{(i)}, \theta) = \delta \}.
\end{equation}
The set $\cA(X,\theta)$ contains all the activation patterns perceived by $X$.

We denote, for all $\alpha \in \lb1,2n\rb^{N_1}$
\[\cL'(\alpha) = \big\{l\in\lb0, 2n\rb\setminus\{n,n+1\} \mid \mbox{there exists } k\in\lb1,N_1\rb, \alpha_k = l \big\}.
\]
The next definition is similar to the usual \lq modulo\rq{}. This is the reason why we abuse of its notation. For all $l'\in\cL'(\alpha)$, we denote
\[
l'[n] = \left\{\begin{array}{ll}
n & \mbox{if } l'=0 \\
l' & \mbox{if } l'\in\lb1,n-1\rb \\
l'-n & \mbox{if } l'\in\lb n+2,2n\rb
\end{array}\right. .
\]
Notice that we always have $l'[n] \in \lb1,n\rb$. We finaly denote, for all $\alpha \in \lb1,2n\rb^{N_1}$
\[\cL''(\alpha) = \big\{l\in\lb1, n\rb\mid \mbox{there exists } l'\in\cL'(\alpha) \mbox{, such that } l = l'[n] \big\}.
\]
Due to the fact that to every element of $\cL''(\alpha)$ corresponds at least one, and at most two, elements of $\cL'(\alpha)$, we have
   \begin{equation}\label{ineq2}
        |\cL''(\alpha)| \leq |\cL'(\alpha)| \leq 2|\cL''(\alpha)|.
    \end{equation}
    
The following lemma makes the connection between $|\cL''(\alpha)|$
and $ |\cA(X,\theta)|$.
\begin{lem}\label{lem12} 
 Consider any deep fully-connected ReLU network architecture $(E,V, Id)$, with $L=2$ and $N_0=N_2=1$. Consider $n\in\NN^*$, and a sample $X=(x^{(1)},x^{(2)}, \ldots,  x^{(n)})\in\RR^{1\times n}$ satisfying \eqref{x_ordered_main_part}. 

    For any $j\in\lb1,p_X\rb$, there exists $\alpha \in \{1,\dots,2n\}^{N_1}$ such that for all $\theta \in \widetilde\cU_j^X$ and all $k\in \lb1,N_1\rb$, $a(X,\theta)_{k,:}= \One_{\alpha_k} $, and 
    \begin{equation}\label{ineq3}
        |\cA(X,\theta)| - 2 \leq |\cL''(\alpha)| \leq |\cA(X,\theta)|.
    \end{equation}

\end{lem}
\begin{proof}
 To establish  \eqref{ineq3}, we consider the mapping
    \begin{eqnarray*}
        f: \cL''(\alpha) & \longrightarrow & \cA(X,\theta) \\
        l & \longmapsto  & a(x^{(l)}, \theta ).
    \end{eqnarray*}
    Below, we first prove that $f$ is injective and then prove that at most two element of $\cA(X,\theta)$ are not in the range of $f$. This leads to \eqref{ineq3}.

    To establish that $f$ is injective, we consider $l$ and $l'\in\cL''(\alpha)$ such that $l\neq l'$. Without loss of generality, we assume that $l< l'$. Because of the definition of $\cL''(\alpha)$ and $\cL'(\alpha)$, we know that there exists $k\in\lb1,N_1\rb$ such that $l' = \alpha_{k}[n]$. We consider such a $k\in\lb1,N_1\rb$ and distinguish two cases.
    \begin{itemize}
        \item If $\alpha_k \in\lb 1, n-1\rb$: We have
        \[a(X,\theta)_{k,:} = \One_{\alpha_k} = (0,\dots,0, \underset{\underset{l'}{\uparrow}}{1},  \dots, 1).
        \]
        As a consequence, since $l<l'$, $a_{k} ( x^{(l)}, \theta ) = a(X,\theta)_{k,l} = 0 \neq 1 =a(X,\theta)_{k,l'} = a_{k} ( x^{(l')}, \theta )$ and therefore $f(l) = a ( x^{(l)}, \theta ) \neq a ( x^{(l')}, \theta )= f(l')$.
        \item If $\alpha_k \in\lb n+2, 2n\rb \cup\{0\}$: We have
        \[a(X,\theta)_{k,:} = \One_{\alpha_k} = (1, \dots,1, \underset{\underset{l'}{\uparrow}}{0}, \dots, 0).
        \]
        As a consequence, since $l<l'$, $a_{k} ( x^{(l)}, \theta )  = a(X,\theta)_{k,l} = 1 \neq 0 = a(X,\theta)_{k,l'} = a_{k} ( x^{(l')}, \theta )$ and therefore $f(l) = a ( x^{(l)}, \theta ) \neq a ( x^{(l')}, \theta )= f(l')$.
    \end{itemize}
In both cases $f(l) \neq f(l')$ and we conclude that $f$ is injective. Therefore, $|\cL''(\alpha)| \leq |\cA(X,\theta)|$. 

To prove that $ |\cA(X,\theta)| - 2 \leq |\cL''(\alpha)|$, we consider
\[\cA'(X,\theta) = \cA(X,\theta) \setminus \Bigl\{ a( x^{(1)}, \theta ), a( x^{(n)}, \theta )\Bigr\}
\]
and prove that all elements of $\cA'(X,\theta)$ are in the range of $f$. We consider $\delta\in \cA'(X,\theta)$. Using the definition of $\cA(X,\theta)$, we know there exists $i\in\lb1,n\rb$ such that $a( x^{(i)} , \theta ) = \delta$. We consider
\[i_{min} = \min \{ i\in\lb1,n\rb \mid \mbox{such that }a( x^{(i)} , \theta ) = \delta\}.
\]
Notice first that by definition of $i_{min}$ and $f$,  if $i_{min}\in\cL''(\alpha)$, we always have $\delta=f(i_{min})$.
To obtain the desired statement, we therefore only need to prove that $i_{min}\in\cL''(\alpha)$.

Because $\delta\in\cA'(X,\theta)$, we know that $i_{min} \in \lb2,n-1\rb$. Therefore, $a( x^{(i_{min})} , \theta )  \neq a( x^{(i_{min}-1)} , \theta ) $ and there exists $k\in\lb1,N_1\rb$ such that $a_k( x^{(i_{min})} , \theta )  \neq a_k( x^{(i_{min}-1)} , \theta ) $.
We distinguish two cases.
\begin{itemize}
    \item If $a_k( x^{(i_{min})} , \theta ) =1$ and $ a_k( x^{(i_{min}-1)} , \theta ) =0$: Then $a(X,\theta)_{k,:} = \One_{i_{min}}$. Therefore, $i_{min} \in \cL'(\alpha)$, and $i_{min} =i_{min}[n]\in\cL''(\alpha)$.
    \item If $a_k( x^{(i_{min})} , \theta ) =0$ and $ a_k( x^{(i_{min}-1)} , \theta ) =1$: Then $a(X,\theta)_{k,:} = \One_{n+i_{min}}$. Therefore, $n+i_{min} \in \cL'(\alpha)$ and $i_{min} = (n+i_{min})[n] \in\cL''(\alpha)$. 
\end{itemize}
In both cases, we conclude that $i_{min}\in\cL''(\alpha)$ and we have $\delta=f(i_{min})$. Therefore, we have $  |\cA'(X,\theta)| \leq |\cL''(\alpha)|$ and, because of the definition of $\cA'(X,\theta)$, we also have $ |\cA(X,\theta)| - 2  \leq |\cA'(X,\theta)|$.

This concludes the proof of Lemma \ref{lem12}.
 
\end{proof}

%%%%%%%%%%%%%%%
% Second Lemme A, et les l0

The following lemma makes the connection between $\ell^0_{neurons}$, and $\ell^0_{linear}$ and $|\cA(X,\theta)|$.
\begin{lem}\label{dernier_lemme}
     Consider any deep fully-connected ReLU network architecture $(E,V, Id)$, with $L=2$ and $N_0=N_2=1$. Consider $n\in\NN^*$, and a sample $X=(x^{(1)},x^{(2)}, \ldots,  x^{(n)})\in\RR^{1\times n}$ satisfying \eqref{x_ordered_main_part}. 

    For any $j\in\lb1,p_X\rb$, there exists $\alpha \in \{1,\dots,2n\}^{N_1}$ such that for all $\theta \in \widetilde\cU_j^X$ and
    all $k\in \lb1,N_1\rb$, $a(X,\theta)_{k,:}= \One_{\alpha_k} $, and the following inequalities hold:
            \begin{equation}\label{ineq4}
        |\cA(X,\theta)| \leq \ell^0_{linear}(f_\theta,X) \leq 2|\cA(X,\theta)|,
    \end{equation}
    and
    \begin{equation}\label{ineq1}
       |\cA(X,\theta)| - 2 \leq \ell^0_{neurons}(\alpha) \leq  4|\cA(X,\theta)|.
    \end{equation}
\end{lem}
\begin{proof}
Equation \eqref{ineq4} is a direct consequence of the definitions of $\cA(X,\theta)$, in \eqref{def_A_annexe}, and $\ell^0_{linear}(f_\theta,X)$, in \eqref{def-l0-linear}. We have indeed,
\begin{eqnarray*}
    |\cA(X,\theta)| \leq \ell^0_{linear}(f_{\theta}, X) & = & \sum_{\delta\in\cA(X,\theta) } \min\Bigl(2, \Bigl|\Bigl\{i\in\lb1,n\rb \mid a(x^{(i)}, \theta)=\delta\Bigr\}\Bigr| \Bigr) \\
    & \leq & 2 |\cA(X,\theta)| .
\end{eqnarray*}

To prove \eqref{ineq1}, we use the definition of $\cL(\alpha)$, in \eqref{eriquvt}, the fact that $\ell^0_{neurons}(\alpha) = |\cL(\alpha)|$, and the fact that to every element of $\cL'(\alpha)$ corresponds at least one, and at most two, elements of $\cL(\alpha)$ to obtain
\[        |\cL'(\alpha)| \leq \ell^0_{neurons}(\alpha) \leq 2|\cL'(\alpha)|.
\]
Using \eqref{ineq2}, we obtain
\[|\cL''(\alpha)|\leq \ell^0_{neurons}(\alpha) \leq  4|\cL''(\alpha)|,
\]
and, using \eqref{ineq3}, we get
\[|\cA(X,\theta)| - 2 \leq \ell^0_{neurons}(\alpha) \leq  4|\cA(X,\theta)|.
\]
This concludes the proof of Lemma \ref{dernier_lemme}.
\end{proof}
\begin{proof}{\bf of Theorem \ref{shallow-case-thm}:}
    Theorem \ref{shallow-case-thm} is a direct consequence of Theorem \ref{shallow-case-thm-accurate} and Lemma \ref{dernier_lemme}. We have indeed
    \[
      \frac{1}{2} |\cA(X,\theta)| \leq \Big(1+\frac{1}{2}\ell^0_{neurons}(\alpha) \Big) \leq 
    \RK{ \DT f_{\theta}(X) } 
    \]
    and
    \begin{eqnarray*}
        \RK{ \DT f_{\theta}(X) }  & \leq &   \min\Big(1+\ell^0_{neurons}(\alpha), \ell^0_{linear}(f_{\theta}, X) \Big) \\
& \leq & \min\Big(1+4|\cA(X,\theta)|, 2 |\cA(X,\theta)|\Big)   =  2 |\cA(X,\theta)|.
    \end{eqnarray*}
\end{proof}

%\section{Experiments}

%\begin{figure}[ht]
%    \centering

%    \begin{subfigure}{\linewidth}
%        \centering
%        \includegraphics[scale=0.5]{Figures/exp2-initialization.png}
%        \caption{Initialization}
%        \label{fig:top}
%    \end{subfigure}

%    \vspace{15pt} % optional space between figures

%    \begin{subfigure}{\linewidth}
%        \centering
%        \includegraphics[scale=0.5]{Figures/exp2-final.png}
%        \caption{Epoch 300}
%        \label{fig:bottom}
%    \end{subfigure}

%    \caption{Evolution of the parameters in the open sets $\widetilde{\cU}^X_j$ and of their images during training, for 1000 different initializations. The parameters are represented in the $(w,b)$ space (left), and their corresponding (projected) images are represented in the output set (right), both at initialization (a) and after 300 epochs of training (b). The color of the points indicates the value of the MSE. }
%    \label{fig:exp1-alt}
%\end{figure}

%%%%%%%%%%%%%%%%%%%%%%%%%%%%%%%%%%%%%%%%%%%%%%%%%%%%%%%%%%%%
%%%%%%%%%%%%%%%%%%%%%%%%%%%%%%%%%%%%%%%%%%%%%%%%%%%%%%%%%%%%

\vskip 0.2in
\bibliography{biblio}

\begin{thebibliography}{69}
\providecommand{\natexlab}[1]{#1}
\providecommand{\url}[1]{\texttt{#1}}
\expandafter\ifx\csname urlstyle\endcsname\relax
  \providecommand{\doi}[1]{doi: #1}\else
  \providecommand{\doi}{doi: \begingroup \urlstyle{rm}\Url}\fi

\bibitem[Abbe et~al.(2023)Abbe, Adsera, and Misiakiewicz]{abbe2023sgd}
Emmanuel Abbe, Enric~Boix Adsera, and Theodor Misiakiewicz.
\newblock {SGD} learning on neural networks: leap complexity and
  saddle-to-saddle dynamics.
\newblock In \emph{The Thirty Sixth Annual Conference on Learning Theory},
  pages 2552--2623. PMLR, 2023.

\bibitem[Achour et~al.(2024)Achour, Malgouyres, and
  Gerchinovitz]{achour2024loss}
El~Mehdi Achour, Fran{\c{c}}ois Malgouyres, and S{\'e}bastien Gerchinovitz.
\newblock The loss landscape of deep linear neural networks: a second-order
  analysis.
\newblock \emph{Journal of Machine Learning Research}, 25\penalty0
  (242):\penalty0 1--76, 2024.

\bibitem[Anthony and Bartlett(2009)]{AB09-NeuralNetworkLearning}
Martin Anthony and Peter~L. Bartlett.
\newblock \emph{Neural Network Learning: Theoretical Foundations}.
\newblock Cambridge University Press, 2009.

\bibitem[Arora et~al.(2018)Arora, Ge, Neyshabur, and Zhang]{pmlr-v80-arora18b}
Sanjeev Arora, Rong Ge, Behnam Neyshabur, and Yi~Zhang.
\newblock Stronger generalization bounds for deep nets via a compression
  approach.
\newblock In \emph{Proceedings of the 35th International Conference on Machine
  Learning}, volume~80, pages 254--263, 2018.

\bibitem[Arora et~al.(2019)Arora, Cohen, Hu, and Luo]{arora2019implicit}
Sanjeev Arora, Nadav Cohen, Wei Hu, and Yuping Luo.
\newblock Implicit regularization in deep matrix factorization.
\newblock In \emph{Advances in Neural Information Processing Systems}, pages
  7413--7424, 2019.

\bibitem[Bartlett et~al.(1998)Bartlett, Maiorov, and Meir]{bartlett1998almost}
Peter~L Bartlett, Vitaly Maiorov, and Ron Meir.
\newblock Almost linear {VC}-dimension bounds for piecewise polynomial
  networks.
\newblock \emph{Neural Computation}, 10\penalty0 (8):\penalty0 2159--2173,
  1998.

\bibitem[Bartlett et~al.(2017)Bartlett, Foster, and
  Telgarsky]{bartlett2017spectrally}
Peter~L Bartlett, Dylan~J Foster, and Matus~J Telgarsky.
\newblock Spectrally-normalized margin bounds for neural networks.
\newblock \emph{Advances in neural information processing systems}, 30, 2017.

\bibitem[Bartlett et~al.(2019)Bartlett, Harvey, Liaw, and
  Mehrabian]{bartlett2019nearly}
Peter~L Bartlett, Nick Harvey, Christopher Liaw, and Abbas Mehrabian.
\newblock Nearly-tight {VC}-dimension and pseudodimension bounds for piecewise
  linear neural networks.
\newblock \emph{Journal of Machine Learning Research}, 20\penalty0
  (63):\penalty0 1--17, 2019.

\bibitem[Bartlett et~al.(2020)Bartlett, Long, Lugosi, and
  Tsigler]{bartlett2020benign}
Peter~L Bartlett, Philip~M Long, G{\'a}bor Lugosi, and Alexander Tsigler.
\newblock Benign overfitting in linear regression.
\newblock \emph{Proceedings of the National Academy of Sciences}, 117\penalty0
  (48):\penalty0 30063--30070, 2020.

\bibitem[Belkin(2021)]{belkin2021fit}
Mikhail Belkin.
\newblock Fit without fear: remarkable mathematical phenomena of deep learning
  through the prism of interpolation.
\newblock \emph{Acta Numerica}, 30:\penalty0 203--248, 2021.

\bibitem[Bona-Pellissier et~al.(2022)Bona-Pellissier, Malgouyres, and
  Bachoc]{bona2022local}
Joachim Bona-Pellissier, Fran{\c{c}}ois Malgouyres, and Fran{\c{c}}ois Bachoc.
\newblock Local identifiability of deep {ReLU} neural networks: the theory.
\newblock In \emph{Advances in Neural Information Processing Systems}, 2022.

\bibitem[Bona-Pellissier et~al.(2023{\natexlab{a}})Bona-Pellissier, Bachoc, and
  Malgouyres]{bona2023parameter}
Joachim Bona-Pellissier, Fran{\c{c}}ois Bachoc, and Fran{\c{c}}ois Malgouyres.
\newblock Parameter identifiability of a deep feedforward {ReLU} neural
  network.
\newblock \emph{Machine Learning}, 112\penalty0 (11):\penalty0 4431--4493,
  2023{\natexlab{a}}.

\bibitem[Bona-Pellissier et~al.(2023{\natexlab{b}})Bona-Pellissier, Malgouyres,
  and Bachoc]{code_calcul_rang}
Joachim Bona-Pellissier, Fran{\c{c}}ois Malgouyres, and Fran{\c{c}}ois Bachoc.
\newblock \url{https://github.com/JoachimBP/Functional-dimension},
  2023{\natexlab{b}}.
\newblock Code of the experiments of this article.

\bibitem[Boursier and Flammarion(2025{\natexlab{a}})]{boursier2024early}
Etienne Boursier and Nicolas Flammarion.
\newblock Early alignment in two-layer networks training is a two-edged sword.
\newblock \emph{Journal of Machine Learning Research}, 26\penalty0
  (183):\penalty0 1--75, 2025{\natexlab{a}}.

\bibitem[Boursier and
  Flammarion(2025{\natexlab{b}})]{boursier2024simplicitybiasoptimizationthreshold}
Etienne Boursier and Nicolas Flammarion.
\newblock Simplicity bias and optimization threshold in two-layer {R}e{LU}
  networks.
\newblock In \emph{International Conference on Machine Learning},
  2025{\natexlab{b}}.

\bibitem[Boursier et~al.(2022)Boursier, Pillaud-Vivien, and
  Flammarion]{boursier2022gradient}
Etienne Boursier, Loucas Pillaud-Vivien, and Nicolas Flammarion.
\newblock Gradient flow dynamics of shallow relu networks for square loss and
  orthogonal inputs.
\newblock \emph{Advances in Neural Information Processing Systems},
  35:\penalty0 20105--20118, 2022.

\bibitem[Camuto et~al.(2021)Camuto, Deligiannidis, Erdogdu, Gurbuzbalaban,
  Simsekli, and Zhu]{camuto2021fractal}
Alexander Camuto, George Deligiannidis, Murat~A Erdogdu, Mert Gurbuzbalaban,
  Umut Simsekli, and Lingjiong Zhu.
\newblock Fractal structure and generalization properties of stochastic
  optimization algorithms.
\newblock \emph{Advances in Neural Information Processing Systems}, 34, 2021.

\bibitem[Carlini et~al.(2020)Carlini, Jagielski, and
  Mironov]{carlini2020cryptanalytic}
Nicholas Carlini, Matthew Jagielski, and Ilya Mironov.
\newblock Cryptanalytic extraction of neural network models.
\newblock In \emph{Annual International Cryptology Conference}, pages 189--218.
  Springer, 2020.

\bibitem[Cha et~al.(2021)Cha, Chun, Lee, Cho, Park, Lee, and Park]{cha2021swad}
Junbum Cha, Sanghyuk Chun, Kyungjae Lee, Han-Cheol Cho, Seunghyun Park, Yunsung
  Lee, and Sungrae Park.
\newblock {SWAD}: Domain generalization by seeking flat minima.
\newblock \emph{Advances in Neural Information Processing Systems},
  34:\penalty0 22405--22418, 2021.

\bibitem[Chaudhari et~al.(2019)Chaudhari, Choromanska, Soatto, LeCun, Baldassi,
  Borgs, Chayes, Sagun, and Zecchina]{chaudhari2019entropy}
Pratik Chaudhari, Anna Choromanska, Stefano Soatto, Yann LeCun, Carlo Baldassi,
  Christian Borgs, Jennifer Chayes, Levent Sagun, and Riccardo Zecchina.
\newblock Entropy-{SGD}: Biasing gradient descent into wide valleys.
\newblock \emph{Journal of Statistical Mechanics: Theory and Experiment},
  2019\penalty0 (12):\penalty0 124018, 2019.

\bibitem[Chizat and Bach(2020)]{chizat2020implicit}
Lenaic Chizat and Francis Bach.
\newblock Implicit bias of gradient descent for wide two-layer neural networks
  trained with the logistic loss.
\newblock In \emph{Conference on learning theory}, pages 1305--1338. PMLR,
  2020.

\bibitem[Cooper(2021)]{cooper2021global}
Yaim Cooper.
\newblock Global minima of overparameterized neural networks.
\newblock \emph{SIAM Journal on Mathematics of Data Science}, 3\penalty0
  (2):\penalty0 676--691, 2021.

\bibitem[Dinh et~al.(2017)Dinh, Pascanu, Bengio, and Bengio]{dinh2017sharp}
Laurent Dinh, Razvan Pascanu, Samy Bengio, and Yoshua Bengio.
\newblock Sharp minima can generalize for deep nets.
\newblock In \emph{International Conference on Machine Learning}, pages
  1019--1028, 2017.

\bibitem[Du et~al.(2019)Du, Lee, Li, Wang, and Zhai]{du2019gradient}
Simon Du, Jason Lee, Haochuan Li, Liwei Wang, and Xiyu Zhai.
\newblock Gradient descent finds global minima of deep neural networks.
\newblock In \emph{International conference on machine learning}, pages
  1675--1685, 2019.

\bibitem[Foret et~al.(2021)Foret, Kleiner, Mobahi, and
  Neyshabur]{foret2021sharpnessaware}
Pierre Foret, Ariel Kleiner, Hossein Mobahi, and Behnam Neyshabur.
\newblock Sharpness-aware minimization for efficiently improving
  generalization.
\newblock In \emph{International Conference on Learning Representations}, 2021.

\bibitem[Ghorbani et~al.(2019)Ghorbani, Krishnan, and
  Xiao]{ghorbani2019investigation}
Behrooz Ghorbani, Shankar Krishnan, and Ying Xiao.
\newblock An investigation into neural net optimization via {Hessian}
  eigenvalue density.
\newblock In \emph{International Conference on Machine Learning}, pages
  2232--2241, 2019.

\bibitem[Gidel et~al.(2019)Gidel, Bach, and Lacoste-Julien]{gidel2019implicit}
Gauthier Gidel, Francis Bach, and Simon Lacoste-Julien.
\newblock Implicit regularization of discrete gradient dynamics in linear
  neural networks.
\newblock In \emph{Advances in Neural Information Processing Systems}, pages
  3202--3211, 2019.

\bibitem[Gissin et~al.(2019)Gissin, Shalev-Shwartz, and
  Daniely]{gissin2019implicit}
Daniel Gissin, Shai Shalev-Shwartz, and Amit Daniely.
\newblock The implicit bias of depth: How incremental learning drives
  generalization.
\newblock In \emph{International Conference on Learning Representations}, 2019.

\bibitem[Golowich et~al.(2018)Golowich, Rakhlin, and Shamir]{golowich2018size}
Noah Golowich, Alexander Rakhlin, and Ohad Shamir.
\newblock Size-independent sample complexity of neural networks.
\newblock In \emph{Conference On Learning Theory}, pages 297--299, 2018.

\bibitem[Grigsby et~al.(2023)Grigsby, Lindsey, and Rolnick]{grigsby2023hidden}
Elisenda Grigsby, Kathryn Lindsey, and David Rolnick.
\newblock Hidden symmetries of {ReLU} networks.
\newblock In \emph{International Conference on Machine Learning}, pages
  11734--11760, 2023.

\bibitem[Grigsby et~al.(2025)Grigsby, Lindsey, Meyerhoff, and
  Wu]{grigsby2022functional}
Elisenda Grigsby, Kathryn Lindsey, Robert Meyerhoff, and Chenxi Wu.
\newblock Functional dimension of feedforward {ReLU} neural networks.
\newblock \emph{Advances in Mathematics}, 482:\penalty0 110636, 2025.

\bibitem[Grigsby and Lindsey(2022)]{grigsby2022transversality}
J~Elisenda Grigsby and Kathryn Lindsey.
\newblock On transversality of bent hyperplane arrangements and the topological
  expressiveness of {ReLU} neural networks.
\newblock \emph{SIAM Journal on Applied Algebra and Geometry}, 6\penalty0
  (2):\penalty0 216--242, 2022.

\bibitem[Grohs and Kutyniok(2022)]{grohs22-mathematicalAspectsDeepLearning}
Philipp Grohs and Gitta Kutyniok, editors.
\newblock \emph{Mathematical Aspects of Deep Learning}.
\newblock Cambridge University Press, 2022.

\bibitem[Gur-Ari et~al.(2018)Gur-Ari, Roberts, and Dyer]{gur2018gradient}
Guy Gur-Ari, Daniel~A Roberts, and Ethan Dyer.
\newblock Gradient descent happens in a tiny subspace.
\newblock \emph{arXiv preprint arXiv:1812.04754}, 2018.

\bibitem[Haddouche et~al.(2025)Haddouche, Viallard, {\c{S}}im{\c{s}}ekli, and
  Guedj]{haddouche2025pac}
Maxime Haddouche, Paul Viallard, Umut {\c{S}}im{\c{s}}ekli, and Benjamin Guedj.
\newblock A {PAC-Bayesian} link between generalisation and flat minima.
\newblock In \emph{ALT 2025-36th International Conference on Algorithmic
  Learning Theory}, pages 1--31, 2025.

\bibitem[Hanin and Rolnick(2019)]{hanin2019complexity}
Boris Hanin and David Rolnick.
\newblock Complexity of linear regions in deep networks.
\newblock In \emph{International Conference on Machine Learning}, pages
  2596--2604, 2019.

\bibitem[Harvey et~al.(2017)Harvey, Liaw, and Mehrabian]{harvey2017nearly}
Nick Harvey, Christopher Liaw, and Abbas Mehrabian.
\newblock Nearly-tight {VC}-dimension bounds for piecewise linear neural
  networks.
\newblock In \emph{Conference on learning theory}, pages 1064--1068, 2017.

\bibitem[Hochreiter and Schmidhuber(1997)]{hochreiter1997flat}
Sepp Hochreiter and J{\"u}rgen Schmidhuber.
\newblock Flat minima.
\newblock \emph{Neural computation}, 9\penalty0 (1):\penalty0 1--42, 1997.

\bibitem[Imaizumi and Schmidt-Hieber(2023)]{imaizumi2022generalization}
Masaaki Imaizumi and Johannes Schmidt-Hieber.
\newblock On generalization bounds for deep networks based on loss surface
  implicit regularization.
\newblock \emph{IEEE Transactions on Information Theory}, 69\penalty0 (2),
  2023.

\bibitem[Jacot et~al.(2021)Jacot, Ged, {\c{S}}im{\c{s}}ek, Hongler, and
  Gabriel]{jacot2021saddle}
Arthur Jacot, Fran{\c{c}}ois Ged, Berfin {\c{S}}im{\c{s}}ek, Cl{\'e}ment
  Hongler, and Franck Gabriel.
\newblock Saddle-to-saddle dynamics in deep linear networks: Small
  initialization training, symmetry, and sparsity.
\newblock \emph{arXiv preprint arXiv:2106.15933}, 2021.

\bibitem[Keskar et~al.(2017)Keskar, Mudigere, Nocedal, Smelyanskiy, and
  Tang]{keskar2017on}
Nitish~Shirish Keskar, Dheevatsa Mudigere, Jorge Nocedal, Mikhail Smelyanskiy,
  and Ping Tak~Peter Tang.
\newblock On large-batch training for deep learning: Generalization gap and
  sharp minima.
\newblock In \emph{International Conference on Learning Representations}, 2017.

\bibitem[Li et~al.(2018)Li, Farkhoor, Liu, and Yosinski]{limeasuring}
Chunyuan Li, Heerad Farkhoor, Rosanne Liu, and Jason Yosinski.
\newblock Measuring the intrinsic dimension of objective landscapes.
\newblock In \emph{International Conference on Learning Representations}, 2018.

\bibitem[Lyu and Li(2020)]{Lyu2020Gradient}
Kaifeng Lyu and Jian Li.
\newblock Gradient descent maximizes the margin of homogeneous neural networks.
\newblock In \emph{International Conference on Learning Representations}, 2020.

\bibitem[Maass(1994)]{maass1994neural}
Wolfgang Maass.
\newblock Neural nets with superlinear {VC}-dimension.
\newblock \emph{Neural Computation}, 6\penalty0 (5):\penalty0 877--884, 1994.

\bibitem[Meinshausen and B{\"u}hlmann(2006)]{meinshausen2006high}
Nicolai Meinshausen and Peter B{\"u}hlmann.
\newblock High-dimensional graphs and variable selection with the lasso.
\newblock \emph{The Annals of Statistics}, 34\penalty0 (3):\penalty0
  1436--1462, 2006.

\bibitem[Mityagin(2020)]{mityagin2015zero}
Boris Mityagin.
\newblock The zero set of a real analytic function.
\newblock \emph{Math Notes}, 107:\penalty0 529–530, 2020.

\bibitem[Montufar et~al.(2014)Montufar, Pascanu, Cho, and
  Bengio]{montufar2014number}
Guido~F Montufar, Razvan Pascanu, Kyunghyun Cho, and Yoshua Bengio.
\newblock On the number of linear regions of deep neural networks.
\newblock \emph{Advances in neural information processing systems}, 27, 2014.

\bibitem[Neyshabur et~al.(2015{\natexlab{a}})Neyshabur, Salakhutdinov, and
  Srebro]{neyshabur2015path}
Behnam Neyshabur, Russ~R Salakhutdinov, and Nati Srebro.
\newblock Path-{SGD}: Path-normalized optimization in deep neural networks.
\newblock \emph{Advances in neural information processing systems}, 28,
  2015{\natexlab{a}}.

\bibitem[Neyshabur et~al.(2015{\natexlab{b}})Neyshabur, Tomioka, and
  Srebro]{neyshabur2015norm}
Behnam Neyshabur, Ryota Tomioka, and Nathan Srebro.
\newblock Norm-based capacity control in neural networks.
\newblock In \emph{Conference on learning theory}, pages 1376--1401,
  2015{\natexlab{b}}.

\bibitem[Neyshabur et~al.(2017)Neyshabur, Bhojanapalli, McAllester, and
  Srebro]{neyshabur2017exploring}
Behnam Neyshabur, Srinadh Bhojanapalli, David McAllester, and Nati Srebro.
\newblock Exploring generalization in deep learning.
\newblock \emph{Advances in neural information processing systems}, 30, 2017.

\bibitem[Nguyen and Hein(2017)]{nguyen2017loss}
Quynh Nguyen and Matthias Hein.
\newblock The loss surface of deep and wide neural networks.
\newblock In \emph{International conference on machine learning}, pages
  2603--2612, 2017.

\bibitem[Pesme and Flammarion(2023)]{pesme2023saddle}
Scott Pesme and Nicolas Flammarion.
\newblock Saddle-to-saddle dynamics in diagonal linear networks.
\newblock \emph{Advances in Neural Information Processing Systems},
  36:\penalty0 7475--7505, 2023.

\bibitem[Petersen et~al.(2021)Petersen, Raslan, and
  Voigtlaender]{petersen2021topological}
Philipp Petersen, Mones Raslan, and Felix Voigtlaender.
\newblock Topological properties of the set of functions generated by neural
  networks of fixed size.
\newblock \emph{Foundations of computational mathematics}, 21\penalty0
  (2):\penalty0 375--444, 2021.

\bibitem[Petzka et~al.(2020)Petzka, Trimmel, and
  Sminchisescu]{petzka2020symmetries}
Henning Petzka, Martin Trimmel, and Cristian Sminchisescu.
\newblock Notes on the symmetries of 2-layer {ReLU}-networks.
\newblock In \emph{Proceedings of the Northern Lights Deep Learning Workshop},
  volume~1, pages 6--6, 2020.

\bibitem[Raghu et~al.(2017)Raghu, Poole, Kleinberg, Ganguli, and
  Sohl-Dickstein]{raghu2017expressive}
Maithra Raghu, Ben Poole, Jon Kleinberg, Surya Ganguli, and Jascha
  Sohl-Dickstein.
\newblock On the expressive power of deep neural networks.
\newblock In \emph{International Conference on Machine Learning}, pages
  2847--2854, 2017.

\bibitem[Razin and Cohen(2020)]{razin2020implicit}
Noam Razin and Nadav Cohen.
\newblock Implicit regularization in deep learning may not be explainable by
  norms.
\newblock \emph{Advances in Neural Information Processing Systems},
  33:\penalty0 21174--21187, 2020.

\bibitem[Rolnick and Kording(2020)]{pmlr-v119-rolnick20a}
David Rolnick and Konrad Kording.
\newblock Reverse-engineering deep {R}e{LU} networks.
\newblock In \emph{Proceedings of the 37th International Conference on Machine
  Learning}, volume 119, pages 8178--8187, 2020.

\bibitem[Safran et~al.(2021)Safran, Yehudai, and Shamir]{safran2021effects}
Itay~M Safran, Gilad Yehudai, and Ohad Shamir.
\newblock The effects of mild over-parameterization on the optimization
  landscape of shallow {ReLU} neural networks.
\newblock In \emph{Conference on Learning Theory}, pages 3889--3934, 2021.

\bibitem[Sagun et~al.(2016)Sagun, Bottou, and LeCun]{sagun2016eigenvalues}
Levent Sagun, Leon Bottou, and Yann LeCun.
\newblock Eigenvalues of the {Hessian} in deep learning: Singularity and
  beyond.
\newblock \emph{arXiv preprint arXiv:1611.07476}, 2016.

\bibitem[Saxe et~al.(2019)Saxe, McClelland, and Ganguli]{saxe2019mathematical}
Andrew~M. Saxe, James~L. McClelland, and Surya Ganguli.
\newblock A mathematical theory of semantic development in deep neural
  networks.
\newblock \emph{Proceedings of the National Academy of Sciences}, 116\penalty0
  (23):\penalty0 11537--11546, 2019.

\bibitem[Sonoda et~al.(2021)Sonoda, Ishikawa, and Ikeda]{sonoda2021ghosts}
Sho Sonoda, Isao Ishikawa, and Masahiro Ikeda.
\newblock Ghosts in neural networks: Existence, structure and role of
  infinite-dimensional null space.
\newblock \emph{arXiv preprint arXiv:2106.04770}, 2021.

\bibitem[Soudry and Carmon(2016)]{soudry2016no}
Daniel Soudry and Yair Carmon.
\newblock No bad local minima: Data independent training error guarantees for
  multilayer neural networks.
\newblock \emph{arXiv preprint arXiv:1605.08361}, 2016.

\bibitem[Stock(2021)]{stock:tel-03208517}
Pierre Stock.
\newblock \emph{{Efficiency and Redundancy in Deep Learning Models :
  Theoretical Considerations and Practical Applications}}.
\newblock Ph{D} thesis, {Universit{\'e} de Lyon}, April 2021.
\newblock URL \url{https://tel.archives-ouvertes.fr/tel-03208517}.

\bibitem[Stock and Gribonval(2022)]{stock2022embedding}
Pierre Stock and R{\'e}mi Gribonval.
\newblock An embedding of {ReLU} networks and an analysis of their
  identifiability.
\newblock \emph{Constructive Approximation}, 2022.

\bibitem[Suzuki et~al.(2020)Suzuki, Abe, and Nishimura]{Suzuki2020Compression}
Taiji Suzuki, Hiroshi Abe, and Tomoaki Nishimura.
\newblock Compression based bound for non-compressed network: unified
  generalization error analysis of large compressible deep neural network.
\newblock In \emph{International Conference on Learning Representations}, 2020.

\bibitem[Wu et~al.(2017)Wu, Zhu, and E]{wu2017towards}
Lei Wu, Zhanxing Zhu, and Weinan E.
\newblock Towards understanding generalization of deep learning: Perspective of
  loss landscapes.
\newblock \emph{Workshop \rq Principled Approaches to Deep Learning\lq, ICML},
  2017.

\bibitem[Yi et~al.(2019)Yi, Meng, Chen, Ma, and Liu]{yi2019positively}
Mingyang Yi, Qi~Meng, Wei Chen, Zhi-ming Ma, and Tie-Yan Liu.
\newblock Positively scale-invariant flatness of {ReLU} neural networks.
\newblock \emph{arXiv preprint arXiv:1903.02237}, 2019.

\bibitem[Yuan and Lin(2007)]{yuan2007non}
Ming Yuan and Yi~Lin.
\newblock On the non-negative garrotte estimator.
\newblock \emph{Journal of the Royal Statistical Society Series B: Statistical
  Methodology}, 69\penalty0 (2):\penalty0 143--161, 2007.

\bibitem[Zhang et~al.(2021)Zhang, Bengio, Hardt, Recht, and
  Vinyals]{zhang2021understanding}
Chiyuan Zhang, Samy Bengio, Moritz Hardt, Benjamin Recht, and Oriol Vinyals.
\newblock Understanding deep learning (still) requires rethinking
  generalization.
\newblock \emph{Communications of the ACM}, 64\penalty0 (3):\penalty0 107--115,
  2021.

\end{thebibliography}
 
\end{document}